%% file: iclr_main.tex
\documentclass{article} %
\usepackage{iclr2026/iclr2026_conference,times}

\usepackage{microtype}
\usepackage{graphicx}
\usepackage{booktabs}

\usepackage{amsmath}
\usepackage{amssymb}
\usepackage{amsthm}
\usepackage{mathtools}
\numberwithin{equation}{section}
\usepackage{algorithm}
\usepackage{algorithmic}

\usepackage{slashed}
\usepackage{braket}

\usepackage{url}
\usepackage{float}
\usepackage{pifont}
\usepackage{wrapfig}

\usepackage{enumitem}
\usepackage[normalem]{ulem}
\useunder{\uline}{\ul}{}

\setlist[itemize]{leftmargin=1em, before=\vspace{-0.5em}, after=\vspace{-0.5em}, itemsep=0.1em}
\setlist[enumerate]{leftmargin=1.5em, before=\vspace{-0.5em}, after=\vspace{-0.5em}, itemsep=0.1em}

\usepackage[svgnames]{xcolor}
\definecolor{DarkGreen}{rgb}{0,0.40,0}
\definecolor{FireBrick}{rgb}{0.698,0.133,0.133}
\usepackage[colorlinks,citecolor=DarkGreen,linkcolor=FireBrick,urlcolor=FireBrick,linktocpage,unicode]{hyperref}

\usepackage{caption}
\usepackage{subcaption}

\usepackage{tikz}
\usetikzlibrary{arrows.meta}
\usetikzlibrary{bayesnet}
\usetikzlibrary{calc}

\usepackage{titlesec}
\titlespacing*{\section}{0pt}{2pt}{2pt}
\titlespacing*{\subsection}{0pt}{0pt}{0pt}
\titlespacing*{\subsubsection}{0pt}{0pt}{0pt}

\usepackage{titletoc}

\input{iclr2026/iclr2026_macros}

\crefname{figure}{Fig.}{Figs.}
\Crefname{figure}{Fig.}{Figs.}
\crefname{table}{Tab.}{Tabs.}
\Crefname{table}{Tab.}{Tabs.}
\crefname{section}{Sec.}{Sec.}
\Crefname{section}{Sec.}{Sec.}
\crefname{equation}{Eq.}{Eqs.}
\Crefname{equation}{Eq.}{Eqs.}
\crefname{appendix}{App.}{App.}
\Crefname{appendix}{App.}{App.}
\crefname{definition}{Def.}{Defs.}
\Crefname{definition}{Def.}{Defs.}
\crefname{lemma}{Lem.}{Lems.}
\Crefname{lemma}{Lem.}{Lems.}
\crefname{theorem}{Thm.}{Thms.}
\Crefname{theorem}{Thm.}{Thms.}
\crefname{remark}{Rmk.}{Rmks.}
\Crefname{remark}{Rmk.}{Rmks.}
\crefname{algorithm}{Alg.}{Algs.}
\Crefname{algorithm}{Alg.}{Algs.}

\title{Decoupled Diffusion Solver for Inverse Problems on Function Spaces}

\author{
Thomas Y.L. Lin$^{1,*}$ \;
Jiachen Yao$^{2,*}$ \;
Lufang Chiang$^{3}$ \;
Julius Berner$^{4}$ \;
Anima Anandkumar$^{2}$
\\[0.6em]
$^{1}$University of Washington, Seattle, USA \\
$^{2}$California Institute of Technology, Pasadena, USA \\
$^{3}$National Taiwan University, Taipei, Taiwan \\
$^{4}$NVIDIA Corporation, USA \\
$^{*}$Equal contribution.
}

\iclrfinalcopy

\begin{document}

\maketitle

\begin{abstract}
\input{0abstract}

\end{abstract}

\section{Introduction}
\label{sec:intro}

\input{1intro}

\section{Posterior Sampling for Inverse Problems}
\label{sec:preliminary}
\input{2preliminary}

\section{Decoupled Diffusion Inverse Solver (DDIS)}
\label{sec:ddis}

\input{3ddis}

\section{Theoretical Analysis}
\label{sec:theory}

\input{4theory}

\section{Experiments}
\label{sec:exp}

\input{5exp}

\section{Conclusion and Discussion}
\label{sec:conclusion}

\input{6conclusion}

\subsubsection*{Acknowledgments}
Anima Anandkumar is supported in part by Bren endowed chair, ONR (MURI grant N00014-23-1-2654), and the AI2050 senior fellow program at Schmidt Sciences. Jiachen Yao is supported in part by the Naren and Vinita Gupta Fellowship.

\bibliography{reference}
\bibliographystyle{iclr2026/iclr2026_conference}

\newpage

\appendix
\label{sec:append}
\part*{Appendix}
{
\setlength{\parskip}{-0em}
\startcontents[sections]
\printcontents[sections]{}{1}{}
}
\input{appendix}

\end{document}

%% file: iclr2026/iclr2026_macros.tex
\allowdisplaybreaks
\usepackage{bm}
\usepackage{mdframed}
\usepackage{nicefrac}
\usepackage{multirow}
\usepackage{adjustbox}
\usepackage{tabularx}

\newcolumntype{Y}{>{\raggedright\arraybackslash}X}
\newcolumntype{M}{>{\centering\arraybackslash}X}
\newcolumntype{L}{>{\raggedright\arraybackslash}p{2.8cm}} 

\usepackage[nameinlink,capitalize,noabbrev]{cleveref}

\definecolor{LightCyan}{rgb}{0.8, 0.9, 1}
\newlength{\Oldarrayrulewidth}

\newcommand{\cellcolorcontents}[2]{%
  \begingroup
  \setlength{\fboxsep}{0pt}%
  \colorbox{#1}{\strut #2}%
  \endgroup
}

\newcolumntype{C}[1]{>{\collectcell\cellcolorcontents{LightCyan}}c<{#1}} 

\definecolor{LightCyan}{rgb}{0.8, 0.9, 1}
\newcolumntype{b}{>{\columncolor{LightCyan}\hspace{0pt}}c}
\definecolor{cadetgrey}{rgb}{0.57, 0.64, 0.69}
\newcolumntype{g}{>{\columncolor{cadetgrey}\hspace{0pt}}c}

\usepackage{array}
\usepackage{makecell}

\newcommand{\bea}{\begin{eqnarray}}
\newcommand{\eea}{\end{eqnarray}}

\usepackage{slashed}

\usepackage{dashbox}
\usepackage{colortbl}
\usepackage{arydshln}
\definecolor{lightyellow}{rgb}{1.0, 0.95, 0.7}
\definecolor{Blue}{rgb}{0, 0, 0.8}
\definecolor{blue}{rgb}{0,0,1}
\definecolor{darkgreen}{rgb}{0,0.40,0}
\definecolor{firebrick}{rgb}{0.698,0.133,0.133}

\definecolor{colorA}{rgb}{1,0,0}
\definecolor{colorB}{rgb}{0,0.3,1}
\definecolor{colorC}{rgb}{0.9,0.8,0.2}
\definecolor{colorD}{rgb}{0,0.65,0}
\definecolor{lesslightgray}{rgb}{0.5,0.5,0.5}
\definecolor{light-gray}{gray}{0.95}

\let\tilde\widetilde
\let\hat\widehat

\let\cite\citep 

\makeatletter
\def\th@remark{%
  \thm@headfont{\bfseries}%
  \normalfont 
  \thm@preskip\parskip 
  \thm@postskip\parskip 
}
\makeatother
\usepackage[many]{tcolorbox}
\theoremstyle{definition}
\newtheorem{theorem}{Theorem}[section]
\tcolorboxenvironment{theorem}{
  breakable,
  colback=black!10,
  colframe=white,%
  width=\dimexpr\linewidth+10pt\relax,%
  enlarge left by=-5pt,%
  enlarge right by=-5pt,%
  boxsep=5pt,%
  boxrule=0pt,
  left=0pt,right=0pt,top=0pt,bottom=0pt,
  sharp corners,
  before skip=\topsep,
  after skip=\topsep
}
\newtheorem{lemma}{Lemma}[section]
\tcolorboxenvironment{lemma}{
  breakable,
  colback=black!10,
  colframe=white,%
  width=\dimexpr\linewidth+10pt\relax,%
  enlarge left by=-5pt,%
  enlarge right by=-5pt,%
  boxsep=5pt,%
  boxrule=0pt,
  left=0pt,right=0pt,top=0pt,bottom=0pt,
  sharp corners,
  before skip=\topsep,
  after skip=\topsep
}
\newtheorem{corollary}{Corollary}[theorem]
\tcolorboxenvironment{corollary}{
  breakable,
  colback=black!10,
  colframe=white,%
  width=\dimexpr\linewidth+10pt\relax,%
  enlarge left by=-5pt,%
  enlarge right by=-5pt,%
  boxsep=5pt,%
  boxrule=0pt,
  left=0pt,right=0pt,top=0pt,bottom=0pt,
  sharp corners,
  before skip=\topsep,
  after skip=\topsep
}
\newtheorem{proposition}{Proposition}[section]
\theoremstyle{definition}
\newtheorem{definition}{Definition}[section]
\tcolorboxenvironment{definition}{
  breakable,
  colback=black!10,
  colframe=white,%
  width=\dimexpr\linewidth+10pt\relax,%
  enlarge left by=-5pt,%
  enlarge right by=-5pt,%
  boxsep=5pt,%
  boxrule=0pt,
  left=0pt,right=0pt,top=0pt,bottom=0pt,
  sharp corners,
  before skip=\topsep,
  after skip=\topsep
}
\theoremstyle{remark}
\newtheorem{remark}{Remark}[section]

\tcolorboxenvironment{assumption}{
  breakable,
  colback=black!10,
  colframe=white,%
  width=\dimexpr\linewidth+10pt\relax,%
  enlarge left by=-5pt,%
  enlarge right by=-5pt,%
  boxsep=5pt,%
  boxrule=0pt,
  left=0pt,right=0pt,top=0pt,bottom=0pt,
  sharp corners,
  before skip=\topsep,
  after skip=\topsep
}

\tcolorboxenvironment{problem}{
  breakable,
  colback=black!10,
  colframe=white,%
  width=\dimexpr\linewidth+10pt\relax,%
  enlarge left by=-5pt,%
  enlarge right by=-5pt,%
  boxsep=5pt,%
  boxrule=0pt,
  left=0pt,right=0pt,top=0pt,bottom=0pt,
  sharp corners,
  before skip=\topsep,
  after skip=\topsep
}

\crefname{theorem}{Theorem}{Theorems}
\crefname{proposition}{Proposition}{Propositions}
\crefname{lemma}{Lemma}{Lemmas}
\crefname{corollary}{Corollary}{Corollaries}
\crefname{definition}{Definition}{Definitions}
\crefname{assumption}{Assumption}{Assumptions}
\crefname{remark}{Remark}{Remarks}
\crefname{problem}{Problem}{Problems}
\crefname{property}{Property}{property}

\numberwithin{equation}{section}
\numberwithin{theorem}{section}
\numberwithin{proposition}{section}
\numberwithin{definition}{section}
\numberwithin{lemma}{section}
\numberwithin{assumption}{section}
\numberwithin{remark}{section}

\newcommand*{\annot}[1]{\tag*{\footnotesize{\textcolor{black!50}{\big(#1\big)}}}}

\makeatletter
\let\save@mathaccent\mathaccent
\newcommand*\if@single[3]{%
    \setbox0\hbox{${\mathaccent"0362{#1}}^H$}%
    \setbox2\hbox{${\mathaccent"0362{\kern0pt#1}}^H$}%
    \ifdim\ht0=\ht2 #3\else #2\fi
}
\newcommand*\rel@kern[1]{\kern#1\dimexpr\macc@kerna}
\newcommand*\widebar[1]{\@ifnextchar^{{\wide@bar{#1}{0}}}{\wide@bar{#1}{1}}}
\newcommand*\wide@bar[2]{\if@single{#1}{\wide@bar@{#1}{#2}{1}}{\wide@bar@{#1}{#2}{2}}}
\newcommand*\wide@bar@[3]{%
    \begingroup
    \def\mathaccent##1##2{%
        \let\mathaccent\save@mathaccent
        \if#32 \let\macc@nucleus\first@char \fi
        \setbox\z@\hbox{$\macc@style{\macc@nucleus}_{}$}%
        \setbox\tw@\hbox{$\macc@style{\macc@nucleus}{}_{}$}%
        \dimen@\wd\tw@
        \advance\dimen@-\wd\z@
        \divide\dimen@ 3
        \@tempdima\wd\tw@
        \advance\@tempdima-\scriptspace
        \divide\@tempdima 10
        \advance\dimen@-\@tempdima
        \ifdim\dimen@>\z@ \dimen@0pt\fi
        \rel@kern{0.6}\kern-\dimen@
        \if#31
        \overline{\rel@kern{-0.6}\kern\dimen@\macc@nucleus\rel@kern{0.4}\kern\dimen@}%
        \advance\dimen@0.4\dimexpr\macc@kerna
        \let\final@kern#2%
        \ifdim\dimen@<\z@ \let\final@kern1\fi
        \if\final@kern1 \kern-\dimen@\fi
        \else
        \overline{\rel@kern{-0.6}\kern\dimen@#1}%
        \fi
    }%
    \macc@depth\@ne
    \let\math@bgroup\@empty \let\math@egroup\macc@set@skewchar
    \mathsurround\z@ \frozen@everymath{\mathgroup\macc@group\relax}%
    \macc@set@skewchar\relax
    \let\mathaccentV\macc@nested@a
    \if#31
    \macc@nested@a\relax111{#1}%
    \else
    \def\gobble@till@marker##1\endmarker{}%
    \futurelet\first@char\gobble@till@marker#1\endmarker
    \ifcat\noexpand\first@char A\else
    \def\first@char{}%
    \fi
    \macc@nested@a\relax111{\first@char}%
    \fi
    \endgroup
    }
\makeatother

\makeatletter
\newcommand*{\redefinesymbolwitharg}[1]{%
  \expandafter\let\csname ltx#1\expandafter\endcsname\csname #1\endcsname
  \@namedef{#1}{\@ifnextchar{^}{\@nameuse{#1@}}{\@nameuse{#1@}^{}}}%
  \expandafter\def\csname #1@\endcsname^##1##2{%
     \csname ltx#1\endcsname\ifx!##1!\else^{##1}\fi\mathopen{}\mathclose\bgroup\left(##2\aftergroup\egroup\right)
     }%
}
\makeatother
\redefinesymbolwitharg{sin}
\redefinesymbolwitharg{cos}

%% file: 0abstract.tex
We propose a data-efficient, physics-aware generative framework in function space for inverse PDE problems.
Existing plug-and-play diffusion posterior samplers represent physics implicitly through joint coefficient-solution modeling, requiring substantial paired supervision.
In contrast, our Decoupled Diffusion Inverse Solver (DDIS) employs a \emph{decoupled} design: an unconditional diffusion learns the coefficient prior, while a neural operator explicitly models the forward PDE for guidance.
This decoupling enables superior data efficiency and effective physics-informed learning, while naturally supporting \emph{Decoupled} Annealing Posterior Sampling (DAPS) to avoid over-smoothing in Diffusion Posterior Sampling (DPS).
Theoretically, we prove that DDIS avoids the guidance attenuation failure of joint models when training data is scarce.
Empirically, DDIS achieves state-of-the-art performance under sparse observation, improving $l_2$ error by 11\% and spectral error by 54\% on average; when data is limited to 1\%, DDIS maintains accuracy with 40\% advantage in $l_2$ error compared to joint models.

%% file: 1intro.tex
\input{figure1}

Inverse problems are commonplace across science and engineering fields.
Typically, such problems are ill-posed, non-unique, and nonlinear, requiring computationally expensive solvers and hand-crafted knowledge~\citep{tarantola2005inverse}.
Inverse problems governed by partial differential equations (PDEs) aim to infer unknown coefficient fields $a \in \mathcal{A}$ (e.g., coefficients, sources) from partial or noisy observations $u_\mathrm{obs}$. 
Formally, with forward operator $L:\mathcal A\to\mathcal U$, mask operator $M$, and noise $\epsilon$, the observation model is
\begin{align}
\label{eqn:observation_model}
u = L(a), \quad u_\mathrm{obs} = M \odot u + \epsilon.
\end{align}
Applications such as weather forecasting~\citep{manshausen2024generative} and geophysical imaging~\citep{herrmann2012efficient}, operate under \emph{sparse} sensor coverage, where observations are available only on a small fraction of the spatial domain.
Separately, acquiring paired training data $(a,u)$ requires repeatedly solving the underlying PDE~\citep{chen2024data}, leading to an imbalanced regime with abundant coefficients but scarce paired coefficient-solution samples.

\begin{figure*}[t]
\centering
\includegraphics[width=.92\textwidth]{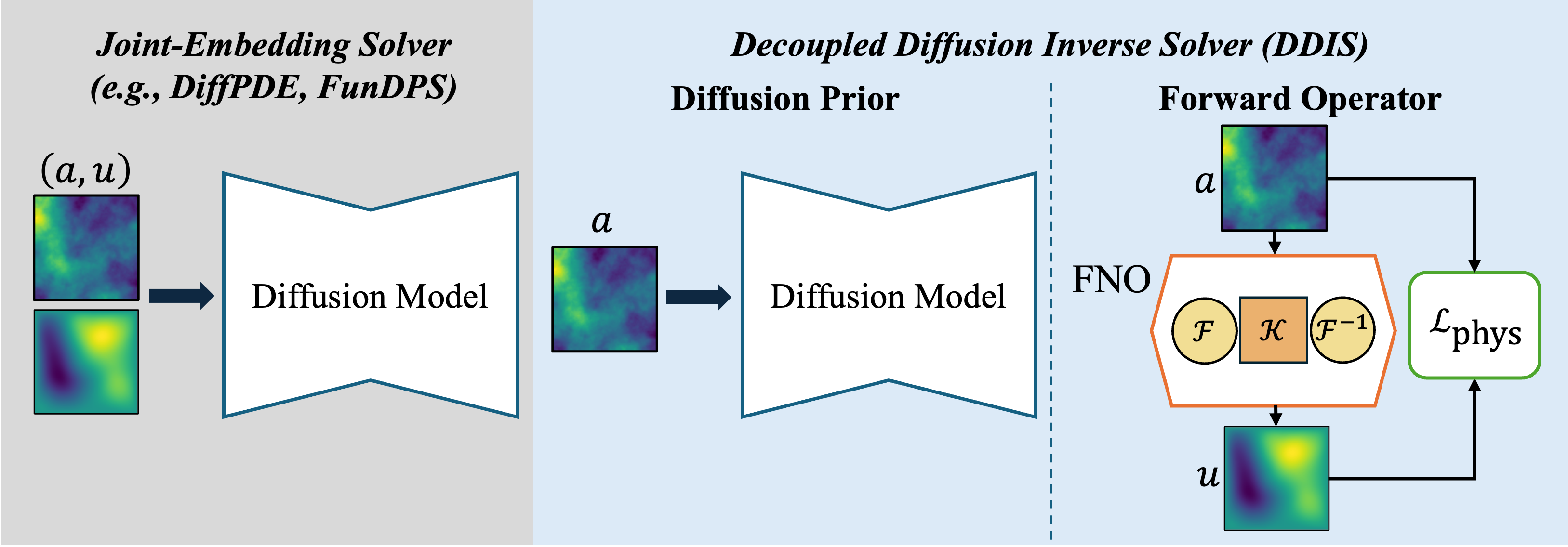}
\caption{
{Training pipeline comparison.} 
{Left (Gray):} Joint-embedding methods rely on \textit{paired} data $(a, u)$ to learn the joint distribution $p(a,u)$.
{Right (Cyan):} DDIS decouples the architecture: the diffusion prior learns $p(a)$ utilizing abundant \textit{unpaired} coefficients, while the neural operator takes paired data $(a, u)$ to directly learn the forward physics map.
}
\label{fig:training_components}
\vspace{-1.5em}
\end{figure*}

Bayesian generative modeling has emerged as a principled paradigm for solving PDE inverse problems by sampling from the posterior $p(a \mid u_\mathrm{obs})$.
Guidance-based methods such as Diffusion Posterior Sampling (DPS~\citep{chung2022diffusion}) evaluate the observation likelihood to steer prior samples toward the posterior during the reverse diffusion process.
Recent frameworks, including DiffusionPDE~\citep{huang2024diffusionpde} and FunDPS~\citep{yao2025guided}, implement this by learning a diffusion prior over the \emph{joint} distribution of coefficients and solutions $p(a,u)$ and apply guidance by masking the observed components.
As a result, joint-embedding models must recover the underlying physics purely from statistical cross-field correlations.
Conceptually, this reduces inverse PDE solving to inpainting problems in the joint space, raising a fundamental question:
\begin{center}
    \vspace{-0.5em}
    \emph{
    Can joint-embedding models effectively support cross-field guidance?
    }
    \vspace{-0.5em}
\end{center}

Our analysis reveals that the answer is no, especially under data scarcity. 
We identify two hurdles that limit the effectiveness of joint embeddings. 
First, when data is scarce, joint models suffer from \emph{guidance attenuation} (\cref{sec:likelihood-attenuation}); we characterize this through geometric criteria (\cref{fig:guidance-geometry}): 
non-vanishing guidance requires at least two training data points in the local neighborhood of current diffusion state, which the curse of dimensionality renders nearly impossible to satisfy.
Second, for advanced samplers like DAPS~\cite{zhang2025improving}, sparse observations cause correlations between observed points and their spatial neighbors to collapse to zero in joint models (\cref{sec:daps_failure}).
These motivate a decoupled design that avoids joint embeddings %
and enforces physics consistency through explicit representations.

We therefore propose {\textbf{Decoupled Diffusion Inverse Solver (DDIS)}}, a modular framework that separates prior modeling from physics-induced likelihood evaluation.

\vspace{-.5em}
\paragraph{Decoupling Physics during Training.}
Our key insight is that priors and physics serve fundamentally different roles:
the prior is defined over the coefficient space, while the likelihood is defined by the PDE solution operator mapping $a$ to $u$.
Modeling the physics via a forward operator is more data-efficient than learning it implicitly through joint embeddings.
Motivated by this decoupling, DDIS
(i) learns the coefficient prior using a diffusion model in function space, and
(ii) uses a neural operator $L_\phi(a)$ to represent the forward physics and evaluate the likelihood (\cref{fig:training_components}).
This enables data-efficient learning by exploiting abundant coefficient prior samples for prior modeling and limited paired data for operator training.

\vspace{-.5em}

\paragraph{Decoupling Sampling during Inference.}
Posterior sampling requires combining the learned prior with the physics-induced likelihood.
Unbiased posterior samplers such as \emph{Decoupled} Annealing Posterior Sampling (DAPS~\cite{zhang2025improving}) require dense guidance signals, which joint-embedding models fail to provide under sparse sensor coverage.
In contrast, the neural operator $L_\phi$ in DDIS propagates sparse observations in the solution space into dense guidance over the coefficient space.
(See \cref{fig:ddis_comparison}.)
This enables the effective use of DAPS during inference, avoiding the over-smoothing artifacts observed in prior DPS-based models.

\begin{figure*}[t]
    \centering
    \includegraphics[width=0.95\linewidth]{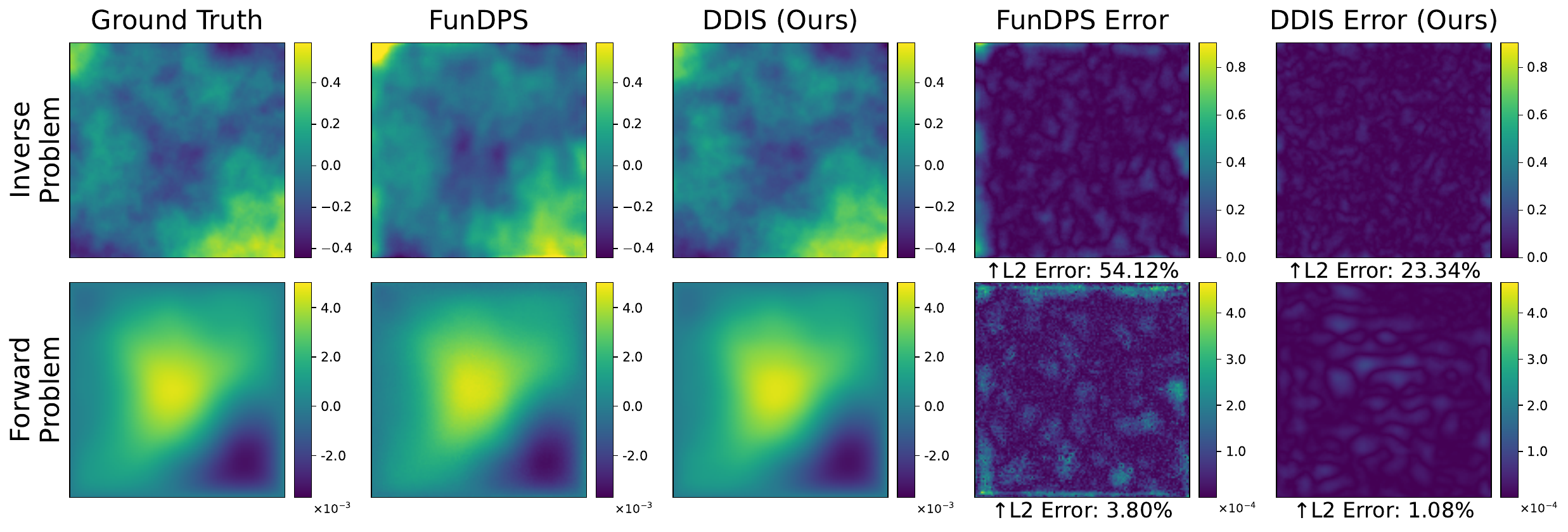}
    \vspace{-1em}
    \caption{Comparison of FunDPS and our DDIS on inverse Poisson problem reconstruction under 1\% paired data scarcity.
    FunDPS suffers from over-smoothing due to Jensen’s gap
    while DDIS achieves sharp and dense guidance with improved accuracy.
    }
    \label{fig:ddis_comparison}
    \vspace{-1em}
\end{figure*}

\paragraph{Advantages.} 
\vspace{-.5em}
DDIS's key advantages include:
\begin{itemize}
    \item \textbf{Data-efficient Learning.}
    DDIS leverages operator learning to bypass the limitations of joint embedding, thereby dominating the Pareto front of the accuracy-data tradeoff and sustaining performance under extreme data scarcity.
    \item \textbf{Physics Integration.}
    DDIS uses a neural operator surrogate to explicitly represent the forward PDE mapping and enforce physics consistency by design. 
    Incorporating physics loss during surrogate training is significantly more robust than previous attempts. Our physics-informed variant performs on par with the 100\%-data variant with only 1\% of paired data.
    
    \item \textbf{Theoretical Justification.}
    We provide theoretical justification for the decoupled design of DDIS by analyzing failure modes of joint-embedding diffusion models.
    In particular, we show that under precise geometric conditions,
    (i) joint-embeddings suffer guidance attenuation,
    (ii) DAPS fails due to sparse-guidance collapse,
    and (iii) DDIS avoids these limitations.
    \item \textbf{Empirical Performance.}
    DDIS achieves state-of-the-art accuracy and runtime efficiency on three challenging inverse PDE problems under sparse supervision, improving $\ell_2$ error by 11\% and spectral error by 54\% on average; the advantage increases to 40\% under data scarcity. 

\end{itemize}

\vspace{-.5em}
\paragraph{Organization.}
\cref{sec:preliminary} introduces preliminaries.
\cref{sec:ddis} presents the proposed Decoupled Diffusion Inverse Solver (DDIS).
\cref{sec:theory} analyzes joint-embedding failures and advantages of decoupling.
\cref{sec:exp} reports benchmark experimental results.
Related work is deferred to \cref{sec:related}.
Design rationales, ablations, and supplementary experiments appear in \cref{sec:design,sec:ablation_study}.

%% file: figure1.tex
\begin{wrapfigure}{r}{0.48\textwidth}
\centering
\vspace{-1.25em} %
\begin{tikzpicture}[
  >=Stealth,
  line cap=round,
  line join=round,
  font=\small,
  scale=0.8,
  every node/.style={scale=0.72}
]

\tikzset{
  xt/.style={circle, draw=blue, thick, inner sep=1.3pt},
  gaussblue/.style={
    shading=radial,
    inner color=blue!25,
    outer color=blue!5,
    draw=black!50,
    thick
  },
  gaussred/.style={
    shading=radial,
    inner color=red!25,
    outer color=red!5,
    draw=black!50,
    thick
  }
}

\def\r{1.0}

\coordinate (d2) at (5.7,1.3);
\coordinate (d3) at (2.7,1.9);

\begin{scope}
  \clip (0,0) rectangle (7.5,3.5);
  \foreach \px in {0.0,0.10,...,7.5} {
    \foreach \py in {0.0,0.10,...,3.5} {
      \pgfmathsetmacro{\valA}{exp(-((\px-5.7)^2+(\py-1.3)^2)/(2*0.7^2))}
      \pgfmathsetmacro{\valB}{exp(-((\px-2.7)^2+(\py-1.9)^2)/(2*0.7^2))}
      \pgfmathsetmacro{\val}{\valA+\valB}
      \pgfmathsetmacro{\intensity}{min(50*\val,100)}
      \ifdim\intensity pt>4pt
        \fill[blue!\intensity, opacity=0.9] (\px,\py) circle (0.05);
      \fi
    }
  }
\end{scope}

\fill (d2) circle (1.5pt);
\fill (d3) circle (1.5pt);

\coordinate (xI) at (0.4,3.0);
\node[circle, draw=black, thick, inner sep=1.5pt] at (xI) {};
\node[red, above right=3pt, font=\bfseries, anchor=south west] at (xI) {\textbf{case (i)}};
\draw[->, thick, black, dashed] (xI) -- ++(1.4,0);
\draw[->, thick, black] (xI) -- ++(0,-0.4);

\coordinate (xII) at (2.0,1.4);
\node[circle, draw=black, thick, inner sep=1.5pt] at (xII) {};
\node[red, above left=3pt, font=\bfseries, anchor=south east] at (xII) {\textbf{case (ii)}};
\draw[->, thick, black, dashed] (xII) -- ++(1.4,0);
\draw[->, thick, black] (xII) -- ++(0,-0.4);

\coordinate (xIII) at (4.0,1.6);
\node[circle, draw=black, thick, inner sep=1.5pt] at (xIII) {};
\node[green!60!black, above=3pt, font=\bfseries, anchor=south] at (xIII) {\textbf{case (iii)}};
\draw[->, thick, black, dashed] (xIII) -- ++(1.4,0);
\draw[->, thick, black] (xIII) -- ++(0,-1.0);

\begin{scope}[shift={(6,2.7)}]
  \draw[black!40, rounded corners, fill=white]
    (-1.6,-0.05) rectangle (1.45,1.35);

  \node[circle, draw=black, thick, inner sep=1.5pt] at (-1.3,1.15) {};
  \node[anchor=west] at (-0.9,1.15) {current state $x_t$};

  \fill (-1.3,0.8) circle (1.5pt);
  \node[anchor=west] at (-0.9,0.8) {training examples};

  \draw[->, thick, black, densely dashed] (-1.45,0.45) -- ++(0.45,0);
  \node[anchor=west] at (-0.9,0.45) {update on $u$};

  \draw[->, thick, black] (-1.3,0.3) -- ++(0,-0.3);
  \node[anchor=west] at (-0.9,0.1) {update on $a$};
\end{scope}

\draw[->] (0,0) -- (7.4,0) node[right] {$u$};
\node[anchor=north east] at (7.6,0) {(solution field, observed)};
\draw[->] (0,0) -- (0,3.5) node[above] {$a$};
\node[anchor=south, rotate=90] at (0.0,1.75) {(coefficient field, unknown)};

\end{tikzpicture}
\vspace{-1em} %
\caption{
\small{
\textbf{Gradient guidance by joint-embedding models vanishes under scarce paired data.}
Blue blobs visualize regions supported by the learned joint model around individual training examples (black dots), and the circle marker denotes diffusion state $x_t$.
We consider three cases: when $x_t$ is far from all blobs \textcolor{red}{(i)}, near a single training sample \textcolor{red}{(ii)}, or close to multiple training samples \textcolor{green!60!black}{(iii)}.
In all cases, the update on $u$ (dashed arrows) remains valid; however, the update on $a$ (solid arrows) vanishes in (i) and (ii), and is nonzero only in (iii).
In high dimensions, case (iii) is rare, rendering coefficient-space guidance ineffective.
}
}
\label{fig:guidance-geometry}
\vspace{-1.25em}
\end{wrapfigure}

%% file: 2preliminary.tex
\vspace{-.25em}

We employ diffusion models~\citep{ho2020denoising,song2020denoising,song2020score} to learn a prior distribution $p(a)$ over the coefficient space.
The diffusion model defines forward and reverse stochastic processes that enable sampling from $p(a)$ by denoising from Gaussian noise.
We refer readers to \cref{appendix:diffusion_background} for a detailed exposition of diffusion models and the diffusion prior.

Given $a = a_0$, the observation model~\eqref{eqn:observation_model} specifies a likelihood 
$p(u_{\mathrm{obs}}  \mid  a_0)$.
Combining with the diffusion prior, our goal is to sample from the posterior distribution
\begin{align}
\label{def:posterior}
p(a_0  \mid  u_{\mathrm{obs}})
~\propto~
p(a_0)\,p(u_{\mathrm{obs}}  \mid  a_0),
\end{align}
where $p(a_0)$ is the prior and $p(u_{\mathrm{obs}} \mid  a_0)$ the likelihood.
Posterior sampling requires modifying the prior score $\nabla_{a_t}\log p(a_t)$ in the reverse process \eqref{eqn:std_reverse_process} to the posterior score $\nabla_{a_t}\log p(a_t \mid  u_{\mathrm{obs}})$.
Yet, direct posterior score computation is intractable due to the dependence on the unknown 
$p(a_0  \mid  a_t)$. 
Thus, practical posterior samplers rely on approximations or
auxiliary latent transitions~\citep{kawar2021snips,kawar2022denoising,dou2024diffusion,wu2024principled}.
Before introducing our approach, we briefly review two representative
methods: DPS~\cite{chung2022diffusion} and DAPS~\cite{zhang2025improving}.

\textbf{Diffusion Posterior Sampling (DPS)}
DPS decomposes the posterior score $\nabla_{a_t} \log p(a_t  \mid  u_\mathrm{obs})$ into the sum of the unconditional score and a likelihood gradient.
Since the likelihood $p(u_\mathrm{obs}  \mid  a_t)$ requires marginalizing over the unknown conditional $p(a_0  \mid  a_t)$, DPS approximates it by a mean estimate $\mathbb{E}[a_0 \mid a_t]$.
The approximation introduces Jensen gap and results in over-smoothed reconstructions, as shown in \cref{remark:DPS_jensen,fig:ddis_comparison}.
A full derivation of DPS is provided in \cref{appendix:dps}.

\textbf{Decoupled Annealing Posterior Sampling (DAPS)}
Posterior sampling is naturally formulated as a correction on the noisy latent variable $a_t$, yet the likelihood term $p(u_{\mathrm{obs}}\mid a_0)$ is defined on the clean variable $a_0$.
DPS forces the correction to act on $a_t$, facing the intractability issue of $p(a_0\mid a_t)$.
DAPS instead separates these roles for updating $a_t$ to $a_{t-1}$: 
it first estimates the clean variable $a_0$,
then applies the likelihood-based correction to this clean estimate, 
and finally re-noises the corrected estimate to $a_{t-1}$.
DAPS avoids the intractability issue in DPS and improves reconstruction quality in many inverse problems \cite{zheng2025inversebench}.
However, under sparse observations and joint-embeddings, the likelihood gradients act only locally and offer ineffective update (\cref{remark:daps_sparse}).
We provide the details of its sparse-guidance failure in \cref{sec:daps}.

%% file: 3ddis.tex
\begin{table*}[t]
\centering
\setlength{\tabcolsep}{6pt}
\renewcommand{\arraystretch}{1.1}
\caption{Comparison between DDIS and joint-embedding diffusion methods.}
\label{tab:ddis_vs_joint}
\vspace{-0.5em}
\begin{tabularx}{\textwidth}{L Y Y}
\toprule
 & \textbf{DDIS} & \textbf{Joint-Embedding Models} \\
\midrule
\multicolumn{3}{l}{\textbf{Training Stage}} \\

Architecture
& $G_{\theta}: \mathcal{Z}\rightarrow \mathcal{A}$ and
  $L_{\phi}: \mathcal{A}\rightarrow \mathcal{U}$
& $H_{\psi}: \mathcal{Z}\rightarrow \mathcal{A}\times \mathcal{U}$\\

Objective
& model $p(a)$ and learn $L(\cdot)$
& model $p(a,u)$ \\

Data Utilization
& prior:\ $n_u+n_p$;\ \ operator:\ $n_p$
& $n_p$ only \\

Complexity$^{\ddagger}$
& prior:\ $d_P$;\ \ operator:\ $d_L$
& $d_J \geq \max(d_P,d_L)$ \\

Generalization Bound$^{\ddagger}$
& \small$\displaystyle \widetilde{\mathcal{O}}\!\left(
   \sqrt{\frac{d_L}{n_p}}
   +
   \sqrt{\frac{d_P}{n_p+n_u}}
   \right)$
& \small$\displaystyle \widetilde{\mathcal{O}}\!\left(
   \sqrt{\frac{d_J}{n_p}}
   \right)$ \\

Physics
& explicit via $L_{\phi}(a)$
& statistical correlations \\

\midrule
\multicolumn{3}{l}{\textbf{Posterior Sampling Stage}} \\

Target
& $p(a \mid u_{\mathrm{obs}})$
& $p(a \mid u_{\mathrm{obs}})$ via marginalizing out $u$ \\

Sampler
& DAPS~\cite{zhang2025improving}
& DPS~\cite{chung2022diffusion} \\

Likelihood
& $\log p(u_{\mathrm{obs}}\mid \hat a_{0|u})$
& $\log p(u_{\mathrm{obs}}\mid \mathbb{E}[a_0\mid a_{t+1}])$ \\

Bias
& asymptotically unbiased$^{\dagger}$
& Jensen-gap bias\\

Spectral Feature &
preserves high-frequency details &
degraded at high frequencies \\

Inference Cost
& low, diffusion forward + FNO backward
& high, diffusion forward and backward \\

\bottomrule
\multicolumn{3}{@{} p{\textwidth} @{}}{\footnotesize $^\dagger$ Asymptotic with infinite Langevin steps and paired data; DAPS also uses Gaussian approximation in practice. $^\ddagger$ Detailed derivation in \cref{sec:data-eff}.}
\end{tabularx}
\vspace{-2em}
\end{table*}

DDIS separates the roles of priors and physics-induced likelihoods in a Bayesian inverse formulation. Instead of forcing a diffusion model to represent coefficient-solution correlations directly from paired data, DDIS learns these components independently in their native domains and integrates them only during sampling.
DDIS consists of:
(i) diffusion prior learning in coefficient space,
(ii) neural operator learning of the forward physics, and
(iii) DAPS-based physics-aware posterior sampling (at inference).
\cref{fig:training_components} compares (i)+(ii) to previous joint-learning methods, and \cref{fig:pipeline_inference} visualizes (iii).
Moreover, each component offers immediate benefits:
(i) scalable prior learning using abundant prior-only data,
(ii) explicit physics consistency, and
(iii) dense guidance in posterior sampling, even under sparse observations.
We elaborate on each component below.

\subsection{Diffusion Prior in Coefficient Space}
\label{sec:diffusion_prior}
We train a diffusion prior $p(a)$ over PDE coefficients using a score-based diffusion model
$\boldsymbol{s}_\theta(a_t, t)$ based on noisy $a_t$ samples.
Training follows standard noise-prediction loss without paired supervision:
\begin{align}
\mathcal{L}_\text{prior} = \mathbb{E}_{a,\epsilon,t} \bigl[ \| \boldsymbol{s}_\theta(a_t,t) - \epsilon \|_2^2 \bigr].
\end{align}

\vspace{-.5em}
\subsection{Neural Operator as a Physics Surrogate} 
\label{sec:surrogate_learning}

Given scarce paired $(a,u)$ samples, we learn a surrogate of the forward map $L:\mathcal{A}\to\mathcal{U}$ using a neural operator $L_\phi$. The model is trained via supervised regression with an optional physics regularizer
\vspace{-1em}
\begin{align}
\label{eqn:no_loss}
\mathcal{L}_\text{operator} = \mathbb{E}_{(a,u)} \bigl[ \|L_\phi(a) - u\|_2^2 + \overbrace{\lambda\|\text{Res}(L_\phi(a),a)\|_2^2}^{\text{PINO regularization}}
\bigr],
\end{align}
where $\text{Res}(\cdot)$ evaluates PDE residual, weighted by $\lambda$.

The forward PDE problem $L:\mathcal{A}\to\mathcal{U}$ is well-posed, making it a natural target for neural operator learning~\citep{kovachki2023neural}. 
Unlike joint-embedding diffusion models that must reconstruct the full coefficient-solution correlations, the neural operator is trained only to represent the forward physics. 
This reduces reliance on paired data and shifts learning complexity into the coefficient prior. 
Operating directly in function space, neural operators are resolution-invariant and can exploit low- or multi-resolution paired data while supporting high-resolution inference. 
Moreover, the optional PDE-residual regularization term in \eqref{eqn:no_loss} further reduces the need for paired supervision, potentially enabling purely physics-based training.
Collectively, these characteristics enable DDIS to maximize data utilization and achieve data efficiency in multiple regimes.

\subsection{Physics-Aware Posterior Sampling via DAPS}
\label{sec:ddis_daps}

\begin{figure*}[t]
    \centering
    \includegraphics[width=0.95\textwidth]{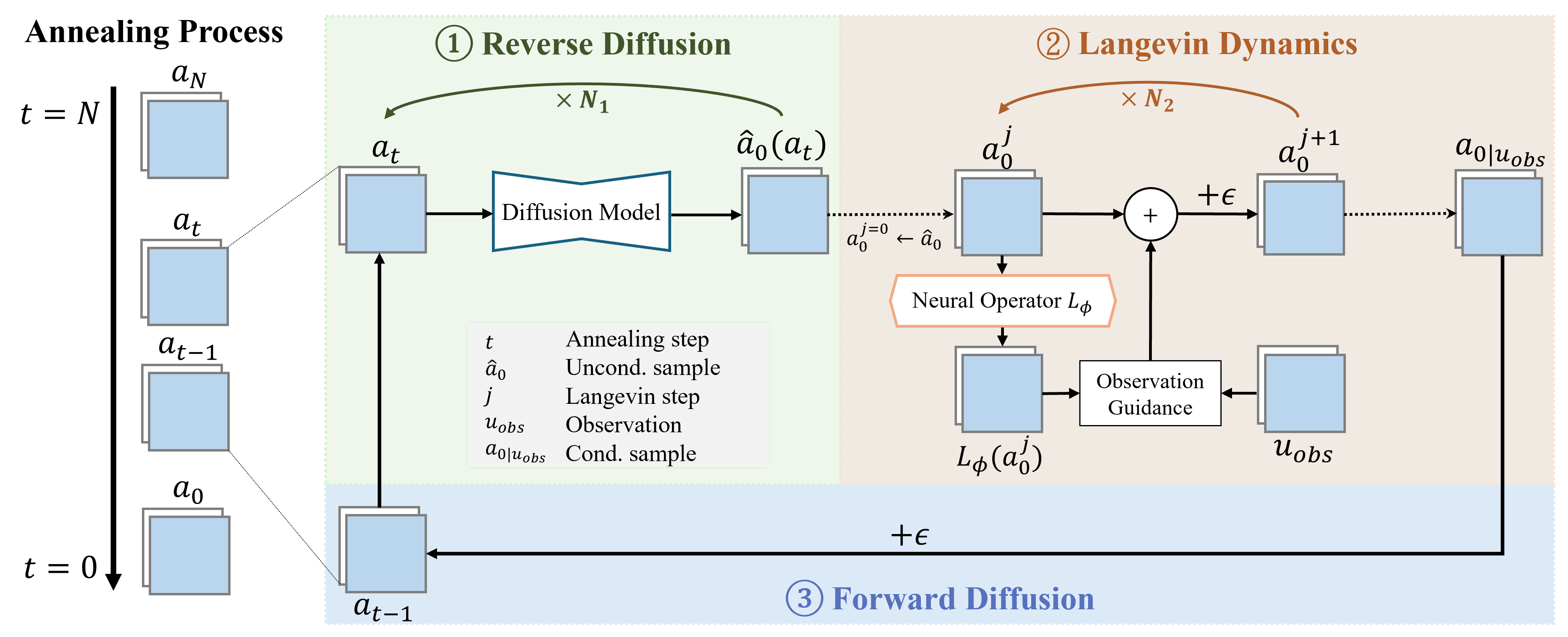}
    \caption{DDIS sampling process. 
    Each annealing step alternates between \textbf{\ding{172} reverse diffusion}, which estimates an unconditional sample $\hat{a}_0(a_t)$ by diffusion model;
    \textbf{\ding{173} Langevin dynamics} guided by the neural operator $L_\phi$ to enforce physics consistency; 
    and \textbf{\ding{174} forward diffusion}, which re-injects noise for the next annealing level. 
    The process iteratively refines the posterior sample $a_0 \mid u_{\mathrm{obs}}$.}
    \label{fig:pipeline_inference}
    \vspace{-2em}
\end{figure*}

Given sparse observations $u_\mathrm{obs} = M \odot L(a) + \epsilon$, DDIS performs posterior inference using DAPS (\cref{sec:daps}).
DAPS operates on the observation-conditioned time marginals
$p_t(a_t \mid u_\mathrm{obs})$ along the reverse diffusion process.
As the noise level $t$ decreases, these marginals anneal toward the desired posterior $p(a_0 \mid u_\mathrm{obs})$ at $t=0$.

At each diffusion timestep, DAPS combines the diffusion prior with likelihood-based guidance to approximate sampling from $p_t(a_t \mid u_\mathrm{obs})$.
In DDIS, this guidance is instantiated through the neural operator surrogate $L_\phi$ and the denoised estimate $\hat a_0(a_t)$ induced by the diffusion prior.
Concretely, given the current latent $a_t$, DDIS applies Langevin MCMC loop~\eqref{eqn:langevin_daps}:
\begin{align}
\label{eqn:langevin_surrogate}
\vspace{-.5em}
\small
a_{0}^{(j+1)} =& a_{0}^{(j)} - \eta\nabla_{a_{0}^{(j)}} \frac{\|a_{0}^{(j)} - \hat a_0(a_t)\|^2}{r_t^2} \\ 
&- \eta\nabla_{a_{0}^{(j)}} \frac{\| M \odot L(a_{0}^{(j)}) - u_\mathrm{obs} \|^2}{2\beta_y^2} + \sqrt{2\eta}\epsilon_{j}, \notag
\end{align}
where $\epsilon_j \sim \mathcal{N}(0, I)$, $\eta>0$ is step size, $r_t$ and $\beta_y$ are prior and likelihood scale.
The first term enforces consistency with the diffusion prior, while the second injects physics-aware guidance via the neural operator. 
After $N$ Langevin steps, DAPS samples the next latent $a_{t-1} \sim p_{t-1}(a_{t-1} \mid u_\mathrm{obs})$ by re-noising $a_0^N$ according to \eqref{eqn:renoise_daps}.
Iterating this procedure yields a trajectory towards the desired posterior.
\cref{alg:ddis_daps} shows the complete sampling pipeline.

The neural operator plays a critical role in enabling effective DAPS-based posterior sampling under sparse observations.
While prior joint-embedding models suffer from sparse-guidance failure (\cref{remark:daps_sparse}), the global receptive field of $L_\phi$ propagates observation errors across the entire coefficient space, ensuring dense guidance.
As a result, DDIS produces sharp, effective, and physics-aware samples without retraining the diffusion prior for new observation patterns.
\cref{sec:design} further discusses design rationale and choices.

%% file: 4theory.tex
This section serves two purposes: 
(i) to formalize the guidance attenuation phenomenon that motivates decoupled designs, and 
(ii) to provide a comparison between joint and decoupled formulations, explaining the empirical results in \cref{sec:exp}.
We begin by analyzing guidance behavior under data scarcity for joint-embedding and decoupled models (\cref{sec:likelihood-attenuation}).
Particularly, the guidance in joint models provably attenuates (\cref{sec:joint-attenuation-short}), whereas our DDIS preserves robust guidance (\cref{sec:ddis_guidance}).
We further analyze a naive application of DAPS in joint models and prove its sparse-guidance failure under sparse observations (\cref{sec:daps_failure}).
We also show that the decoupled design offers a tighter generalization bound than joint-embedding in \cref{sec:data-eff}.

\subsection{Guidance Behaviors under Data Scarcity}
\label{sec:likelihood-attenuation}

Assume observations are generated by a physical operator with additive homoscedastic noise:
\begin{align}
\label{eq:physical_obs_model}
u_{\mathrm{obs}} = L^*(a) + \varepsilon,
\qquad
\varepsilon \sim \mathcal N(0,\sigma_{\mathrm{obs}}^2 I),
\end{align}
where $L^*:\mathcal A \to \mathcal U$ is the true forward operator and $\sigma_{\mathrm{obs}}^2$ is a fixed
sensor noise variance.
Thus, the ideal inference-time guidance is given by the likelihood score
\vspace{-0.25em}
\begin{align}
\small
\nabla_a \log p(u_{\mathrm{obs}} | a)
\!=\!
\frac{1}{\sigma_{\mathrm{obs}}^2}
\big(\nabla_a L^*(a)\big)\!^\top
\big(u_{\mathrm{obs}} \!-\! L^*(a)\big),
\label{eqn:true-likelihood-score}
\end{align}
\vspace{-0.25em}
which defines a non-zero gradient on the function space $\mathcal A$,
with fixed noise scale $1/\sigma_{\mathrm{obs}}^2$.

\subsubsection{Guidance Attenuation in Joint-Embedding Models}
\label{sec:joint-attenuation-short}

We first summarize why guidance necessarily attenuates in joint-embedding diffusion models under data scarcity; full derivations are deferred to \cref{sec:joint-attenuation}.

Joint-embedding methods learn a diffusion prior over the joint variable
$x=(a,u)$ and impose observations with
\vspace{-1em}
\begin{align*}
u_{\mathrm{obs}} = M x_0 + \varepsilon, \qquad 
\varepsilon \sim \mathcal N(0,\tilde\sigma_{\mathrm{obs}}^2 I).
\end{align*}
For Diffusion Posterior Sampling (DPS), the likelihood gradient with respect
to the noisy state $x_t$ is approximated using a denoised estimate
$\hat x_0(x_t,t)=\mathbb E[x_0\mid x_t]$,
\begin{align}
\nabla_{x_t}\log p(u_{\mathrm{obs}}\mid x_t)
\approx
\frac{1}{\tilde\sigma_{\mathrm{obs}}^2}
J_{\hat x_0}(x_t,t)^\top M^\top
\big(u_{\mathrm{obs}}-M\hat x_0(x_t,t)\big).
\end{align}
We define the \emph{scale-free guidance} (\cref{def:joint-guidance}):
\begin{align}
g(x_t,t)
\;\coloneqq\;
J_{\hat x_0}(x_t,t)^\top M^\top r(x_t,t),
\quad
r=u_{\mathrm{obs}}-M\hat x_0.
\end{align}

Using the Tweedie estimator for score-based diffusion,
$\hat x_0(x_t,t)=\alpha_t^{-1}(x_t+\sigma_t^2 s_\theta(x_t,t))$,
the guidance admits the block decomposition (\cref{lemma:block-guidance}):
\begin{align}
\small
g(x_t,t)
\propto
\begin{pmatrix}
\sigma_t^2\,\partial_{a_t}s_{\theta,u}(x_t,t)^\top\\[2pt]
I+\sigma_t^2\,\partial_{u_t}s_{\theta,u}(x_t,t)^\top
\end{pmatrix}
r(x_t,t).
\end{align}
The coefficient update $g_a$ depends on the cross-partial
$\partial_{a_t}s_{\theta,u}$; thus, observations can influence $a$ only through learned $a$-$u$ coupling in the joint score model.

Under data scarcity, the learned score is well approximated by an empirical
Gaussian mixture (\cref{lemma:empirical-score}).
Let $\varphi_n(x)$ denote the $n$-th component density and
$w_n(x)=\varphi_n(x)/\sum_j\varphi_j(x)$ its responsibility.
Since a single isotropic Gaussian cannot couple $a$ and $u$ (\cref{rem:responsibility-source}), all cross-component
guidance arises solely through the responsibility gradients $\partial_a w_n$
(\cref{lemma:cross-block-from-w,lemma:single-gaussian-cross}).
We therefore analyze when these gradients vanish or remain nontrivial.
\begin{theorem}[Local dominance $\Rightarrow$ vanishing responsibility gradients (informal \cref{thm:local-dominance-gradw})]
\label{thm:local-dominance-gradw-informal}
Fix $(x,t)$.
If there exists $k\in[N]$ such that
$
\small
\varphi\big(x-\alpha_t x_0^{(k)}\big)
\gg
\varphi\big(x-\alpha_t x_0^{(j)}\big),
\; \forall j\neq k,
$
then all responsibility gradients vanish:
$
\small
\partial_a w_n(x,t)\approx 0,
\; \forall n\in[N].
$
\end{theorem}
This yields a sufficient condition for guidance attenuation.
\begin{corollary}[Locality implies attenuation (informal \cref{cor:local-dominance-attenuation})]
\label{cor:local-dominance-attenuation-informal}
If $x_t$ lies close to a single mixture center $\alpha_t x_0^{(k)}$, the $a$-component of the scale-free guidance attenuates by
$g_a(x_t,t)\approx 0$.
\end{corollary}
The first result shows that, under the locality condition, responsibility gradients sufficiently vanish.
We now state the complementary: nontrivial responsibility gradients can occur only
when $x$ lies in an overlap region of the mixture (informal \cref{thm:dw-implies-overlap}).
\begin{theorem}[Non-vanishing responsibility gradients require overlap]
\label{thm:dw-implies-overlap-informal}
Fix $(x,t)$.
If
$
\big\|\partial_a w_n(x,t)\big\|>0
\; \text{for some } n\in[N],
$
then there exist $p\neq q$ such that
$
\small
\Big|
\|x-\alpha_t x_0^{(p)}\|_2^2
-
\|x-\alpha_t x_0^{(q)}\|_2^2
\Big|
\lesssim
\sigma^2(t).
$
\end{theorem}
We then have a necessary condition for non-zero guidance.
\begin{corollary}[Non-vanishing guidance requires overlap (informal \cref{cor:nonvanishing-requires-overlap})]
\label{cor:nonvanishing-requires-overlap-informal}
If the $a$-component of the scale-free guidance is nonzero $\|g_a(x_t,t)\|>0$,
then $x_t$ must lie in an overlap region of the mixture.
\end{corollary}

\cref{thm:local-dominance-gradw-informal,thm:dw-implies-overlap-informal} offer a clear geometric interpretation of guidance behavior in joint-embedding diffusion models during posterior sampling.
Guidance attenuates when $x$ lies far from all mixture centers and also when $x$ is dominated by a single component (\cref{cor:local-dominance-attenuation-informal}).
Non-vanishing guidance is possible only when $x$ lies in an overlap region where at least two mixture components have comparable responsibility (\cref{cor:nonvanishing-requires-overlap-informal}).
We provide a visualization for the two criteria in \cref{fig:guidance-geometry}.
Under data scarcity, such overlap regions may be rare or entirely absent, causing coefficient-space guidance to collapse throughout the sampling process.
This failure mode is intrinsic to joint-embedding formulations and cannot be remedied by alternative noise schedules or sampling heuristics.
These observations motivate a decoupled formulation in which observation consistency is enforced through an explicit forward operator rather than learned joint correlations.

\subsubsection{Guidance Robustness in DDIS}
\label{sec:ddis_guidance}

Unlike joint-embedding formulations, DDIS enforces observation consistency through a neural operator rather than learned joint correlations.
As a result, likelihood-based guidance in DDIS is mediated directly by the Jacobian of a deterministic operator and does not rely on data-dependent coupling between coefficients and solutions.

DDIS approximates the physical operator by a neural operator $L_\phi \approx L^*$ and performs posterior sampling via DAPS, where likelihood corrections are applied to clean-variable estimates $a_0^{(j)}$.
For a noise scale $\tilde{\sigma}_{\mathrm{obs}}$, the resulting guidance approximates \eqref{eqn:true-likelihood-score} by 
\begin{align}
\frac{1}{\tilde{\sigma}_{\mathrm{obs}}^2}
\big(\nabla_{a} L_\phi(a_0^{(j)})\big)^\top
\big(u_{\mathrm{obs}} - L_\phi(a_0^{(j)})\big),
\label{eqn:ddis-like-grad-daps}
\end{align}
which directly approximates the ideal likelihood score~\eqref{eqn:true-likelihood-score}.

Here, the magnitude of this guidance is governed by the explicit Jacobian $\nabla_a L_\phi(a)$.
Since $L_\phi$ is trained via regression independently of any joint data density, data scarcity affects approximation accuracy but does not induce guidance attenuation, unlike joint-embedding models.

\begin{proposition}[Structural robustness of DDIS guidance]
\label{prop:ddis-robust-guidance}
In contrast to joint-embedding models, DDIS admits no data-dependent mechanism that forces guidance attenuation as the number of paired data decreases.
\end{proposition}

\vspace{-0.5em}
\subsection{Joint Embeddings and DAPS: Failure Modes with Sparse Guidance}
\label{sec:daps_failure}
\vspace{-0.5em}
In this section, we analyze a structural failure mode of applying DAPS~\cite{zhang2025improving} with joint embeddings.
The Langevin dynamics inside DAPS create spatially discontinuous gradient updates under sparse observations.
These updates alter the covariance structure of generated samples.
Thus, the sampling processes produce outputs that are out-of-distribution relative to the diffusion model's learned manifold, resulting in poor performance.
Extending DAPS to function spaces does not resolve this issue.
We present an informal derivation here; the rigorous one is in \cref{app:daps-failure}.

The DAPS inner Langevin chain (\cref{eqn:langevin_daps}) targets a density proportional to
$p(x\mid x_t)\,p(u_{\mathrm{obs}}\mid x)$, where $x=(a,u)$ is the joint variable.
We extend DAPS's Gaussian assumptions to function spaces as $p(x \mid x_t) \approx \mathcal{N}(\hat{{x}}_0({x}_t),C)$, where $\hat{{x}}_0$ is an estimator of $x_0$ given $x_t$ and $C$ is a covariance function.
Under sparse observations, it suffices to analyze a single point observation $x_i=c$ (well-separated constraints contribute approximately additively), which we model as
$p(u_{\mathrm{obs}}\mid x) \approx \mathcal{N}(c, \sigma_s^2)$.
Taking the continuous-time limit yields the preconditioned Langevin SDE:
\begin{align}
dx_t = -\Sigma\,\nabla U(x_t)\, dt + \sqrt{2} \Sigma^{1/2} dW_t,
\label{eq:langevin}
\end{align}
where $\Sigma$ is the noise covariance and the potential $U(x)$ combines the GRF prior and the pointwise constraint:
\vspace{-1em}
\begin{align}
U(x) = \frac{1}{2}(x - x_0)^T C^{-1} (x - x_0) + \frac{1}{2\sigma_s^2}(e_i^T x - c)^2.
\label{eq:potential}
\end{align}

\vspace{-1em}
We denote by $\Sigma_\infty$ the stationary covariance of $x_t$.

\begin{theorem}[Sparse constraint induces correlation shrinkage (informal \cref{thm:constrained_terms})]
\label{thm:sparse-constraint-corr-shrink-informal}
Under the dynamics induced by \cref{eq:langevin,eq:potential}, all covariance terms involving the constrained index $i$ are uniformly scaled:
\vspace{-1em}
\begin{align}
\small (\Sigma_\infty)_{ik} = (\Sigma_\infty)_{ki}
= \frac{\sigma_s^2}{\sigma_s^2 + C_{ii}} \, C_{ik,} \quad \forall k
\label{eq:scaled_terms}
\end{align}
\end{theorem}

\begin{corollary}[Strong sparse guidance collapses correlations]
\label{cor:sparse-cov-collapse-informal}
If $\sigma_s^2 \ll C_{ii}$, then 

\center{$(\Sigma_\infty)_{ik} = (\Sigma_\infty)_{ki} \approx 0 \quad \forall k$}.
\end{corollary}

In practice, sparse-guidance methods often set $\sigma_s^2$ to a small value (e.g., $10^{-3}$ \citep{zhang2025improving}), so the constrained point rapidly becomes (nearly) uncorrelated from the rest.

This covariance collapse has a geometric interpretation.
Physical coefficient fields are typically continuous and correlated across space, but under sparse constraints, the constrained location becomes nearly independent of its neighbors, yielding discontinuities that are atypical under the diffusion prior.
As a result, the sample is pushed off the data manifold, which degrades reconstruction quality.

%% file: 5exp.tex
\begin{table*}[t]
\vspace{-2em}
\centering
\small
\begin{minipage}[t]{0.53\textwidth}
\setlength{\tabcolsep}{0.85pt}
\caption{
Standard supervision results for the inverse PDE problems under 3\% sparse observations.
Relative $\ell_2$ error (\%), spectral error $E_s$, and runtime per sample are reported.
Rows are grouped by comparable time budgets.
\cref{fig:pareto} visualizes the table.
}
\label{table:performance}
\vspace{-0.5em}
\begin{tabular}{lccccccc}
\toprule
 \multirow{2}{*}{Method} & \multicolumn{2}{c}{Poisson} &
 \multicolumn{2}{c}{Helmholtz} &
 \multicolumn{2}{c}{N-S} &
 \multirow{2}{*}{\makecell{Time \\ (s)}} \\
\cmidrule(lr){2-3}
\cmidrule(lr){4-5}
\cmidrule(lr){6-7}
 & $\ell_2$ & $E_s$
& $\ell_2$ & $E_s$
& $\ell_2$ & $E_s$
&  \\
\midrule

DiffusionPDE & 74.68 & .566 & 46.10 & .315 & 32.78 & 2.099 & 17.75 \\
FunDPS       & {\ul 19.96} & {\ul .192} & {\ul 17.16} & {\ul .140} & {\ul 8.99} & {\ul .382} & 14.58 \\
DDIS
             & \textbf{15.78} & \textbf{.074} & \textbf{15.08} & \textbf{.044} & \textbf{8.93} & \textbf{.165} & 16.75 \\
\midrule

DiffusionPDE & 35.53 & .281 & 23.86 & .164 & 11.01 & {\ul.221} & 35.63 \\
FunDPS       & {\ul 17.14} & {\ul .135} & {\ul 16.05} & {\ul .157} & {\ul 8.23} & .423 & 28.75 \\
DDIS & \textbf{14.36} & \textbf{.086} & \textbf{14.24} & \textbf{.055} & \textbf{8.02} & \textbf{.178} & 32.67 \\
\midrule

DiffusionPDE & 16.65 & \textbf{.085} & 17.73 & \textbf{.096} & 9.24 & {\ul .219} & 142.62 \\
FunDPS       & {\ul 14.73} & .136 & {\ul 14.14} & .198 & {\ul 7.98} & {.479} & 113.83 \\
DDIS & \textbf{12.32} & {\ul .087} & \textbf{12.20} & {\ul .097} & \textbf{7.81} & \textbf{.188} & 127.42 \\
\midrule

ECI-sampling$^{*}$ & 94.63 & / & 92.83 & / & 42.36 & / & 0.20 \\
ECI-sampling$^{*}$ & 93.47 & / & 93.23 & / & 41.68 & / & 0.38 \\
OFM$^{*}$ & 71.87 & / & 49.60 & / & 37.57 & / & 470.44 \\
OFM$^{*}$ & 47.04 & / & 42.07 & / & 20.98 & / & 4366.34 \\

\bottomrule
\end{tabular}
\parbox{\linewidth}{\footnotesize
}
\end{minipage}%
\hfill
\begin{minipage}[t]{0.455\textwidth}
\setlength{\tabcolsep}{0.84pt}
\caption{Relative $\ell_2$ error (\%) under scarce paired-data supervision (100\%, 5\%, and 1\% data) for inverse PDE problems.
Bold numbers highlight better performance under same data scarcity.
Two models in the same group share comparable time budgets as in \cref{table:performance}.
}
\label{table:data-eff_performance}
\vspace{-.5em}
\begin{tabular}{ccccccccc}
\toprule
\multirow{2}{*}{\makecell{Data \\ Scarcity}} &
\multicolumn{2}{c}{Poisson} &
\multicolumn{2}{c}{Helmholtz} &
\multicolumn{2}{c}{N-S} &
\multirow{2}{*}{\makecell{T \\ (s)}} \\
\cmidrule(lr){2-3}
\cmidrule(lr){4-5}
\cmidrule(lr){6-7}
 & DDIS & Fun$.^\ddagger$ & DDIS & Fun$.^\ddagger$ & DDIS & Fun$.^\ddagger$ & \\
\midrule

100\% & \textbf{16.36} & 20.47 & \textbf{15.19} & 17.16 & \textbf{8.22} & 8.48 & \multirow{4}{*}{16} \\
5\% & \textbf{17.28} & 23.24 & \textbf{15.69} & 20.97 & \textbf{9.21} & 11.56 & \\
1\% & \textbf{{18.70}} & 35.81 & \textbf{{16.40}} & 41.69 & \textbf{12.05} & 13.65 & \\
1\%+Phys & \textbf{16.56} & 35.81 & \textbf{16.05} & 41.69 & / $^{\dagger}$ & / $^{\dagger}$ & \\

\midrule
100\% & \textbf{15.34} & 17.30 & \textbf{14.03} & 16.05 & 8.00 & \textbf{7.67} & \multirow{3}{*}{32} \\
5\% & \textbf{15.76} & 22.66 & \textbf{15.10} & 19.49 & \textbf{9.12} & 11.15 & \\
1\% & \textbf{{17.79}} & 35.79 & \textbf{{15.90}} & 41.44 & \textbf{{12.28}} & 13.50 & \\
1\%+Phys & \textbf{15.63} & 35.81 & \textbf{15.21} & 41.69 & / $^{\dagger}$ & / $^{\dagger}$ & \\

\bottomrule
\end{tabular}
\parbox{\linewidth}{\footnotesize
$^{\dagger}$ Physics loss is not applicable for Navier-Stokes due to insufficient information; see \cref{para:dataset-ns}.
}
\parbox{\linewidth}{\footnotesize
$^{\ddagger}$ Fun. stands for FunDPS.
}
\parbox{\linewidth}{\footnotesize
\vspace{0.25em}
$^{\ast}$ See \cref{sec: A.5} for analysis of the flow models.
}
\end{minipage}
\vspace{-1.5em}
\end{table*}

\paragraph{Task.}

We consider inverse Helmholtz, Poisson, and Navier-Stokes problems with sparse observations.
For each instance, we observe 500 randomly sampled solution points from $\small u(x)\in\mathbb{R}^{128^2}$ ($\sim$3\% of the domain) and reconstruct the underlying coefficient field $a(x)$.
Performance is measured by relative $\ell_2$ error and spectral energy error $E_s$ between reconstructed and ground-truth coefficient.

To reflect practical scenarios of available data, we train the neural operator under three supervision regimes:
(1) \emph{Standard supervision}, where the operator is trained on full-resolution $128^2$ paired data;
(2) \emph{Scarce paired-data supervision}, where the operator is trained using only $5\%$ or $1\%$ of the full dataset; and
(3) \emph{Low-/multi-resolution supervision}, where the operator is trained either on full $64^2$ low-res data, or on a mixed dataset consisting of $64^2$ data and $10\%$ of $128^2$ high-res samples.

\vspace{-1em}
\paragraph{Benchmark.}
We compare DDIS against prior diffusion-based solvers and variants, including DiffusionPDE~\citep{huang2024diffusionpde}, previous state-of-the-art FunDPS~\citep{yao2025guided}, recent flow-based ECI-sampling~\citep{cheng2025gradientfree} and OFM~\citep{shi2025stochastic}. 
For ablations, we test FunDAPS, a baseline where the posterior sampler in FunDPS is replaced by DAPS, and DecoupledDPS, which applies DPS sampling within our modular DDIS framework. 
This setup allows us to isolate the effects of 
(i) the sampling scheme (DPS vs. DAPS) and 
(ii) the architectural choice of decoupling prior learning from operator-informed sampling.

\textbf{Standard supervision.}
We compare our DDIS with the benchmarks DiffusionPDE~\citep{huang2024diffusionpde} and FunDPS~\citep{yao2025guided}. 
Besides relative $\ell_2$ error, we report a spectral error $E_s$ to assess model's ability to reconstruct features across different spatial frequencies, measured by wave number $k$.
Since the power spectrum scales logarithmically as $k$ decreases, directly averaging errors would be heavily biased toward low-frequency modes.
We thus take the geometric mean of the spectral error over all $k$. DDIS achieves state-of-the-art performance on the inverse Poisson, Helmholtz, and Navier-Stokes problems across various budgets (\cref{table:performance}).

\begin{wrapfigure}{r}{0.6\textwidth}
    \centering
    \includegraphics[width=0.6\textwidth]{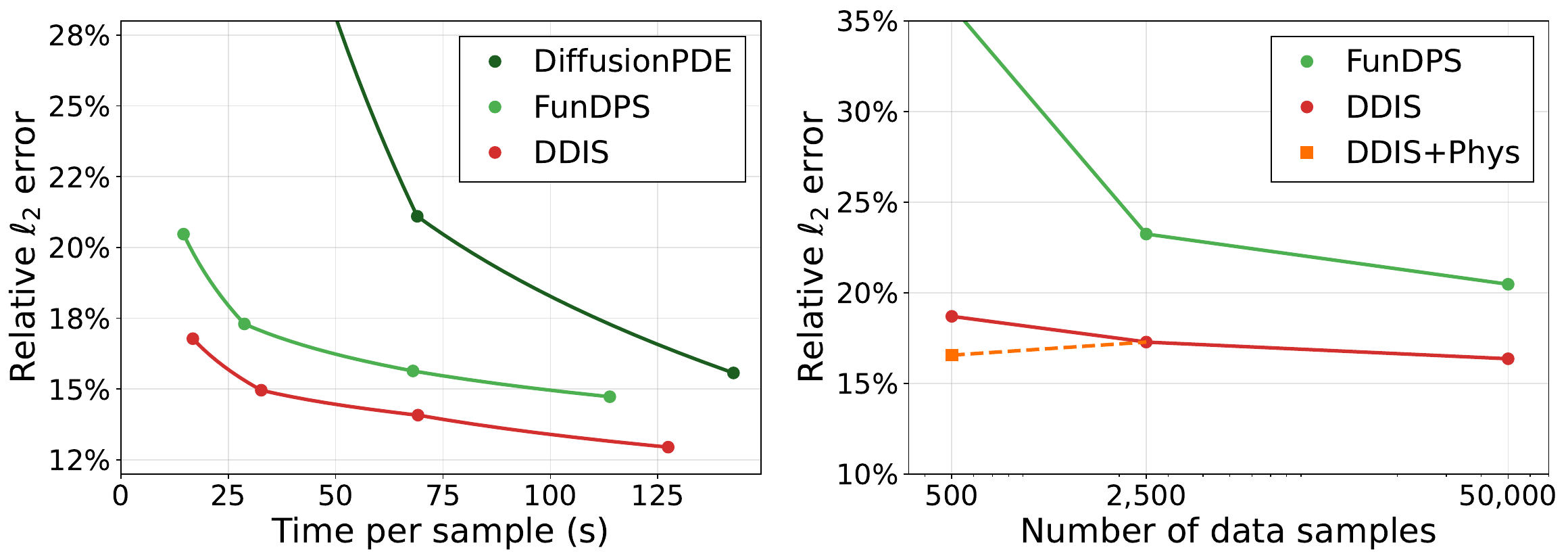}
    \caption{Pareto frontiers for the inverse Poisson problem. (Left) Relative $\ell_2$ error versus inference time. (Right) Relative $\ell_2$ error versus training set size. 
    DDIS dominates baselines across both computational and data efficiency dimensions, with the physics-informed variant (DDIS+Phys) further enhances sample efficiency.}
    \label{fig:pareto}
    \vspace{-1em}
\end{wrapfigure}
As the number of sampling steps increases, DDIS continues to improve, whereas the baseline models saturate.
As shown in \cref{fig:pareto}, DDIS lies near the accuracy-runtime Pareto frontier for both Poisson and Helmholtz.
Across all tasks, including Navier-Stokes where $\ell_2$ differences are small, DDIS achieves substantially lower spectral error than FunDPS by a factor up to $3.2$.
We visualize the spectral error in \cref{sec:power_spec_comp}, showing that FunDPS fails to capture high-frequency features, whereas DDIS preserves fidelity across frequencies.

\textbf{Scarce paired-data supervision.}
To assess the data-efficiency of DDIS, we consider data-scarcity scenarios with $100\%$, $5\%$, and $1\%$ of the original paired-data supervision and FunDPS as the benchmark model, keeping the observation sparsity at around $3\%$.
As paired data become scarce, FunDPS degrades sharply, while DDIS maintains stable reconstruction accuracy even at $1\%$ paired data, indicating superior data efficiency (\cref{table:data-eff_performance}).
We further observe that DecoupledDPS improves over FunDPS (\cref{tab:decoupled_dps}), but DDIS achieves larger gains, highlighting the additional benefit of DAPS enabled by decoupling.
This confirms our claim: a deterministic representation of physics is more appropriate and data-efficient than joint statistical modeling.

\textbf{Low-/Multi-resolution supervision}
\label{sec:multi-res-supervision}
To examine the resolution-invariant properties of DDIS, we evaluate performance when the neural operator is trained on low-resolution ($64^2$) data or on a mixed dataset combining $64^2$ data with 10\% of the $128^2$ samples.

As shown in \cref{table:mixres_inf_performance}, DDIS achieves the best accuracy when trained on full $128^2$ data, while training on low- or mixed-resolution data results in only modest degradation.
This demonstrates DDIS’s robustness to resolution enabled by the resolution-invariant neural operator.

%% file: 6conclusion.tex
We proposed Decoupled Diffusion Inverse Solver (DDIS), a physics-aware generative framework for inverse PDE problems under sparse observations and imbalanced data regimes.
Unlike prior joint-embedding approaches, DDIS separates prior modeling in coefficient space from physics-based likelihood evaluation via a neural operator, aligning with Bayesian inverse formulations (\cref{sec:ddis}).

Our central finding is that joint-embedding models fail to provide effective cross-field guidance under data scarcity or sparse sensor layouts (\cref{sec:theory}).
Specifically, we characterize geometric conditions under which joint models suffer guidance attenuation (\cref{sec:likelihood-attenuation}) and identify a failure mode of joint models with DAPS-based sampling under sparse observation (\cref{sec:daps_failure}).
By decoupling, DDIS provides reliable guidance via the neural operator (\cref{sec:ddis_daps}).

Empirically (\cref{sec:exp}), DDIS achieves state-of-the-art performance.
Across budgets, DDIS creates new accuracy-runtime Pareto frontiers (\cref{fig:pareto})
and attains lower spectral error (\cref{table:performance}).
Under scarce paired-data supervision, DDIS remains stable down to $1\%$ paired data, whereas joint-embedding methods degrade (\cref{table:data-eff_performance}).
DDIS is also robust to low- and mixed-resolution supervision (\cref{table:multi_res}).

%% file: appendix.tex
\clearpage

\section{Design Rationale and Justification}
\label{sec:design}

In this section, we analyze the design choices behind DDIS by addressing questions regarding alternative approaches in \cref{sec:related}. 
We provide theoretical justifications and empirical evidence to explain why DDIS adopts a decoupled generative framework with neural operator surrogates over other common strategies.

\subsection{Why not train conditional diffusion?}
A straightforward approach to inverse problems is to train a conditional diffusion model $p_\theta(a|u_{obs})$ directly. 
While effective for fixed tasks, this ``supervised" strategy suffers from severe limitations in scientific contexts:
\begin{itemize}
    \item \textbf{Combinatorial Generalization Failure:} Conditional models are tied to the specific observation mask $M$ and forward operator $L$ seen during training. 
    Changing the sensor layout (mask) or the physics (operator) requires retraining the entire model from scratch.
    \item \textbf{Data Inefficiency:} Learning the conditional distribution requires massive paired datasets $(a, u_{obs})$ covering all potential measurement configurations. 
    In contrast, DDIS learns the prior $p(a)$ unconditionally from unpaired data, allowing it to generalize to {any} observation mask and operator at inference time without retraining.
\end{itemize}

\subsection{Why not use posterior sampling directly with numerical solver?}
\label{sec:B.2}
In theory, direct posterior sampling strategies could handle arbitrary, sparse observations by masking the exact numerical PDE solver during likelihood evaluation, i.e., computing $\log p(u_{obs} \mid M \odot L_\text{num}(a))$. 
However, this approach is impractical due to:
\begin{itemize}
    \item \textbf{Backpropagation Computational Instability:} 
    Backpropagating gradients through iterative solvers is known to suffer from computational instability and vanishing gradients \citep{zheng2025inversebench} even under dense observation.

    \item \textbf{Violation of Solver Stability Conditions:} 
    The noisy intermediate samples generated during diffusion frequently violate strict stability conditions (e.g., CFL condition \citep{courant1967partial}), causing the solver to diverge or provide unreliable gradients \citep{zheng2025inversebench}.
    
    \item \textbf{Ill-Conditioning under Sparsity:} 
    These instabilities are amplified under sparse observations, where the gradient calculation becomes ill-conditioned \citep{plessix2006review,leung2021level} and the solver fails to propagate observations back into a meaningful coefficient update.
\end{itemize}

In contrast, DDIS adopts learned neural operators in place of numerical solvers, avoiding the above failure modes and offering the following advantages.
\begin{itemize}
    \item \textbf{Differentiability and Stability:} 
    Neural operator provides a smooth, differentiable surrogate. It ensures stable gradient flow for guidance and enables rapid likelihood evaluation ($1000\times$ faster) during the iterative sampling process.

    \item \textbf{Dense Guidance:} 
    Neural operator processes information globally using a spectral basis. 
    This naturally ``smears'' the sparse pointwise errors across the spatial domain, converting sparse observations into the {dense, global guidance} required for stable Langevin updates.
    
    \item \textbf{Resolution Invariance:} 
    As shown in \cref{table:multi_res}, the FNO allows DDIS to be trained on multi-resolution data and sampling at resolutions different from the observation grid.
\end{itemize}

\subsection{Why use DAPS-based posterior sampling?}
We employ DAPS \citep{zhang2025improving} as our sampling backend after ablating other unsupervised posterior sampling strategies. 
Our selection is motivated by the specific limitations of alternative methods in non-linear PDE inverse problems:

\begin{itemize}
    \item \textbf{Inapplicability of Decomposition Methods:} 
    Decomposition-based methods such as DDRM \citep{kawar2022denoising} and DDNM \citep{wang2022zero} rely on the linearity of the forward operator, so they are inapplicable to the non-linear PDE.

    \item \textbf{Bias in Guidance-Based Methods:} 
    Guidance-based methods like DPS \citep{chung2022diffusion} and LGD \citep{song2023loss} rely on approximations of the intractable likelihood score. 
    In practical implementations (e.g., using Tweedie's formula or limited Monte Carlo samples), these approximations introduce Jensen gap (\cref{remark:DPS_jensen}) that results in higher error rates (\cref{tab:decoupled_dps}) and over-smoothed reconstructions lacking high-frequency physical details (\cref{fig:ddis_comparison}).
    
    \item \textbf{Inefficiency of Asymptotically Exact Samplers:} 
    Asymptotically exact samplers like FPS \citep{dou2024diffusion} and MCGDiff \citep{cardoso2023monte} avoid this approximation bias by using Sequential Monte Carlo (SMC) to properly weigh samples. 
    However, the heavy computational costs make them prohibitive for high-dimensional PDE problems.
\end{itemize}
DAPS effectively functions as an efficient asymptotically exact sampler. 
Unlike SMC-based methods that require maintaining and resampling a large population to represent the posterior distribution, DAPS operates on a single sample trajectory and achieves the rigorous sampling quality of exact samplers, making it the our most Pareto-optimal choice.

\subsection{Why not use joint-embedding with better posterior sampling?}
\label{sec:B.4}
A natural ablation is to integrate DAPS into joint-embedding framework like FunDPS---a baseline we term \emph{FunDAPS}. 
However, our 
analysis (\cref{sec:daps_failure}) shows that this combination fails under sparse observations. 
The failure comes from the joint-embedding design because:
\begin{itemize}
    \item \textbf{Correlation does not imply Causation:} Joint-embedding models learn $p(a, u)$ via statistical correlations rather than an explicit causal mechanism. 
    They lack a representation of physics operator to transport information between the solution space $\mathcal{U}$ and the coefficient space $\mathcal{A}$.
    \item \textbf{Sparse Guidance Failure:} 
    As shown in \cref{sec:daps_failure} and \cref{fig:fundaps_failure}, applying Langevin dynamics to a joint embedding (e.g., FunDAPS) under sparse constraints causes the stationary covariance to collapse.
    Without a differentiable operator to smooth these signals, the sampling trajectory is pushed off the data manifold.
\end{itemize}

\begin{figure}[t]
    \centering
    \includegraphics[width=1.0\linewidth]{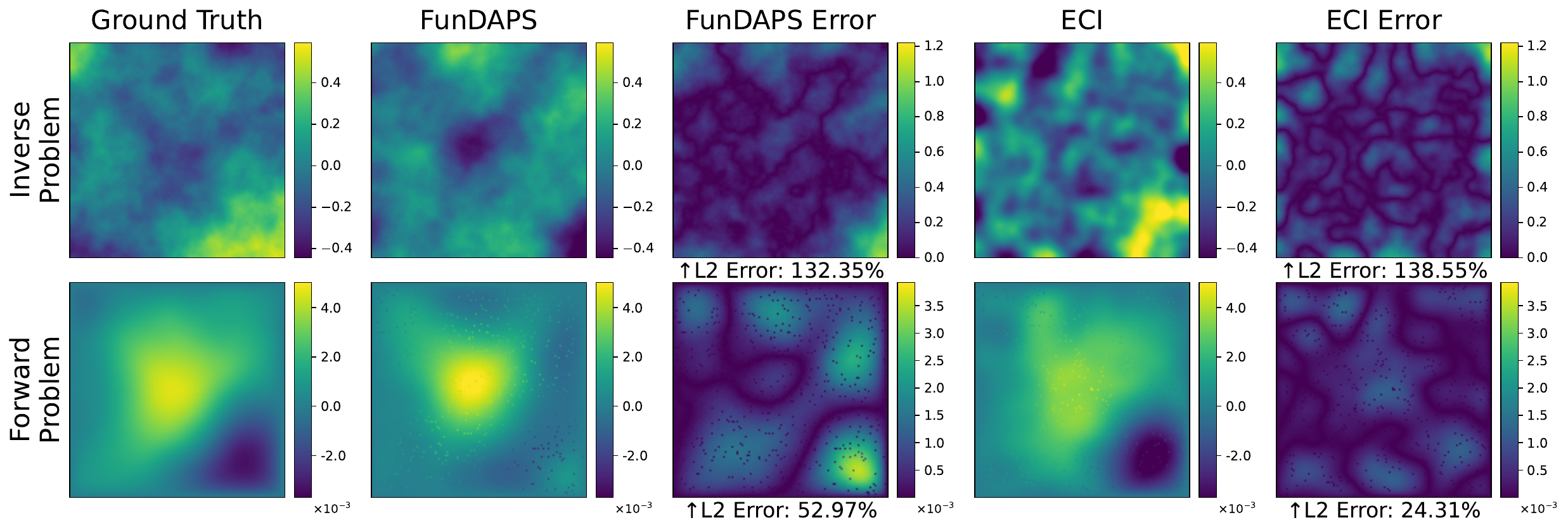} 
    \caption{\textbf{Failure of FunDAPS and ECI-sampling under sparse observations.} 
    \textbf{Top:} The joint-embedding model fails to recover the coefficient field despite using DAPS or ECI-sampling. 
    \textbf{Bottom:} Forward error reveals overfitting to specific sensor locations (visible as dots), with no generalization to unobserved regions.}
    \label{fig:fundaps_failure}
    \vspace{-2em}
\end{figure}

\subsection{Current Challenges in Flow-Based Posterior Sampling}
\label{sec: A.5}

A natural extension is to integrate flow-based inverse solvers, such as ECI-sampling \citep{cheng2025gradientfree} and OFM regression \citep{shi2025stochastic}, into a joint solution--coefficient embedding framework.
However, our empirical analysis (\cref{sec:exp}; experimental details are provided in \cref{appendix:exp-flow}) shows that these approaches exhibit clear limitations under sparse observations, either in terms of reconstruction performance or computational cost.
These failures stem from structural issues in how joint embeddings are handled during inference, for the following reasons:

\begin{itemize}
    \item \textbf{Heuristic Cross-Channel Coupling:} 
    As shown in \cref{table:performance} and \cref{fig:fundaps_failure}, ECI-sampling's heuristic replacement of $u$ with observations—lacking both $\nabla_a \| M \odot u - u_{\mathrm{obs}}\|^2$ to guide the $a$-channel and a manifold projection step—causes the inference trajectory to drift from the joint manifold, even destabilizing the $u$-channel's own manifold consistency. As also noted in PCFM \citep{utkarsh2025physicsconstrainedflowmatchingsampling}, ECI shows limited robustness in scenarios with discontinuities, making it less effective for our sparse and masked observations that require resolving high-frequency details from partial data.

    \item \textbf{High Computational Cost of Exact Langevin Inference.} 
    As shown in \cref{table:performance} and \cref{table:flow-setup}, OFM regression requires a large number of Langevin iterations to obtain satisfactory posterior samples. 
    Using accurate ODE solvers further necessitates backpropagation through the entire trajectory, resulting in substantial computational and memory overhead (\cref{exp:ofm-issue}) and making it impractical.
\end{itemize}

\begin{algorithm}[H]
\caption{DDIS-DAPS Sampler}
\label{alg:ddis_daps}
\textbf{Require:}
Sparse observation $u_{\mathrm{obs}}$, mask $M$;
diffusion model $\hat a_0(\cdot)$;
neural operator $L_\phi$;
noise schedule $\{\sigma(t_i)\}$;
prior scale $\{r_{t_i}\}$;
likelihood scale $\beta_y$;
Langevin step size $\eta$ and step count~$N_{\mathrm{L}}$. \\
\vspace{-1em}
\begin{algorithmic}[1]
\STATE $a_N \sim \mathcal{N}(0,I)$ \hfill \{\textit{Initialize from prior}\}
\FOR{$i = N \ \textbf{to}\ 1$}
    \STATE $a_0^{(0)} \leftarrow \hat a_0(a_i)$
            \hfill \{\textit{Reverse diffusion (denoising)}\}
    \FOR{$j = 0 \ \textbf{to}\ N_{\mathrm{L}}-1$}
        \STATE $\epsilon_j \sim \mathcal{N}(0,I)$
        \STATE $g_{\mathrm{prior}} \leftarrow
        - \nabla_{a_0^{(j)}} \dfrac{\|a_0^{(j)}-a_0^{(0)}\|_2^2}{r_{t_i}^2}$
        \STATE $g_{\mathrm{like}} \leftarrow
        - \nabla_{a_0^{(j)}} \dfrac{\|M \odot L_\phi(a_0^{(j)}) - u_{\mathrm{obs}}\|_2^2}{2\beta_y^2}$
        \STATE $a_0^{(j+1)} \leftarrow a_0^{(j)}
        + \eta\big(g_{\mathrm{prior}} + g_{\mathrm{like}}\big)
        + \sqrt{2\eta}\,\epsilon_j$  \hfill  \{\textit{Langevin MCMC update}\}
    \ENDFOR
    \STATE $\xi_i \sim \mathcal{N}(0,I)$
    \STATE $a_{i-1} \leftarrow a_0^{(N_{\mathrm{L}})} + \sigma(t_{i-1})\,\xi_i$ \hfill \{\textit{Re-noising with Gaussian approximation}\}
\ENDFOR
\STATE \textbf{return} $a_0$
\end{algorithmic}
\end{algorithm}

\section{Ablation Studies and Supplementary Results}
\label{sec:ablation_study}

\subsection{Pareto Fronts: Accuracy-Speed Trade-offs}

\begin{figure}[H]
    \centering
    \includegraphics[width=\textwidth]{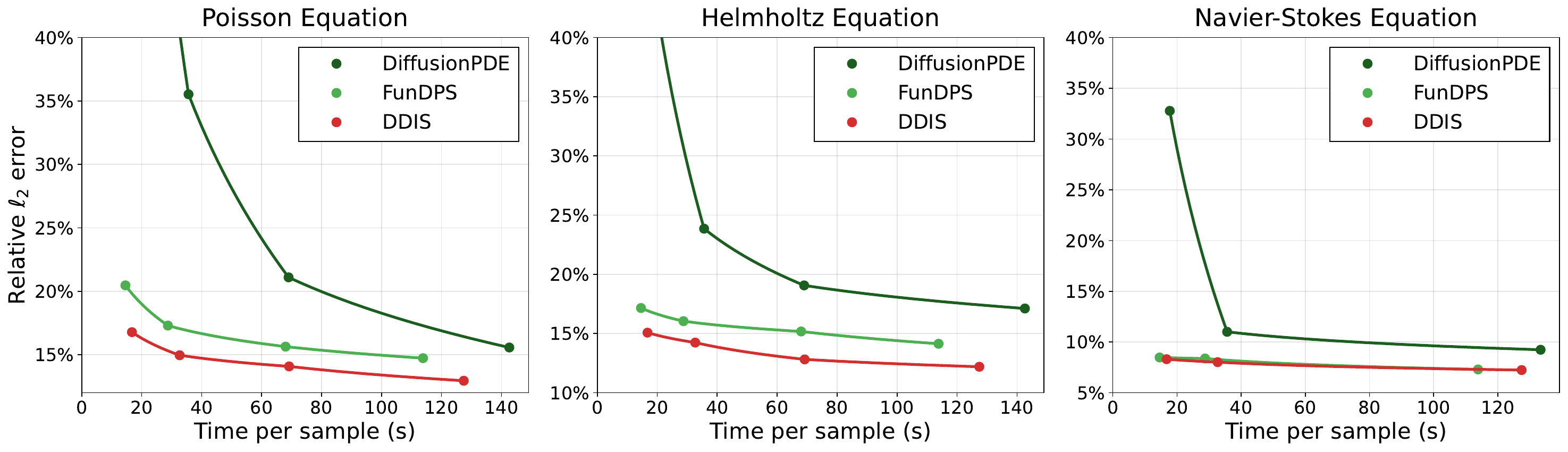}
    \caption{\textbf{DDIS consistently dominates the Pareto frontier across three PDEs.} Pareto fronts of relative $\ell_2$ error versus wall-clock time (see Table~\ref{table:performance} for settings). DDIS (red envelopes) achieves  superior accuracy-speed trade-offs, demonstrating Pareto optimality over all joint-embedding baselines (green envelopes). Even on Navier-Stokes, where the advantage on $\ell_2$ error over FunDPS is modest, DDIS achieves 60\% lower spectral error, resulting in significantly higher quality.
    }
    \label{fig:pareto_full}
\end{figure}

\clearpage
\subsection{Power Spectrum Comparison}
\label{sec:power_spec_comp}

\begin{figure}[H]
    \centering
    \includegraphics[width=\textwidth]{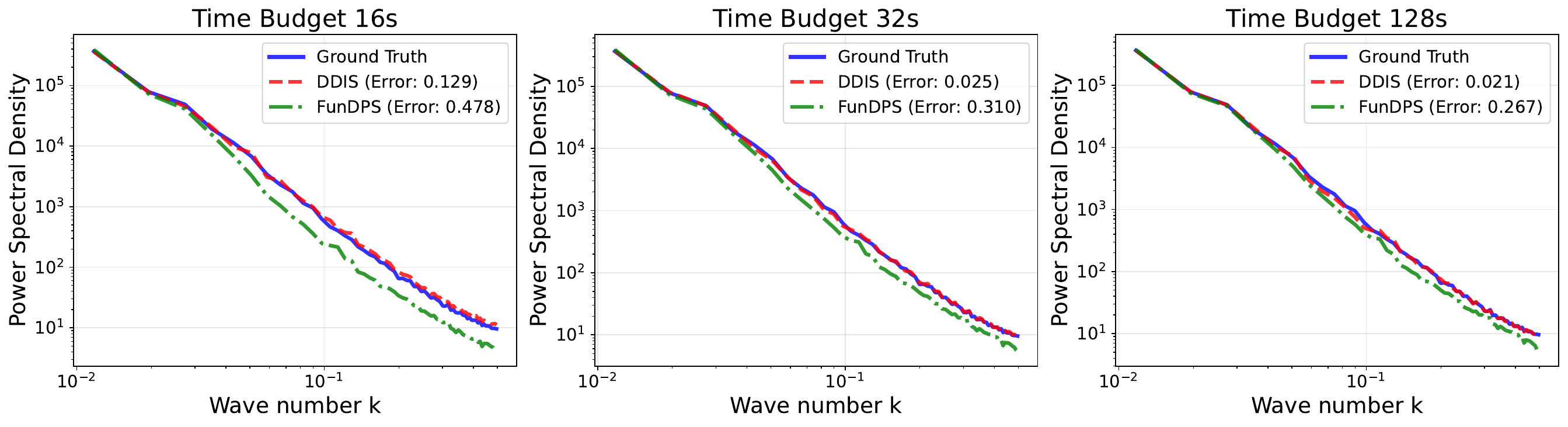}
    \caption{
    \textbf{Power spectrum comparison for the inverse Poisson problem.}
    A representative batch comparison between FunDPS and DDIS is shown.
    The predicted power spectral density is plotted against the ground truth as a function of wave number $k$.
    }
    \label{fig:poisson_spectrum}
\end{figure}

\begin{figure}[H]
    \centering
    \includegraphics[width=\textwidth]{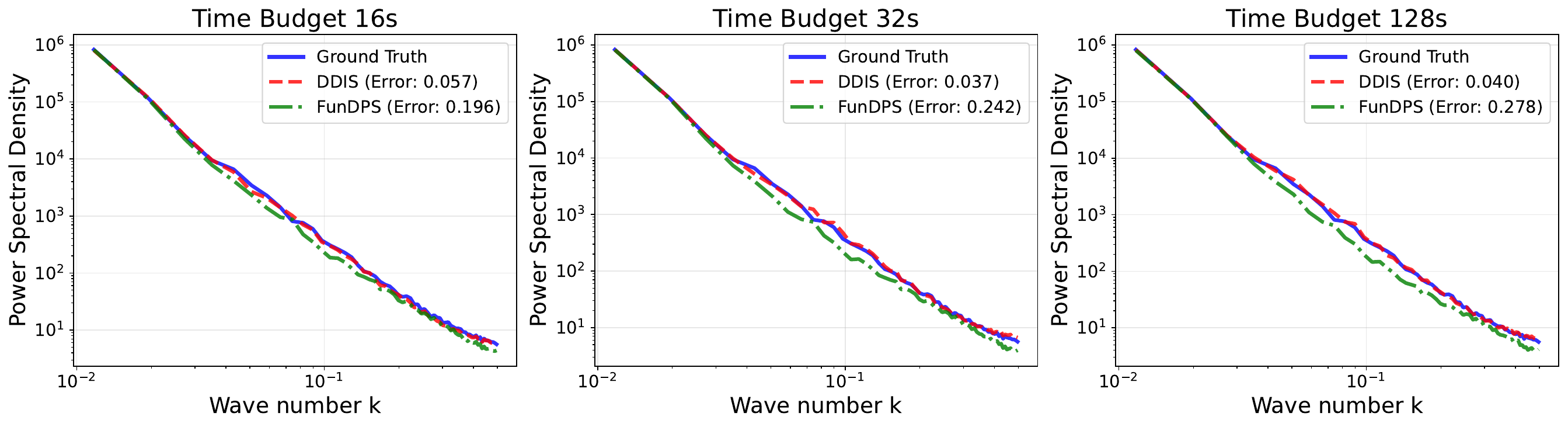}
    \caption{
    \textbf{Power spectrum comparison for the inverse Helmholtz problem.}
    DDIS accurately captures the fine-scale structures.
    }
    \label{fig:helmholtz_spectrum}
\end{figure}

\begin{figure}[H]
    \centering
    \includegraphics[width=\textwidth]{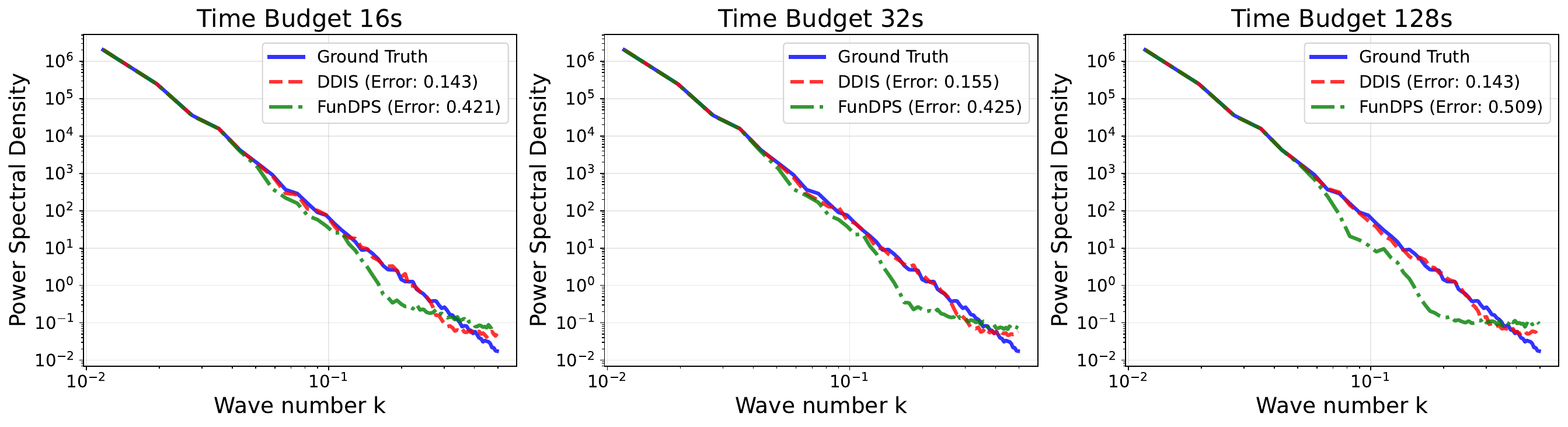}
    \caption{
    \textbf{Power spectrum comparison for the inverse Navier-Stokes problem.}
    This case is particularly challenging due to the high nonlinearity of the Navier-Stokes equation, reflected in FunDPS's gap at high wave numbers.
    Yet DDIS accurately captures them.
    }
    \label{fig:ns_spectrum}
\end{figure}

\subsection{Mixed-resolution Training}

\begin{table}[H]
\centering
\caption{
DDIS performance under low- and mixed-resolution supervision.
Relative $\ell_2$ error (\%) is reported.
Mixed-resolution training uses full $64^2$ data with an additional 10\% of $128^2$ data, maintaining accuracy while significantly reducing training costs.
}
\label{table:multi_res}
\renewcommand{\arraystretch}{0.85}
\begin{tabular}{cccc}
\toprule
Training Resolution & Poisson & Helmholtz & Steps (Diff/Lang) \\
\midrule

128   & 17.19 & 16.50 &  \\
64    & 18.28 & 16.47 & 500/20 \\
Mixed & 17.75 & 16.39 &  \\

\midrule
128   & 16.28 & 15.37 &  \\
64    & 17.27 & 15.67 & 1000/20 \\
Mixed & 16.76 & 15.58 &  \\

\midrule
128   & 15.93 & 14.94 &  \\
64    & 17.64 & 15.10 & 1000/50 \\
Mixed & 16.61 & 14.85 &  \\

\bottomrule
\end{tabular}
\end{table}

\subsection{Mixed-resolution Sampling}

\begin{table}[H]
\centering
\setlength{\tabcolsep}{3pt}
\caption{
Mixed-resolution sampling results under $3\%$ sparse observations.
\emph{DDIS (Full)} performs posterior sampling entirely at $128^2$ resolution, whereas \emph{DDIS (Mixed)} uses a coarse-to-fine pipeline (first half of the annealing at $64^2$, then $128^2$).
Relative $\ell_2$ error (\%), spectral error $E_s$, and runtime per sample are reported.
DDIS (Mixed) achieves comparable accuracy with 30\% lower runtime.
}
\label{table:mixres_inf_performance}
\begin{tabular}{lccccccc}
\toprule
 \multirow{2}{*}{Method} & \multicolumn{2}{c}{Poisson} &
 \multicolumn{2}{c}{Helmholtz} &
 \multicolumn{2}{c}{N-S} &
 \multirow{2}{*}{\makecell{Time \\ (s)}} \\
\cmidrule(lr){2-3}
\cmidrule(lr){4-5}
\cmidrule(lr){6-7}
 & $\ell_2$ & $E_s$
& $\ell_2$ & $E_s$
& $\ell_2$ & $E_s$
&  \\
\midrule

DDIS (Full)         & \textbf{16.36} & { .066} & { 15.19} & { .050} & \textbf{8.22} & \textbf{.163} & 25.92 \\
DDIS (Mixed)
             & 16.78 & \textbf{.059} & \textbf{15.08} & \textbf{.044} & 8.31 & .185 & 16.75 \\
\midrule

DDIS (Full)         & { 15.34} & \textbf{.057} & \textbf{14.03} & \textbf{.053} & \textbf{8.00} & { .174} & 51.83 \\
DDIS (Mixed) & \textbf{14.96} & .080 & 14.24 & .055 & 8.04 & \textbf{.162} & 32.67 \\
\midrule

DDIS (Full)         & 12.99 & { .098} & { 12.28} & \textbf{.079} & {7.31} & \textbf{.188} & 207.83 \\
DDIS (Mixed) & \textbf{12.95} & \textbf{.071} & \textbf{12.20} & .097 & \textbf{7.24} & .200 & 127.42 \\

\bottomrule
\end{tabular}
\end{table}

\section{Related Works}
\label{sec:related}

\input{appendix_related_work}

\section{Preliminaries}
\label{appendix:diffusion_background}

\subsection{Diffusion Model}

The minimal form of diffusion models~\citep{sohl2015deep,ho2020denoising,song2020denoising,song2020score} learns to sample from a prior distribution. 
In the inverse problems we consider, the prior is over the coefficient $p(a)$.
Formally, a diffusion model defines a forward stochastic process $\{a_t\}_{t=0}^T$ with the initial one $ a_0 \coloneqq a $ 
and is governed by:
\begin{align}
    da_t = f(a_t, t)  dt + g(t)  dw,
    \label{eqn:std_forward_process}
\end{align}
where $ dw $ is a Wiener process and $ f,g $ are drift and diffusion coefficients.
Let $ p(a_t) $ be the marginal distribution of $ a_t $.
As $ T \to \infty $, the distribution $ p(a_T) $ converges to an unstructured noise distribution determined by \eqref{eqn:std_forward_process}, irrespective of $ p(a) $. 
Then, by sampling $ a_T $ from noise distribution, e.g., Gaussian noise $\mathcal{N}(0,I)$, and reversing the process:
\begin{align}
da_t = \Big[ f(a_t, t) - g^2(t) \nabla_{a_t} \log p(a_t) \Big] dt + g(t) dw,
\label{eqn:std_reverse_process}
\end{align}
one ultimately obtains samples $ a_0 $ from $ p(a)$.
Here, the score term $ \nabla_{a_t} \log p(a_t) $ is learned by a neural network $s_{\theta}(a_t,t)$ during forward process~\eqref{eqn:std_forward_process} via score-matching \citep{vincent2011connection}.
We define
the diffusion prior as the distribution $p(a_0)$ obtained by
evolving the reverse-time dynamics~\eqref{eqn:std_reverse_process} from a noisy sample at $t=T$ down to $t=0$.

\subsection{Diffusion Posterior Sampling (DPS)}
\label{appendix:dps}

DPS modifies the reverse process \eqref{eqn:std_reverse_process} so that we obtain samples from the posterior $ p(a  \mid  u_\mathrm{obs})$:
\begin{align}
    da_t 
    &= \Big[ f(a_t, t) - g^2(t) \big( \nabla_{a_t} \log p(a_t) + \nabla_{a_t} \log p(u_\mathrm{obs}  \mid  a_t) \big) \Big] dt + g(t) dw_t.  \label{eqn:DPS}
\end{align}
Here, $ \nabla_{a_t} \log p(a_t) $ is obtained via score-matching \citep{vincent2011connection} and $ \nabla_{a_t} \log p(u_\mathrm{obs}  \mid  a_t) $ is the log-likelihood gradient that guides the reverse process toward samples consistent with $ u_\mathrm{obs} $.

However, it is impossible to compute the likelihood by $p(u_\mathrm{obs}  \mid  a_t) = \int p(u_\mathrm{obs}  \mid  a_0) p(a_0  \mid  a_t) da_0$ due to the intractability of $ p(a_0  \mid  a_t) $.
Therefore, DPS takes the approximation:
\begin{align}%
p(u_\mathrm{obs}  \mid  a_t) = \mathbb{E}_{a_0\sim p(\cdot \mid  a_t)}\left[p(u_\mathrm{obs} \mid  a_0)\right]\approx p\big( u_\mathrm{obs}  \mid  a_0 = \hat{a}_0(a_t) \big), \quad \hat{a}_0(a_t) \coloneqq \mathbb{E}[a_0  \mid  a_t].
\label{eqn:dps_approx}
\end{align}

In practice, DPS implements the process \eqref{eqn:DPS} iteratively, effectively constructing a Markov chain $p(a_t  \mid  u_\mathrm{obs}, a_{t+1})$. Each update is therefore incremental, limiting the ability to make large corrective moves toward the posterior. This motivates the Decoupled Annealing Posterior Sampling (DAPS) approach, which avoids the Markov chain formulation and enables direct posterior alignment.

\begin{remark}[Jensen-gap Approximation of DPS.]
\label{remark:DPS_jensen}
The approximation \eqref{eqn:dps_approx} incurs a \emph{Jensen's gap}:
\begin{align*}
\mathbb{E}[f(X)] \neq f(\mathbb{E}[X]),
\end{align*}
where $f(a_0)\coloneqq\log p(u_\mathrm{obs} \mid  a_0)$.
Thus, this approximation is exact only when $p(u_\mathrm{obs} \mid  a_0)$ is affine in $a_0$, and otherwise introduces bias that depends on the distribution of $a_0 \mid  a_t$.
\end{remark}

\subsubsection{Extension: DecoupledDPS}
\label{appendix:ddis_dps}

As an ablation, we consider combining the DDIS insight with the standard DPS update in \eqref{eqn:DPS}, yielding a \emph{DecoupledDPS} variant.
Unlike joint-embedding DPS~\citep{huang2024diffusionpde,yao2025guided}, which defines the likelihood directly on the noisy latent $a_t$, this variant preserves the DDIS separation between the diffusion prior and the physics-induced likelihood.

Concretely, the diffusion prior is defined over coefficients $a$ as in \cref{sec:diffusion_prior}, while the likelihood is specified on the clean variable through the neural operator surrogate:
\begin{align}
p(u_\mathrm{obs} \mid a_0)
\;\propto\;
\exp\!\left(
-\frac{\|M \odot L_\phi(a_0) - u_\mathrm{obs}\|_2^2}{2\beta_y^2}
\right).
\label{eqn:ddis_dps_likelihood}
\end{align}
Following the DPS approximation in \eqref{eqn:dps_approx}, the likelihood gradient in \eqref{eqn:DPS} is instantiated by substituting $a_0 \approx \hat a_0(a_t)$:
\begin{align}
\nabla_{a_t} \log p(u_\mathrm{obs} \mid a_t)
\;\approx\;
\nabla_{a_t}
\log p\!\left(u_\mathrm{obs} \mid a_0 = \hat a_0(a_t)\right),
\label{eqn:ddis_dps_guidance}
\end{align}
where $\hat a_0(a_t)$ is the denoised estimate induced by the diffusion prior.

Combining \cref{eqn:ddis_dps_likelihood,eqn:ddis_dps_guidance,eqn:DPS} yields a DPS-style reverse process whose guidance is mediated by the neural operator rather than a joint-embedding model.
While this decoupled DPS variant consistently improves over FunDPS (\cref{tab:decoupled_dps}), it still inherits the Jensen gap discussed in \cref{remark:DPS_jensen} and is generally worse than DDIS.
This shows that decoupling already improves over joint embeddings, while greater gains arise from DAPS naturally enabled by this decoupled design.
We therefore include decoupled DPS solely as an ablation to isolate this effect.

\begin{table}[t]
\centering
\caption{Ablation of decoupled DPS.
Relative $\ell_2$ error (\%) for FunDPS, DecoupledDPS, and DDIS under 100\% and 1\% paired data.
DecoupledDPS improves over joint embeddings but remains inferior to DDIS.
All methods use $T=16$s.}
\label{tab:decoupled_dps}
\begin{tabular}{c c c c c}
\toprule
Method & Data & Poisson & Helmholtz & Navier--Stokes \\
\midrule
FunDPS 
& 100\% & 20.47 & \underline{17.16} & 8.48 \\
& 1\%   & 35.81 & 41.69 & 13.65 \\
\midrule
DecoupledDPS
& 100\% & \underline{19.23} & 19.24 & \underline{8.37} \\
& 1\%   & \underline{24.25} & \underline{21.68} & \underline{13.10} \\
\midrule
DDIS
& 100\% & \textbf{16.36} & \textbf{15.19} & \textbf{8.22} \\
& 1\%   & \textbf{18.70} & \textbf{16.40} & \textbf{12.05} \\
\bottomrule
\end{tabular}

\end{table}

\subsection{Decoupled Annealing Posterior Sampling (DAPS)}
\label{sec:daps}

The key motivation behind DAPS is that the two components of posterior sampling---the prior update and the likelihood correction---naturally operate on different variables.
The diffusion prior defines a score on the noisy latent variable $a_t$, whereas the likelihood $p(u_\mathrm{obs} \mid a_0)$ is defined on the clean $a_0$.
DPS forces both operations to act on the same noisy $a_t$ and thus requires approximating the intractable $p(a_0 \mid a_t)$.
DAPS avoids the conflict by decoupling the updates: it performs the likelihood correction on a clean estimate $\hat a_0(a_t)$, and then re-noises it to obtain the updated $a_{t-1}$.

Let $p_t(a_t  \mid  u_\mathrm{obs})$ denote the observation-conditioned time-marginal at noise level $t$.
DAPS implements an annealing process with a noise schedule $\{\sigma_t\}$ such that, as $t$ decreases, $p_t(a_t  \mid  u_\mathrm{obs})$ anneals toward the target posterior $p(a_0  \mid  u_\mathrm{obs})$.
At each annealing step, DAPS applies a two-phase sampling based on the current sample $a_t$:
First, it performs Langevin MCMC to approximate
sampling from
\begin{align*}
p(a_0  \mid  a_t, u_\mathrm{obs}) \ \propto\ p(a_0  \mid  a_t)\, p(u_\mathrm{obs} \mid  a_0).
\end{align*}
Initialized with the clean estimate $a_0^{(0)} = \hat a_0(a_t)$ obtained
from the DDPM/DDIM denoiser~\citep{ho2020denoising,song2020denoising},
the Langevin updates take the form:
\begin{align}
\label{eqn:langevin_daps}
a_{0}^{(j+1)}
&= a_{0}^{(j)}
+ \eta\, \nabla_{a_0^{(j)}} \log p(a_0^{(j)}  \mid  a_t) \nonumber \\
&~+ \eta\, \nabla_{a_0^{(j)}} \log p(u_\mathrm{obs} \mid  a_0^{(j)})
+ \sqrt{2\eta}\,\epsilon_j,
\end{align}
where $\epsilon_j \sim \mathcal{N}(0, I)$, $\eta>0$ is step size, and
$p(a_0^{(j)}  \mid  a_t)$ is approximated by a Gaussian distribution with variance $r_t^2$:
\begin{align}
    p(a_0 \mid a_t) \approx \mathcal{N}(a_0; \hat{a}_0(a_t), r_t^2 I).
\end{align}
Second, given $a_0^{(N)}$, DAPS samples the next latent $a_{t-1}$ by re-noising $a_0^{(N)}$ under a Gaussian assumption\footnote{As \citet{zhang2025improving} noted, this Gaussian assumption is optional and adopted for practical efficiency. DAPS also allows exact posterior sampling via diffusion-score estimation at higher computational cost.}:
\begin{align}
\label{eqn:renoise_daps}
a_{t-1} \sim p(a_{t-1} \mid  a_0^{(N)}) \approx \mathcal{N}\big(a_0^{(N)}, \sigma_{t-1}^2 I\big),
\end{align}
which yields $a_{t-1}$ approximately distributed according to $p_{t-1}(a_{t-1} \mid  u_\mathrm{obs})$.

Despite its success in various inverse problems, as benchmarked in \citep{zhang2025improving}, 
DAPS degrades in sparse observation settings, a flaw highlighted by \citep{yao2025guided}:
\begin{remark}[Sparse-Guidance Failure of DAPS]
\label{remark:daps_sparse}
The degradation arises from localized gradient updates during the Langevin dynamics phase. 
When the sparse measurement $x_{\mathrm{obs}}$ is drawn from the same space as the prior variable $x$, 
i.e., $x_{\mathrm{obs}} = M \odot x + \epsilon$,
the likelihood gradient $\nabla_{x_0} \log p(x_{\mathrm{obs}}  \mid  x_0)$ is non-zero only at observed locations.
Consequently, the Langevin process creates a spatially discontinuous intermediate state that is 
out-of-distribution for the pretrained diffusion model.
\cref{sec:daps_failure} offers a quantitative analysis for this degradation.
\end{remark}

\subsubsection{Asymptotic Guarantee for DAPS}
\label{sec:theory_daps}

Our analysis starts with a statistical interpretation of the DAPS that is not explicitly formulated in the original paper~\citep{zhang2025improving}.
In particular, we show that the ideal DAPS update preserves the observation-conditioned time-marginals and converges asymptotically to the target posterior.
In DDIS, the observation model acts in solution space while sampling occurs in coefficient space. The surrogate operator $L_\phi$ is therefore introduced to transport likelihood information across these spaces. When $L_\phi \to L$, DDIS ensures not only physically meaningful but also dense guidance for the two-step update.

Let $p(a_0  \mid  u_{\mathrm{obs}})$ denote the target posterior (\cref{def:posterior}).
DAPS introduces a decreasing noise schedule $\{\sigma_t\}$,
and at noise level $t$ the latent variable $a_t$ follows an
observation-conditioned marginal distribution $p(a_t  \mid  u_{\mathrm{obs}})$.
The ideal DAPS update from $t$ to $t-1$ consists of two phases:
\begin{enumerate}[label=(\roman*)]
\item \textbf{Posterior pullback:}
sample $a_0 \sim p(a_0  \mid  a_t, u_{\mathrm{obs}})$ through Langevin dynamics \eqref{eqn:langevin_daps}.
\item \textbf{Re-noising:}
sample $a_{t-1} \sim \mathcal{N}(a_0, \sigma_{t-1}^2 I)$,
corresponding to \eqref{eqn:renoise_daps}.
\end{enumerate}

\begin{lemma}[Time-Marginal Invariance]
\label{lemma:invariance}
If $a_t \sim p(a_t  \mid  u_{\mathrm{obs}})$, the ideal DAPS
update (i)-(ii) yields
${a_{t-1} \sim p(a_{t-1}  \mid  u_{\mathrm{obs}})}$.
\end{lemma}
\begin{proof}[Proof of \cref{lemma:invariance}]
The DAPS update consists of:

\noindent
\textbf{Step 1:} Sample $a_0 \sim p(a_0  \mid  a_t, u_{\mathrm{obs}})$,  
\textbf{Step 2:} Sample $a_{t-1} \sim p(a_{t-1}  \mid  a_0)$.

\noindent
We compute the law of $a_{t-1}$ by marginalizing over all randomness:
\begin{align*}
p(a_{t-1}  \mid  u_{\mathrm{obs}})
&=
\iint 
p(a_{t-1}  \mid  a_0)\,
p(a_0  \mid  a_t, u_{\mathrm{obs}})\,
p(a_t  \mid  u_{\mathrm{obs}})\,
\mathrm da_0\,\mathrm da_t.
\end{align*}

Using Bayes' rule for the posterior correction step,
\begin{align*}
p(a_0  \mid  a_t, u_{\mathrm{obs}})
=
\frac{
p(a_t  \mid  a_0)\,p(a_0  \mid  u_{\mathrm{obs}})
}{
p(a_t  \mid  u_{\mathrm{obs}})
},
\end{align*}
we substitute and simplify:
\begin{align*}
p(a_{t-1}  \mid  u_{\mathrm{obs}})
&=
\iint 
p(a_{t-1}  \mid  a_0)\,
\frac{
p(a_t  \mid  a_0)\,p(a_0  \mid  u_{\mathrm{obs}})
}{
p(a_t  \mid  u_{\mathrm{obs}})
}
p(a_t  \mid  u_{\mathrm{obs}})\,
\mathrm da_0\,\mathrm da_t
\\
&=
\iint 
p(a_{t-1}  \mid  a_0)\,
p(a_t  \mid  a_0)\,
p(a_0  \mid  u_{\mathrm{obs}})\,
\mathrm da_0\,\mathrm da_t
\\
&=
\int 
p(a_{t-1}  \mid  a_0)\,
p(a_0  \mid  u_{\mathrm{obs}})\,
\mathrm da_0,
\end{align*}
where the final equality follows from integrating out $a_t$.
This is exactly the time-$(t-1)$ marginal of the observation-conditioned latent variable.
Thus, the two-step update preserves the smoothed posterior marginals $\small{a_{t-1} \sim p(a_{t-1}  \mid  u_{\mathrm{obs}})}$.
\end{proof}

The invariance property leads to an asymptotic guarantee:
\begin{lemma}[Asymptotic Guarantee]
\label{cor:asymptotic}
As $t \to 0$, the distribution of $a_t$ converges in distribution to
the posterior $p(a_0  \mid  u_{\mathrm{obs}})$.
\end{lemma}

The DDIS framework generalizes DAPS by backpropagating the gradient $\nabla_{a_0^{(j)}} \log p(u_\mathrm{obs}  \mid  a_0^{(j)})$ in \eqref{eqn:langevin_daps} from the solution space to the latent coefficient space through a surrogate operator $L_\phi$ (\cref{sec:ddis}).
This preserves the statistical structure of the ideal DAPS update while resolving its sparse-guidance failure (\cref{remark:daps_sparse}).
When the surrogate matches the true forward map ($L_\phi = L$), the DAPS transition in coefficient space is recovered, so the asymptotic guarantee continues to hold.

\section{Detailed Derivation of Guidance Attenuation in Joint-Embedding Models}
\label{sec:joint-attenuation}
\input{appendix_guidance_att}

\section{Sample-Complexity Analysis}
\label{sec:data-eff}

We analyze the sample complexity of DDIS compared to joint-embedding diffusion models. 
Let $\mathcal{A}$ and $\mathcal{U}$ denote the coefficient and solution function spaces, and $z \sim p_0(z)$ be base noise. 
We define three hypothesis classes. 
The diffusion prior class is the set of generators mapping noise to coefficient fields,
$\mathcal{P} \coloneqq \{ G_\theta : \mathcal{Z} \to \mathcal{A} \}_{\theta \in \Theta_P}$ with $a = G_\theta(z)$. 
The operator (likelihood) class is the set of neural operators approximating the PDE forward map,
$\mathcal{L} \coloneqq \{ L_\phi : \mathcal{A} \to \mathcal{U} \}_{\phi \in \Theta_L}$ with $u = L_\phi(a)$. 
The joint-embedding class is the set of diffusion generators mapping noise directly to coefficient-solution pairs,
$\mathcal{J} \coloneqq \{ H_\psi : \mathcal{Z} \to \mathcal{A} \times \mathcal{U} \}_{\psi \in \Theta_J}$ with $(a,u) = H_\psi(z)$.
Below, we derive the hypothesis-class inclusion relationships between $\mathcal{P}, \mathcal{L}$, and $\mathcal{J}$.
\subsection{Hypothesis Class of Joint and Decoupled Models}
\label{sec:class-inclusion-relation}

A joint-embedding generative model aims to represent the joint distribution $p(a,u)$. 
Since $p(a,u)=p(a)\,p(u\mid a)$, modeling the joint is sufficient to model both the prior $p(a)$ and the likelihood $p(u\mid a)$. 
Thus, for every prior generator in $\mathcal{P}$ and every operator surrogate in $\mathcal{L}$, there must exist a corresponding joint model in $\mathcal{J}$ whose marginals and conditionals match them. 
This induces the hypothesis-class inclusion relationships between $\mathcal{P}$, $\mathcal{L}$, and $\mathcal{J}$, elaborated in the following lemmas:

\begin{lemma}[Joint models subsume prior generators]
\label{lemma:prior_inclusion}
$$\mathcal{P} \subseteq \mathcal{J}.$$
\end{lemma}
\begin{proof}
Any prior generator $G_\theta \in \mathcal{P}$ maps noise $z$ to coefficients $a \in \mathcal{A}$.
A joint model can always reproduce this behavior by ignoring the solution component.
For any fixed $u_0 \in \mathcal{U}$, define
\begin{align*}
H_\psi(z) = (G_\theta(z),\, u_0).
\end{align*}
Then $H_\psi \in \mathcal{J}$, and its marginal matches $G_\theta$, i.e., every prior generator is realizable within $\mathcal{P}$.
\end{proof}

\begin{lemma}[Joint models subsume operator surrogates]
\label{lemma:likelihood_inclusion}
$$\mathcal{L} \subseteq \mathcal{J}.$$
\end{lemma}
\begin{proof}
Any operator surrogate $L_\phi \in \mathcal{L}$ maps coefficients $a$ to solutions $u$.
Since every joint model $H_\psi \in \mathcal{J}$ generates pairs $(a,u)$ together,
it must be expressive enough to reproduce the mappings.
Choose any base coefficient generator $G_{\theta_0}$ and define
\begin{align*}
H_\psi(z) = \bigl(G_{\theta_0}(z),\, L_\phi(G_{\theta_0}(z))\bigr).
\end{align*}
Then $H_\psi \in \mathcal{J}$ with output
$u = L_\phi(a)$, i.e., every operator surrogate is realizable within $\mathcal{J}$.
\end{proof}

\begin{lemma}[Joint class inclusion and complexity]
\label{lemma:inclusion_complexity}
Let $d_P$, $d_L$, and $d_J$ denote the Rademacher complexities of the prior, operator, and joint hypothesis classes, respectively.
Then
\begin{align}
d_J \ge \max(d_P,d_L).
\end{align}
\end{lemma}

\begin{proof}
By Lemma~\ref{lemma:prior_inclusion} we have $\mathcal{P} \subseteq \mathcal{J}$, and by Lemma~\ref{lemma:likelihood_inclusion} we have $\mathcal{L} \subseteq \mathcal{J}$.
Hence $\mathcal{P} \cup \mathcal{L} \subseteq \mathcal{J}$.
Rademacher complexity is monotone under inclusion, so
\begin{align*}
d_J = \mathfrak{R}(\mathcal{J})
\;\ge\;
\mathfrak{R}(\mathcal{P} \cup \mathcal{L})
\;\ge\;
\max\bigl(\mathfrak{R}(\mathcal{P}),\,\mathfrak{R}(\mathcal{L})\bigr)
=
\max(d_P,d_L).
\end{align*}
\vspace{-2em}
\end{proof}

\subsection{Generalization Bound}
We adopt Rademacher complexities to derive generalization bounds of the hypothesis class:
\begin{lemma}[Rademacher generalization bound, Ch.~26 of \cite{shalev2014understanding}]
\label{lemma:rad_bound}
Let $\mathcal{H}$ be a hypothesis class with Rademacher complexity $d_H$, and
let $\hat{h}$ denote an empirical risk minimizer trained on $n$ i.i.d.\ samples.
Then its expected excess risk satisfies
\begin{align}
\mathbb{E}[\ell(\hat{h})]
-
\min_{h \in \mathcal{H}} \mathbb{E}[\ell(h)]
=
\tilde{\mathcal{O}}
\!\left(
\sqrt{\frac{d_H}{n}}
\right).
\end{align}
\end{lemma}

Then, based on \cref{lemma:rad_bound}, we evaluate the excess risk of DDIS and joint-embedding methods.

\begin{proposition}[Sample complexity of DDIS]
\label{prop:ddis_bound}
Given $n_u$ unpaired samples and $n_p$ paired samples, the DDIS estimation error satisfies
\begin{align}
\mathrm{err}_{\mathrm{DDIS}}
=
\tilde{\mathcal{O}}
\left(
\sqrt{\frac{d_L}{n_p}}
+
\sqrt{\frac{d_P}{n_p+n_u}}
\right)
&\xrightarrow[{n_u \to \infty}]{}
\tilde{\mathcal{O}}
\left(
\sqrt{\frac{d_L}{n_p}}
\right).
\end{align}
\end{proposition}
\begin{proof}
Apply \cref{lemma:rad_bound} to $\mathcal{P}$ using $n_p+n_u$ samples, 
and to $\mathcal{L}$ using $n_p$ paired samples, then combine both.
\end{proof}

\begin{proposition}[Sample complexity of joint-embedding methods]
\label{prop:joint_bound}
Joint-embedding models trained on $n_p$ paired samples satisfy
\begin{align}
\mathrm{err}_{\mathrm{joint}}
=
\tilde{\mathcal{O}}
\left(
\sqrt{\frac{d_J}{n_p}}
\right).
\end{align}
\end{proposition}

\begin{remark}[Comparison of sample-complexity guarantees]
Combining \cref{prop:ddis_bound} and \cref{prop:joint_bound} with
\cref{lemma:inclusion_complexity}, in the imbalanced regime $n_u \gg n_p$
the DDIS bound scales as $\tilde{\mathcal{O}}(\sqrt{d_L/n_p})$, whereas the
joint-embedding bound scales as $\tilde{\mathcal{O}}(\sqrt{d_J/n_p})$ with
$d_J \ge \max(d_P,d_L)$.
Thus DDIS yields a strictly more favorable upper bound
than joint-embedding models with scarce paired data.
\end{remark}

\section{Proofs of the Failure Modes of Joint Embeddings + DAPS}
\input{appendix_daps_failure.tex}

\clearpage
\section{Experimental Setup}
\label{sec:exp-setup}
\input{appendix_exp_setup}

\section{Qualitative Results}

In this section, we present qualitative results for forward and inverse problems, providing a visual comparison of representative methods to supplement the quantitative analysis in \cref{table:performance}.

\subsection{Inverse Poisson problem}
\begin{figure}[h]
    \centering
    \includegraphics[width=1.0\linewidth]{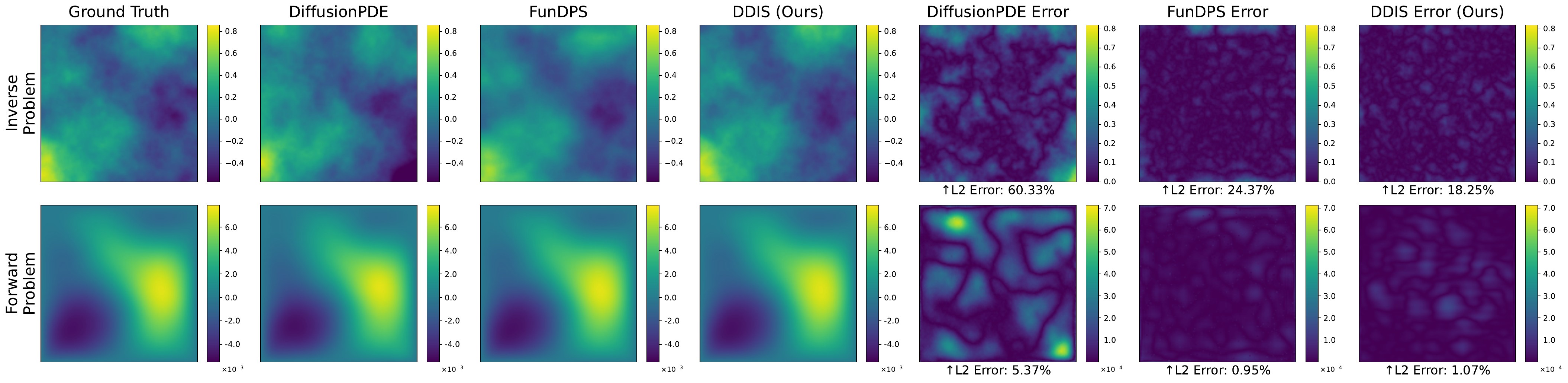} 
    \vspace{-1.5em}
    \caption{Qualitative comparison of results for the inverse Poisson problem (Time Budget: 16 s).}
    \vspace{-0.5em}
\end{figure}

\begin{figure}[h]
    \centering
    \includegraphics[width=1.0\linewidth]{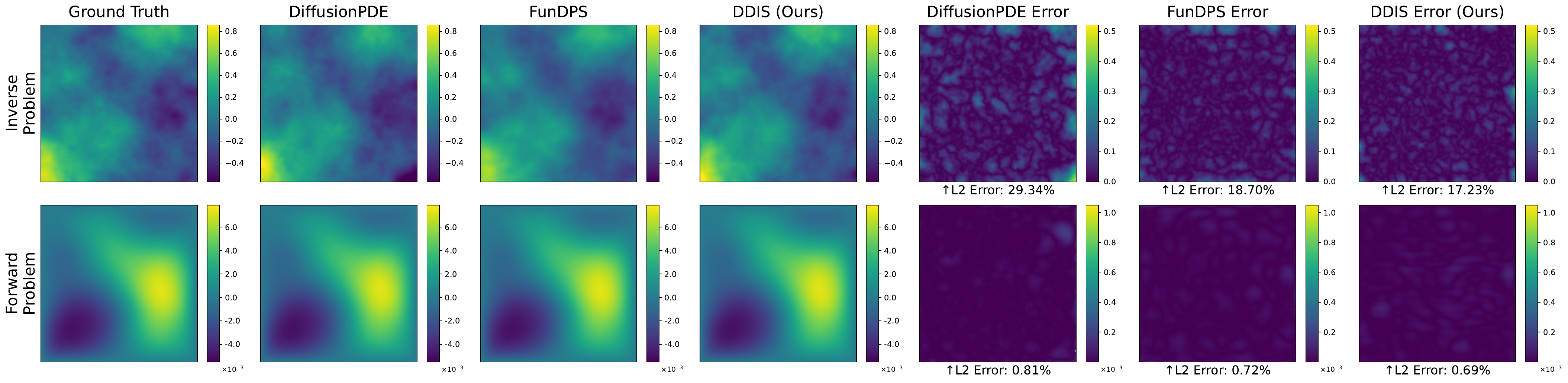} 
    \vspace{-1.5em}
    \caption{Qualitative comparison of results for the inverse Poisson problem (Time Budget: 32 s).}
    \vspace{-0.5em}
\end{figure}

\begin{figure}[!h]
    \centering
    \includegraphics[width=1.0\linewidth]{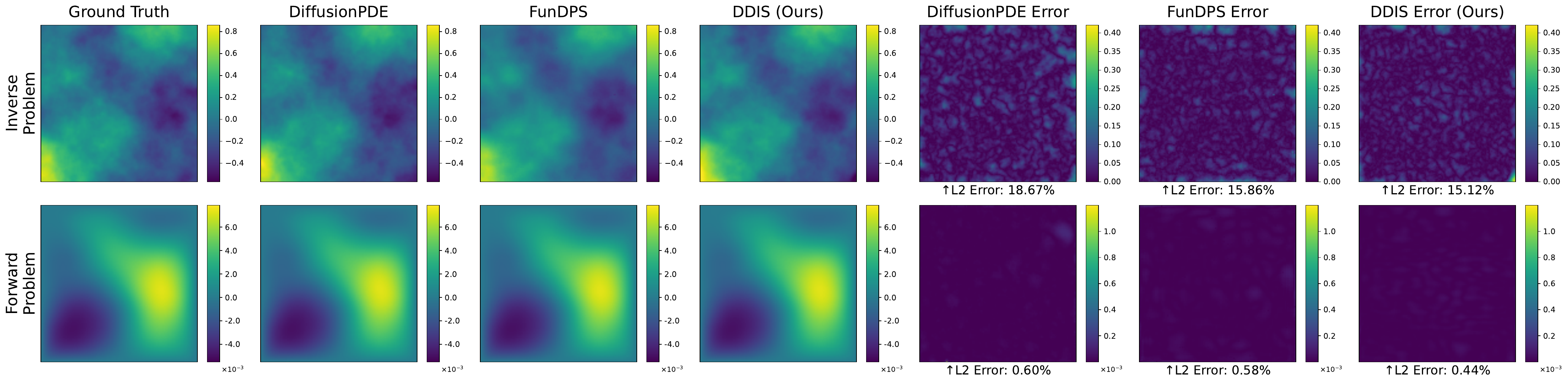} 
    \vspace{-1.5em}
    \caption{Qualitative comparison of results for the inverse Poisson problem (Time Budget: 128 s).}
    \vspace{-0.5em}
\end{figure}

\clearpage
\subsection{Inverse Helmholtz problem}
\begin{figure}[h]
    \centering
    \includegraphics[width=1.0\linewidth]{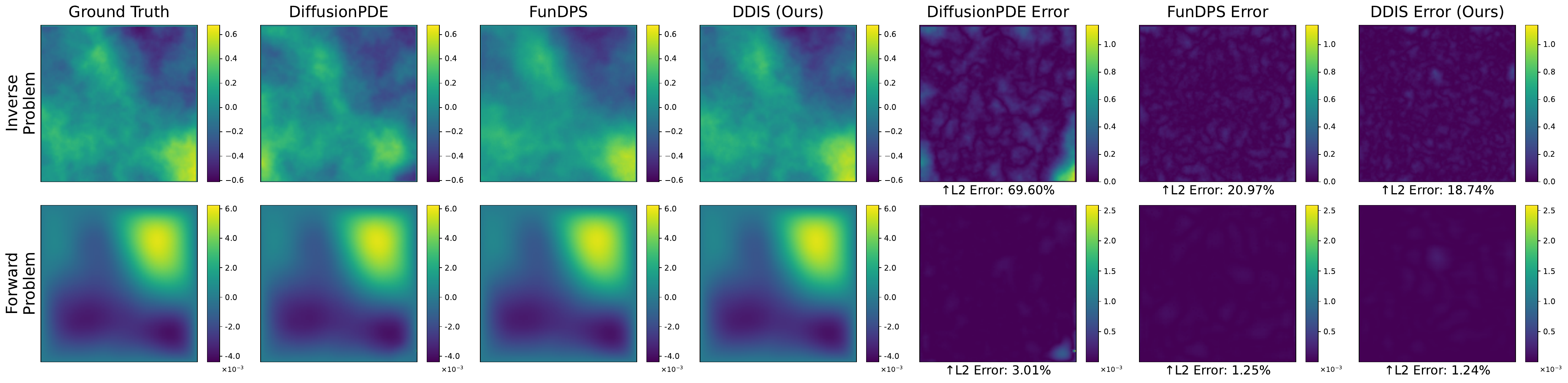} 
    \caption{Qualitative comparison of results for the inverse Helmholtz problem (Time Budget: 16 s).}
\end{figure}

\begin{figure}[h]
    \centering
    \includegraphics[width=1.0\linewidth]{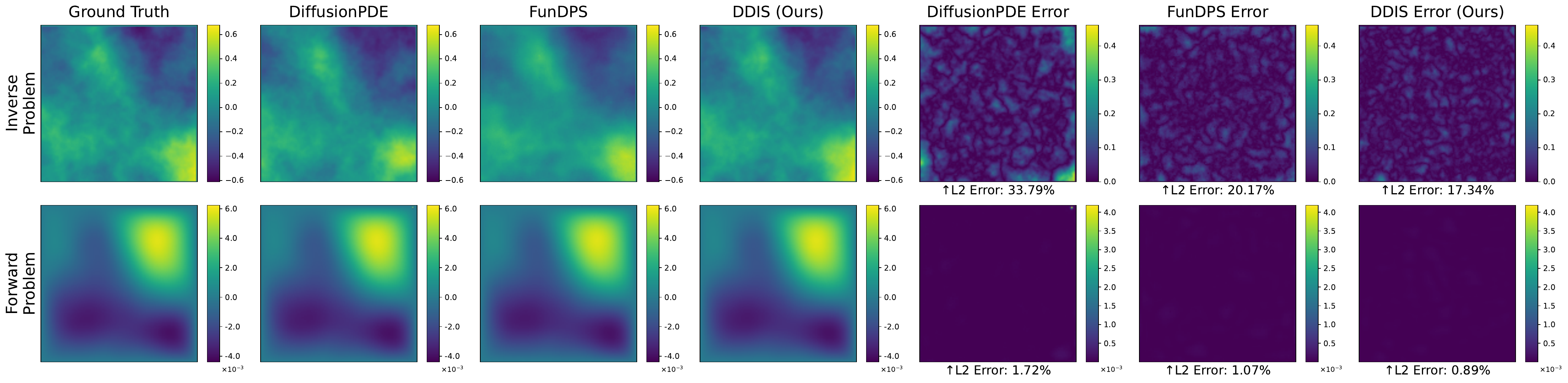} 
    \caption{Qualitative comparison of results for the inverse Helmholtz problem (Time Budget: 32 s).}
\end{figure}

\begin{figure}[h]
    \centering
    \includegraphics[width=1.0\linewidth]{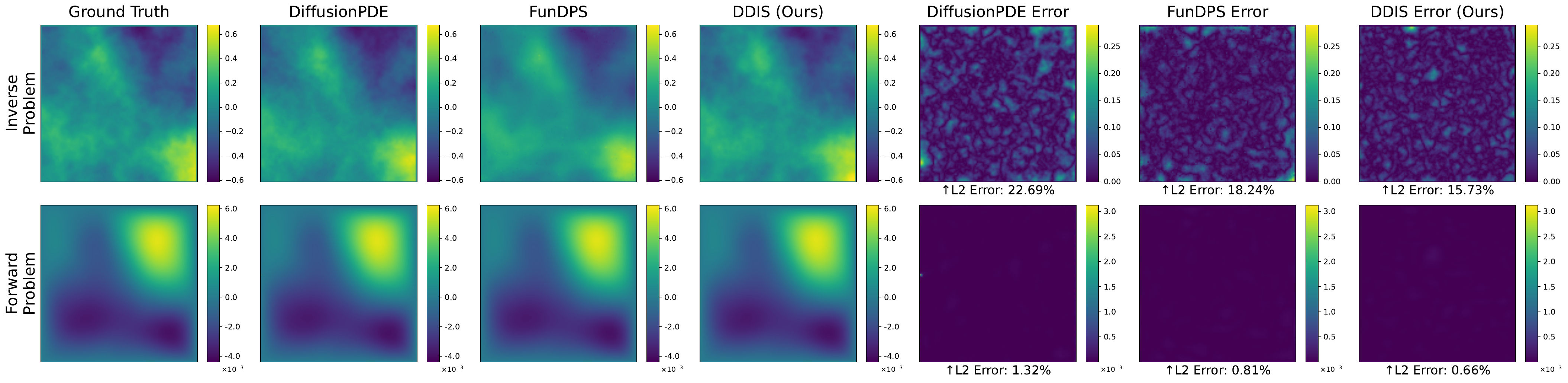} 
    \caption{Qualitative comparison of results for the inverse Helmholtz problem (Time Budget: 128 s).}
\end{figure}

\newpage
\subsection{Inverse Navier-Stokes problem}
\begin{figure}[h]
    \centering
    \includegraphics[width=1.0\linewidth]{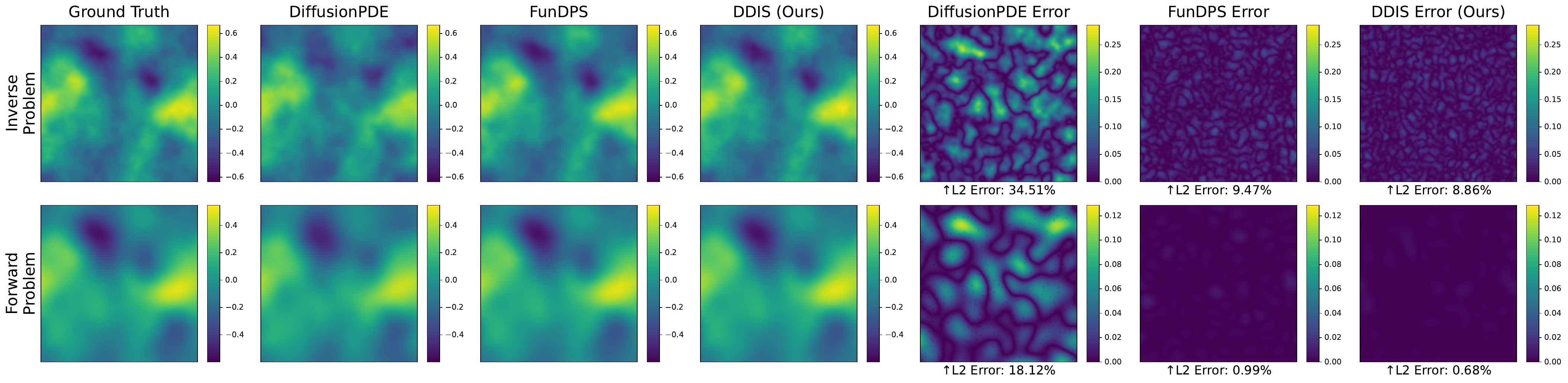} 
    \caption{Qualitative comparison of results for the inverse Navier-Stokes problem (Time Budget: 16 s).}
\end{figure}

\begin{figure}[h]
    \centering
    \includegraphics[width=1.0\linewidth]{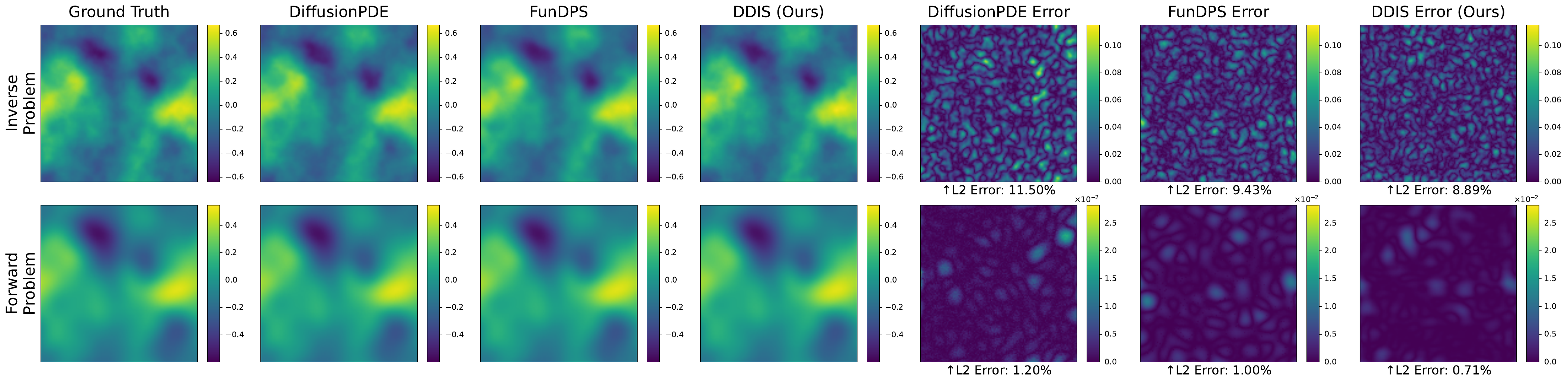} 
    \caption{Qualitative comparison of results for the inverse Navier-Stokes problem (Time Budget: 32 s).}
\end{figure}

\begin{figure}[h]
    \centering
    \includegraphics[width=1.0\linewidth]{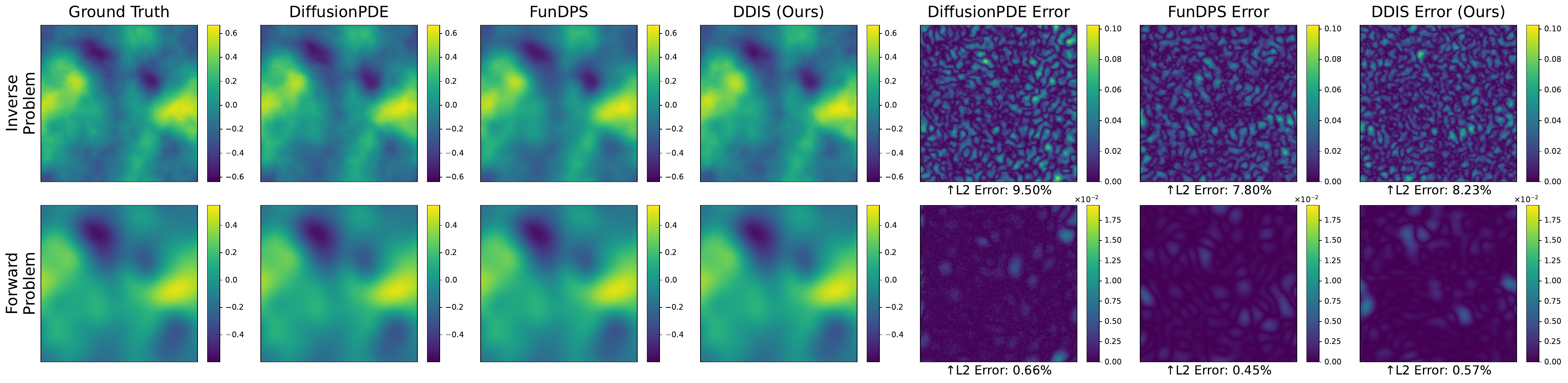} 
    \caption{Qualitative comparison of results for the inverse Navier-Stokes problem (Time Budget: 128 s).}
\end{figure}

\newpage
\subsection{Operator Flow Matching}
\begin{figure}[h]
    \centering
    \includegraphics[width=1.0\linewidth]{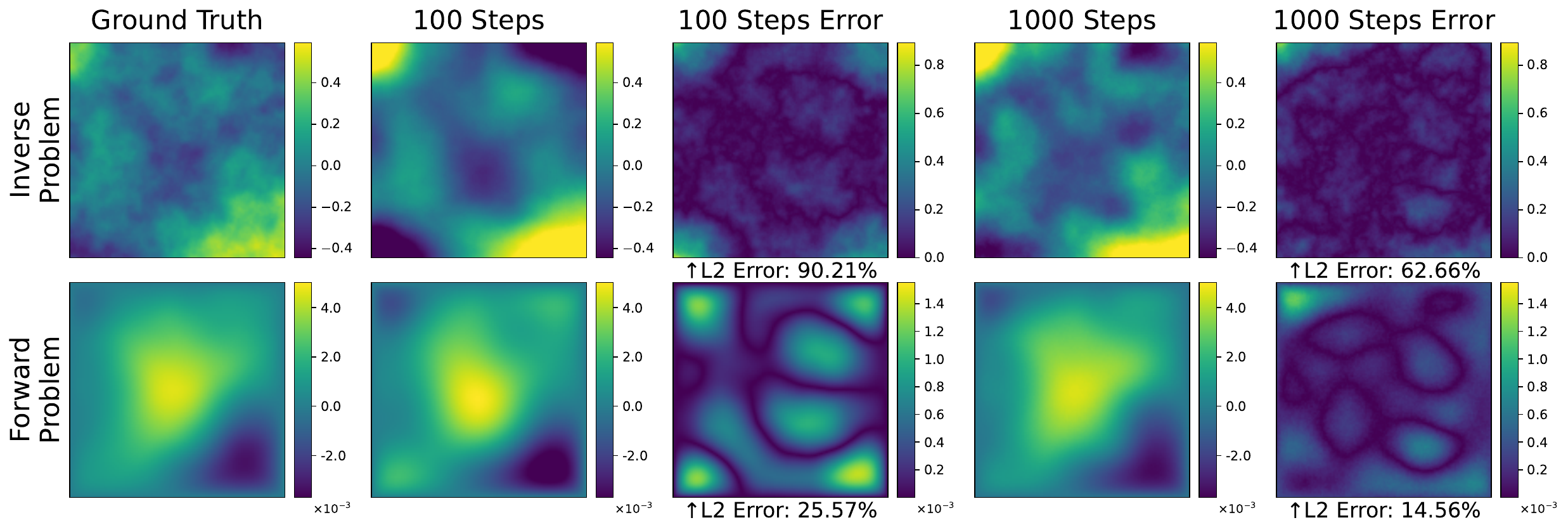} 
    \caption{Generation mean with OFM regression across different Langevin steps on the Poisson inverse problem.}
    \label{fig:ofm}
\end{figure}

%% file: appendix_related_work.tex
\paragraph{Diffusion Models for Inverse Problems.}
Diffusion models have become a dominant paradigm for solving inverse problems, broadly categorized into two strategies.
The first and more trivial strategy, \textit{conditional diffusion}, trains models directly on the conditional distribution $p(x|y)$ or modifies the reverse process with task-specific architectures \citep{saharia2022palette, whang2022deblurring}.
While effective for fixed observation patterns, this supervised approach lacks flexibility, requiring retraining for every new forward operator or measurement configuration.
The alternative strategy, \textit{unsupervised posterior sampling}, leverages a pre-trained unconditional prior $p(x)$ and enforces measurement consistency at inference time.
For linear inverse problems, decomposition-based methods exploit the linear structure of the operator: DDRM \citep{kawar2022denoising} utilizes Singular Value Decomposition (SVD) to efficiently perform diffusion in the spectral space, whereas DDNM \citep{wang2022zero} employs Range-Null Space Decomposition to avoid heavy computation in high-dimensional tasks \citep{daras2024survey}.
For general non-linear problems, guidance-based methods approximate the intractable likelihood score. 
Notably, DPS \citep{chung2022diffusion} employs Tweedie's formula to estimate the clean data mean, while LGD \citep{song2023loss} incorporates loss-based gradient guidance.
Beyond simple guidance, optimization-centric approaches like DiffPIR \citep{zhu2023denoising}, PnP-DM \citep{wu2024principled}, and DPnP \citep{xu2024provably} utilize variable splitting techniques to alternate between proximal data updates and denoising \citep{zheng2025inversebench}.
Recently, researchers have proposed asymptotically exact samplers using Sequential Monte Carlo, such as FPS \citep{dou2024diffusion} and MCGDiff \citep{cardoso2023monte}, as well as variational frameworks like RED-diff \citep{mardani2023variational}.
Within this landscape, DAPS \citep{zhang2025improving} distinguishes itself by decoupling the denoising and likelihood updates via Langevin dynamics to mitigate approximation bias.
However, as highlighted in \citep{zheng2025inversebench}, even these advanced posterior sampling methods remain fragile: they violate the stability conditions of PDE solvers and yield spatially discontinuous updates when observations are sparse.

\paragraph{Inverse Problems with Arbitrary, Sparse Observations.}
Practical scientific inverse problems often rely on sparse, irregular measurements, emerging to be a new structural challenge. 
Inverse solvers in the previous paragraph actually treat inverse problems as natural image problems, and thus mostly fail under sparse observation.
Conditional generative approaches, such as PIDM \citep{shu2023physics}, CoCoGen \citep{jacobsen2025cocogen}, PalSB \citep{li2025physicsaligned}, and PRISMA \citep{sawhney2025beyond}, fail to handle this variability due to a combinatorial explosion of possible sensor layouts.
As noted in \citep{daras2024survey}, they lack the flexibility to generalize to varying masks without exhaustive retraining.
Alternatively, direct posterior sampling strategies could theoretically handle arbitrary masks by masking the numerical PDE solver during likelihood evaluation.  
However, backpropagating gradients through iterative solvers is already known to suffer from computational instability and vanishing gradients even in dense settings \citep{zheng2025inversebench}. 
This instability is amplified under sparse observations, where the gradient calculation becomes ill-conditioned \citep{plessix2006review,leung2021level}. While certain gradient-free replacement techniques \citep{Amor_s_Trepat_2026} can satisfy sparse observations without backpropagation, they are difficult to generalize to complex physics-constrained settings. Furthermore, these methods focus on solution field reconstruction and are not designed to resolve the underlying coefficient space.

To avoid these solver-based issues, recent works like DiffusionPDE \citep{huang2024diffusionpde} and FunDPS \citep{yao2025guided} adopt joint-embedding methods.
By training a diffusion model on the concatenated distribution of coefficients and solutions $p(a, u)$, these methods effectively recast the complex inverse problem as a data-driven inpainting task.
While this bypasses unstable solvers, it relies on a questionable premise: modeling physics as a statistical correlation rather than a causal mechanism. 
Intuitively, we argue that enforcing strict physical laws into a joint distribution places a heavy representational burden on the model, forcing it to ``rediscover" the forward map purely from data.
This suggests joint-embedding methods may be inherently data-inefficient compared to approaches that explicitly distinguish the prior and the forward operator.

\paragraph{Flow Matching for PDE Inverse Problems.}
Beyond diffusion models, flow-based models have emerged as a promising paradigm for inverse problems due to the computational efficiency demonstrated in other domains like computer vision. 
Several recent works have explored flow-based models for image inverse problems.
For example, D-Flow \citep{benhamu2024dflowdifferentiatingflowscontrolled} differentiates the cost function through continuous normalizing flows, while FlowChef \citep{Patel_2025_ICCV} derives Jacobian-based guidance by treating the clean sample as an intermediate variable.
FlowDPS \citep{kim2025flowdpsflowdrivenposteriorsampling} further leverages Tweedie’s formula to decompose denoised and noisy estimates, applying gradients to appropriate components for posterior sampling. 
For scientific inverse problems, function-space modeling plays a central role.
Functional Flow Matching (FFM) \citep{kerrigan2023functionalflowmatching} extends flow matching to infinite-dimensional settings, providing a principled formulation grounded in measure theory.
More recently, ECI-sampling \citep{cheng2025gradientfree} and PCFM \citep{utkarsh2025physicsconstrainedflowmatchingsampling} propose training-free algorithms for enforcing physical constraints.
From a more general perspective, Operator Flow Matching (OFM) \citep{shi2025stochastic} introduces a stochastic process learning framework for universal functional regression, enabling posterior sampling via SGLD-based updates. 
However, existing flow-based methods have largely focused on inverse problems defined on a single state variable, and have not yet been developed for joint coefficient-solution inverse problems.
As a result, the structured dependencies between coefficients and solutions remain largely unexploited in current flow-based formulations.

\paragraph{Neural Operators as Physics Representation.}
Neural operators are architectures designed to learn mappings between infinite-dimensional function spaces rather than fixed vector spaces \citep{kovachki2023neural,berner2025principled,duruisseaux2025fourierneuraloperatorsexplained}.
Unlike standard CNNs or MLPs constrained to specific discretizations, neural operators are resolution-invariant, allowing models trained on coarse discretizations to work on finer grids \citep{li2020fourier}. 
Although architectures like DeepONet \citep{lu2019deeponet} offer flexibility for complex geometries,
the Fourier Neural Operator (FNO) \citep{li2020fourier} achieves state-of-the-art efficiency on uniform grids, particularly for massive-scale climate modeling \citep{kurth2023fourcastnet} and chaotic fluid dynamics \citep{li2020fourier}. 
To address specific physical tasks, recent works have extended these backbones: PINO \citep{li2024physics} incorporates PDE constraints into the loss function to enable unsupervised or semi-supervised learning; Geo-FNO \citep{li2023fourier} adapts spectral methods to arbitrary meshes via geometric deformations; and OFormer \citep{li2022transformer} integrates Transformer attention to capture long-range dependencies.
These architectures have become robust surrogates for traditional solvers, especially when repeated evaluations of complex physics are needed.

\nocite{das2012simulated,gregson2014capture,fan2020solving}
\nocite{mardani2017deep,sidky2020cnns,song2021solving}
\nocite{davis1995solving,wang2012regularization,kostsov2015general}
\nocite{tarantola2005inverse,groetsch1993inverse,beck1985inverse,mueller2012linear}
\nocite{morss2001idealized,andrychowicz2023deep,manshausen2024generative,allen2025end}
\nocite{gholami2010regularization,herrmann2012efficient,parsekian2015multiscale}

%% file: appendix_guidance_att.tex
Let the joint coefficient-solution variable be $x \coloneqq (a,u)\in \mathcal A\times \mathcal U$.
Joint-embedding models learn a probabilistic prior $p(x)$.
Unlike DDIS, which specifies a conditional likelihood $p(u_{\mathrm{obs}}\mid a)$ through a forward operator,
joint-embedding methods assume a Gaussian observation model defined directly on the clean joint variable $x_0$:
\begin{align*}
u_{\mathrm{obs}} = M x_0 + \varepsilon,
\qquad
\varepsilon \sim \mathcal N(0,\tilde{\sigma}_{\mathrm{obs}}^2 I),
\end{align*}
where $M$ is a masking operator selecting the observed components of $u$ and
$\tilde{\sigma}_{\mathrm{obs}}^2$ denotes the prescribed observation noise level.
This induces the likelihood
\begin{align*}
p_{\mathrm{joint}}(u_{\mathrm{obs}} \mid x_0)
=
\mathcal N\big(M x_0,\tilde{\sigma}_{\mathrm{obs}}^2 I\big).
\end{align*}
For posterior sampling, we consider DPS as adopted by our baselines
FunDPS~\cite{yao2025guided} and DiffusionPDE~\cite{huang2024diffusionpde},
where the likelihood is instantiated within the diffusion process
$\{x_t\}_{t=0}^{T}$ over the joint variable $x_t=(a_t,u_t)$.
Let $\hat{x}_0(x_t,t)$ denote an estimate of the denoised sample $\mathbb{E}[x_0\mid x_t]$.
The guidance is given by the log-likelihood gradient with respect to $x_t$,
approximating the likelihood via $\hat{x}_0$ as in~\eqref{eqn:dps_approx}:
\begin{align}
\nabla_{x_t} \log p_{\mathrm{joint}}(u_{\mathrm{obs}}\mid x_t)
&\approx
\nabla_{x_t} \log p_{\mathrm{joint}}(u_{\mathrm{obs}}\mid \hat{x}_0(x_t,t)) \nonumber\\
&=
\frac{1}{\tilde{\sigma}_{\mathrm{obs}}^2}
J_{\hat x_0}(x_t,t)^\top
M^\top r(x_t,t)
\label{eqn:joint-like-grad-dps}
\end{align}
where the Jacobian $J_{\hat x_0}(x_t,t) \coloneqq \nabla_{x_t} \hat x_0(x_t,t)$ and the residual $r(x_t,t) \coloneqq u_\mathrm{obs}-M\hat x_0(x_t,t)$.
\begin{definition}[Scale-free guidance]
\label{def:joint-guidance}
We define the scale-free guidance as
\begin{align}
g(x_t,t)
\coloneqq
J_{\hat x_0}(x_t,t)^\top
M^\top r(x_t,t).
\end{align}
\end{definition}

We consider score-based diffusion models with the following lemma:
\begin{lemma}[Tweedie-type estimator of denoised sample]
\label{lemma:tweedie-x0hat}
Let $\{x_t\}_{t=0}^{T}$ be a diffusion process with clean variable $x_0$.
Given the intermediate state $x_t$ and a score model
$s_\theta(x_t,t)\approx\nabla_{x_t}\log p_t(x_t)$,
the estimator of the denoised sample $\mathbb{E}[x_0\mid x_t]$ is
\begin{align}
\hat x_0(x_t,t)
&=
\frac{1}{\alpha_t}\Big(x_t+\sigma_t^2\,s_\theta(x_t,t)\Big),
\end{align}
where $\alpha_t$ and $\sigma_t$ are time-dependent coefficients.
\end{lemma}

We decompose the denoised estimate and the learned score field into their $a$- and $u$-components:
\begin{align*}
\hat x_0(x_t,t)
=
\begin{pmatrix}
\hat a_0(x_t,t)\\
\hat u_0(x_t,t)
\end{pmatrix},
\qquad
s_\theta(x_t,t)
=
\begin{pmatrix}
s_{\theta,a}(x_t,t)\\
s_{\theta,u}(x_t,t)
\end{pmatrix},
\end{align*}
then we can decompose the guidance into block-form:
\begin{lemma}[Block-form guidance]
\label{lemma:block-guidance}
Let $\hat x_0(x_t,t)$ be defined by the Tweedie estimator (\cref{lemma:tweedie-x0hat}).
Then the likelihood guidance (\cref{def:joint-guidance}) under DPS admits the block decomposition
\begin{align}
\label{eq:block-guidance-final}
g(x_t,t)
\propto
\begin{pmatrix}
\sigma_t^2\,\partial_{a_t}s_{\theta,u}(x_t,t)^\top\\
I+\sigma_t^2\,\partial_{u_t}s_{\theta,u}(x_t,t)^\top
\end{pmatrix}
\, r(x_t,t),
\end{align}
where $r(x_t,t)=u_{\mathrm{obs}}-M\hat x_0(x_t,t)$.
\end{lemma}
\begin{proof}
The Jacobian of the denoised estimate admits the block decomposition
\begin{align}
J_{\hat x_0}(x_t,t)
=
\nabla_{x_t}\hat x_0(x_t,t)
=
\begin{pmatrix}
\partial_{a_t}\hat a_0 & \partial_{u_t}\hat a_0\\
\partial_{a_t}\hat u_0 & \partial_{u_t}\hat u_0
\end{pmatrix}.
\label{eq:block-Jx0hat}
\end{align}
Since $M$ selects the $u$-component, we have
\begin{align}
M^\top r(x_t,t)
=
\begin{pmatrix}
0\\
r(x_t,t)
\end{pmatrix}.
\label{eq:Mt-r-block}
\end{align}
Substituting \eqref{eq:block-Jx0hat} and \eqref{eq:Mt-r-block} into the scale-free guidance
(\cref{def:joint-guidance}) yields
\begin{align}
g(x_t,t)
&\propto
J_{\hat x_0}(x_t,t)^\top
M^\top r(x_t,t)
=
J_{\hat x_0}(x_t,t)^\top
\begin{pmatrix}
0\\
r(x_t,t)
\end{pmatrix}
\nonumber\\
&=
\begin{pmatrix}
(\partial_{a_t}\hat u_0(x_t,t))^\top\, r(x_t,t)\\
(\partial_{u_t}\hat u_0(x_t,t))^\top\, r(x_t,t)
\end{pmatrix}.
\label{eq:g-block-multiply}
\end{align}

From \cref{lemma:tweedie-x0hat}, the $u$-component satisfies
\begin{align}
\hat u_0(x_t,t)
=
\frac{1}{\alpha_t}\Big(u_t+\sigma_t^2\,s_{\theta,u}(x_t,t)\Big).
\label{eq:u0hat-comp}
\end{align}
Differentiating \eqref{eq:u0hat-comp} with respect to $a_t$ gives
\begin{align}
\partial_{a_t}\hat u_0(x_t,t)
=
\frac{\sigma_t^2}{\alpha_t}\,\partial_{a_t}s_{\theta,u}(x_t,t),
\label{eq:dua-u0hat}
\end{align}
and differentiating with respect to $u_t$ gives
\begin{align}
\partial_{u_t}\hat u_0(x_t,t)
=
\frac{1}{\alpha_t}\Big(I+\sigma_t^2\,\partial_{u_t}s_{\theta,u}(x_t,t)\Big).
\label{eq:duu-u0hat}
\end{align}
Substituting \eqref{eq:dua-u0hat}-\eqref{eq:duu-u0hat} into \eqref{eq:g-block-multiply}
yields \eqref{eq:block-guidance-final}.
\vspace{-1em}
\end{proof}
Lemma~\ref{lemma:block-guidance} shows that the $a$-component of the guidance depends exclusively on the cross-partial $\partial_{a_t}s_{\theta,u}(x_t,t)$.
This is the term that transmits information from the observation space into updates of $a$ in the joint-embedding models.
A sufficient condition for guidance attenuation is therefore
\begin{align}
\label{eq:sufficient_condition}
\partial_{a_t}s_{\theta,u}(x_t,t)\approx 0
\quad\Longrightarrow\quad
g_a(x_t,t)\approx 0.
\end{align}
Equivalently, this characterization can be stated as the following requirement:
\begin{proposition}[Coupling is necessary for effective guidance]
\label{prop:coupled-score-necessary}
Under \cref{eq:physical_obs_model}, effective guidance in coefficient space
requires the joint-embedding score model $s_\theta$ to learn a joint distribution that does not factorize across $a_t$ and $u_t$, i.e., $p_t(a_t,u_t) \neq p_t(a_t)\,p_t(u_t)$.
\end{proposition}
Let $\{x_0^{(n)}\}_{n=1}^N$ denote paired training samples of the clean variable $x_0$.
Below we connect the cross-partial to an empirical score learned on the finite samples.
\begin{lemma}[Empirical-score approximation, modified Theorem~3.2 of \citep{baptista2025memorization}]
\label{lemma:empirical-score}
Let $\{x_0^{(n)}\}_{n=1}^N$ be paired training samples of the clean joint variable $x_0=(a,u)$.
Define the empirical score $s_N(x_t,t)\coloneqq\nabla_{x_t}\log p_t^{(N)}(x_t)$
where the empirical time-$t$ density $p_t^{(N)}(x_t)$ is an isotropic Gaussian mixture:
\begin{align*}
p_t^{(N)}(x_t)
\coloneqq
\frac{1}{N}\sum_{n=1}^N
\mathcal N\big(x_t;\alpha_t x_0^{(n)},\sigma_t^2 I\big).
\end{align*}
We assume the trained joint score model $s_\theta(x_t,t)$ satisfies
\begin{align}
\sup_{x_t}\,
\big\|s_\theta(x_t,t)-s_N(x_t,t)\big\|
\le
\varepsilon_N(t),
\label{eq:empirical-score-approx}
\end{align}
with $\varepsilon_N(t)$ small under data-scarcity.
\end{lemma}

From this point onward, we work with finite-dimensional representations of the coefficient and solution variables obtained via discretization.
Accordingly, we treat $a\in\mathbb R^{d_a}$, $u\in\mathbb R^{d_u}$,
and $x=(a,u)\in\mathbb R^{d}$ with $d=d_a+d_u$.
The empirical score admits the closed form:
\begin{lemma}[Closed-form empirical score]
\label{lemma:closed-form-empirical-score}
Let the closed-form isotropic Gaussian kernel be
\begin{align*}
\varphi(r)
\coloneqq
(2\pi\sigma^2)^{-d/2}
\exp\Big(-\frac{\|r\|_2^2}{2\sigma^2}\Big).
\end{align*}
Define the mixture responsibilities
\begin{align*}
w_n(x,t)
\coloneqq
\frac{\varphi\big(x-\alpha_t x_0^{(n)}\big)}
{\sum_{j=1}^N\varphi\big(x-\alpha_t x_0^{(j)}\big)},
\qquad
\sum_{n=1}^N w_n(x,t)=1.
\end{align*}
Under \cref{lemma:empirical-score}, the empirical score admits the closed form
\begin{align}
s_N(x,t)
=
\frac{1}{\sigma^2(t)}\sum_{n=1}^N w_n(x,t)\,\big(\alpha_t x_0^{(n)}-x\big).
\label{eq:mixture-score}
\end{align}
\end{lemma}
\begin{proof}
Under \cref{lemma:empirical-score}, the empirical time-$t$ density is
\begin{align*}
p_t^{(N)}(x)
=
\frac{1}{N}\sum_{n=1}^N
\varphi\big(x-\alpha_t x_0^{(n)}\big).
\end{align*}
So the empirical score follows
\begin{align}
s_N(x,t)
=
\nabla_x \log p_t^{(N)}(x) \nonumber
&=
\frac{\sum_{n=1}^N \nabla_x \varphi\big(x-\alpha_t x_0^{(n)}\big)}
{\sum_{j=1}^N \varphi\big(x-\alpha_t x_0^{(j)}\big)}
\nonumber\\
&=
-\frac{1}{\sigma^2(t)}
\sum_{n=1}^N
\frac{\varphi\big(x-\alpha_t x_0^{(n)}\big)}
{\sum_{j=1}^N\varphi\big(x-\alpha_t x_0^{(j)}\big)}
\big(x-\alpha_t x_0^{(n)}\big)
\nonumber\\
&=
\frac{1}{\sigma^2(t)}\sum_{n=1}^N w_n(x,t)\,\big(\alpha_t x_0^{(n)}-x\big) \nonumber
\end{align}
\end{proof}

We now turn to the empirical score approximation
(\cref{lemma:empirical-score}, \cref{lemma:closed-form-empirical-score})
to analyze when the coupling condition in \cref{prop:coupled-score-necessary} can be satisfied.
Specifically, we characterize the structural source of $a$-$u$ coupling
in the empirical log-density and show that it arises solely through the mixture responsibilities.
We formalize this in the following two lemmas.
Recall that the empirical score is defined as
$s_N(x,t)=\nabla_x \log p_t^{(N)}(x)$ with
$p_t^{(N)}(x)=\sum_{n=1}^N \varphi_n(x)$ (\cref{lemma:empirical-score}).
Accordingly, the $a$--$u$ cross-block of the empirical score is
$\nabla_a\nabla_u \log \sum_{n=1}^N \varphi_n(x)$.

First, we show that this cross-partial depends only on responsibility gradients.
Let $\varphi_n(x) \coloneqq \varphi\big(x-\alpha_t x_0^{(n)}\big)$
and define the mixture responsibilities
$w_n(x)\coloneqq {\varphi_n(x)}/{\sum_{j=1}^N \varphi_n(x)}$ as in \cref{lemma:closed-form-empirical-score}.
\begin{lemma}[Cross-block of the empirical score]
\label{lemma:cross-block-from-w}
The cross-block of the empirical score satisfies
\begin{align}
\nabla_a\nabla_u\log \sum_{n=1}^N \varphi_n(x)
=
\sum_{n=1}^N
\big(\nabla_a w_n(x)\big)
\big(\nabla_u \log \varphi_n(x)\big).
\label{eq:cross-block-from-w}
\end{align}
\end{lemma}
We next provide an auxiliary lemma showing that an individual isotropic Gaussian
component induces no $a$-$u$ coupling.
\begin{lemma}
\label{lemma:single-gaussian-cross}
A single isotropic Gaussian mixture component satisfies:
\begin{align}
\nabla_a \nabla_u \log \varphi_n(x)
= 0.
\end{align}
\end{lemma}
\begin{proof}
Let $(a_c,u_c) \coloneqq \alpha_t x_0^{(n)}$.
The log-density of an isotropic Gaussian kernel is:
\begin{align*}
\log \varphi_n(x)
&=
-\frac{1}{2\sigma^2}\|a-a_c\|_2^2
-\frac{1}{2\sigma^2}\|u-u_c\|_2^2
+\mathrm{const},
\end{align*}
which contains no coupling term of the form $\langle a-a_c,\,u-u_c\rangle$.
Therefore,
\begin{align*}
\nabla_u \log \varphi_n(x)
&=
-\frac{1}{\sigma^2}(u-u_c),
\end{align*}
and differentiating w.r.t.\ $a$ yields
\begin{align*}
\nabla_a \nabla_u \log \varphi_n(x)
&= 0.
\end{align*}
\end{proof}
\cref{lemma:cross-block-from-w} and \cref{lemma:single-gaussian-cross} distinguish the source of $a$-$u$ coupling.
Specifically,
\begin{remark}[Source of $a$-$u$ coupling in the empirical score]
\label{rem:responsibility-source}
\cref{lemma:single-gaussian-cross} shows that a single isotropic Gaussian component
induces no $a$-$u$ coupling.
Since
\begin{align*}
\log\Big(\sum_{n=1}^N \varphi_n(x)\Big)
\neq
\sum_{n=1}^N \log \varphi_n(x),
\end{align*}
any nonzero $a$-$u$ coupling in the empirical log-density
$\log p^{(N)}(x)=\log\sum_{n=1}^N \varphi_n(x)$
must arise exclusively from the $\log\sum$ operation,
i.e., through the dependence of the responsibilities $w_n(x)$.
\end{remark}
Using the closed-form empirical score (\cref{lemma:closed-form-empirical-score}),
we now make explicit the cross-partial $\partial_a s_{N,u}(x,t)$ that appears in the
block-form guidance (\cref{lemma:block-guidance}).
From \cref{lemma:closed-form-empirical-score}, the $u$-component of the empirical score is
$s_{N,u}(x,t)
=
\frac{1}{\sigma^{2}(t)}
\sum_{n=1}^N
w_n(x,t)\,
\big(\alpha_t u_0^{(n)}-u\big)$.
Since $\partial_a(\alpha_t u_0^{(n)}-u)=0$, the cross-partial $\partial_a s_{N,u}(x,t)$ becomes
\begin{align}
\partial_a s_{N,u}(x,t)
=
\frac{\alpha_t}{\sigma^2(t)}
\sum_{n=1}^N
\big(\partial_a w_n(x,t)\big)\,
u_0^{(n)}.
\label{eq:exact-cross-block-from-w}
\end{align}
\cref{eq:exact-cross-block-from-w} is an explicit elaboration of \eqref{eq:cross-block-from-w} in
\cref{lemma:cross-block-from-w}.

Based on \cref{rem:responsibility-source}, we therefore analyze when these gradients vanish or remain nontrivial via responsibility gradients.
Recall the responsibilities
\begin{align}
w_n(x,t)
=
\frac{\varphi_n(x)}{\sum_{j=1}^N \varphi_n(x)}
=
\frac{\varphi\big(x-\alpha_t x_0^{(n)}\big)}
{\sum_{j=1}^N\varphi\big(x-\alpha_t x_0^{(j)}\big)},
\qquad
\sum_{n=1}^N w_n(x,t)=1.
\label{eq:w-def}
\end{align}
Let
$Z(x,t)\coloneqq \sum_{j=1}^N \varphi_j(x,t)$.
Then $w_n(x,t)={\varphi_n(x,t)}/{Z(x,t)}$ with the derivative
\begin{align}
\partial_a w_n(x,t) 
&= \frac{(\partial_a \varphi_n)Z-\varphi_n(\partial_a Z)}{Z^2} \nonumber \\
&= \frac{\varphi_n}{Z} \Big(\frac{\partial_a \varphi_n}{\varphi_n} - \sum_{j=1}^N \frac{\varphi_j}{Z} \frac{\partial_a \varphi_j}{\varphi_j} \Big) \annot{By $\partial_a Z = \sum\partial_a \varphi_j $} \nonumber \\
&= w_n(x,t)\Big(\partial_a \log \varphi_n(x,t)-\sum_{j=1}^N w_j(x,t)\,\partial_a \log \varphi_j(x,t)\Big).
\label{eq:dw-identity-varphi}
\end{align}
The following two theorems characterize this behavior through
the geometric position of $x$ relative to the training samples
$\{x_0^{(n)}\}_{n=1}^N$.
They separate two cases: a \emph{local} regime, where $x$ is dominated by a single
mixture component, and an \emph{overlap} regime, where at least two components have
comparable responsibilities.

\noindent\emph{Case 1 (local regime).}
\begin{theorem}[Local dominance implies vanishing responsibility gradients]
\label{thm:local-dominance-gradw}
Fix $(x,t)$ and define responsibilities $w_n(x,t)$ by \eqref{eq:w-def}.
Assume there exists an index $k\in[N]$ such that
\begin{align}
\sum_{j\neq k}\varphi\big(x-\alpha_t x_0^{(j)}\big)
\le
\eta\,\varphi\big(x-\alpha_t x_0^{(k)}\big)
\label{eq:eta-dominance-thm}
\end{align}
for some $\eta\in(0,1)$. 
Then the responsibility gradients satisfy
\begin{align}
\big\|\partial_a w_k(x,t)\big\|
=
O({\eta}),
\qquad
\sum_{j\neq k}\big\|\partial_a w_j(x,t)\big\|
=
O({\eta}),
\label{eq:dwrest-bigO}
\end{align}
which vanishes as $\eta \to 0$.
\end{theorem}

\begin{proof}
Under \eqref{eq:eta-dominance-thm},
\begin{align*}
Z(x,t)
=
\varphi_k(x,t)+\sum_{j\neq k}\varphi_j(x,t)
\le (1+\eta)\varphi_k(x,t),
\end{align*}
hence
\begin{align}
w_k(x,t)=\frac{\varphi_k}{Z}\ge \frac{1}{1+\eta},
\qquad
\sum_{j\neq k}w_j(x,t)=1-w_k(x,t)\le \frac{\eta}{1+\eta}.
\label{eq:offmass-bound}
\end{align}

Define the local bound
\begin{align}
G(x,t)\coloneqq \max_{1\le j\le N}\big\|\partial_a\log\varphi_j(x,t)\big\|.
\label{eq:local-bound}
\end{align}

Then $\partial_a w_n$ \eqref{eq:dw-identity-varphi} is bounded by the following two cases:

\emph{Case A: $n=k$.}
\begin{align}
\|\partial_a w_k(x,t)\| 
&= w_k \|\partial_a \log \varphi_k - \sum_j w_j\partial_a \log \varphi_j\| \nonumber \\
&= w_k \|\sum_{j\neq k} w_j \big(\partial_a \log \varphi_k - \partial_a \log \varphi_j\big) \| \nonumber \\
&\leq w_k \sum_{j\neq k} w_j \|\big(\partial_a \log \varphi_k - \partial_a \log \varphi_j\big) \| \nonumber \\
&\leq 2G(x,t) \,w_k\sum_{j\neq k} w_j \annot{By \eqref{eq:local-bound}} \nonumber \\
&\leq 2G(x,t)\frac{\eta}{1+\eta}. \annot{By \eqref{eq:offmass-bound}}
\end{align}

\emph{Case B: $n\neq k$.}
\begin{align}
\|\partial_a w_j(x,t)\|
&= w_j \|\partial_a \log \varphi_j - \sum_m w_m\partial_a \log \varphi_m\| \nonumber \\
&\leq w_j \Big( \|\partial_a \log \varphi_j \| + \sum_m w_m \|\partial_a \log \varphi_m\| \Big) \nonumber \\
&\leq 2G(x,t) w_j,
\qquad j\neq k.
\annot{By \eqref{eq:local-bound}}
\end{align}

Combining bounds from \emph{Case A} and \emph{Case B},
\begin{align}
\sum_{n=1}^N \|\partial_a w_n\| 
&= \| \partial_a w_k\| + \sum_{j\neq k} \| \partial_a w_j \| \nonumber \\
&\leq 2G\frac{\eta}{1+\eta} + 2G \sum_{j\neq k}w_j \nonumber \\
&\leq 4G \frac{\eta}{1+\eta}. \annot{By \eqref{eq:offmass-bound}} \nonumber \\
& = O(\eta) \nonumber
\end{align}
\end{proof}

\cref{thm:local-dominance-gradw} leads to a sufficient local-dominance condition for guidance attenuation during sampling process.
\begin{corollary}[Local dominance implies guidance attenuation]
\label{cor:local-dominance-attenuation}
Define $\varphi_n(x,t)\coloneqq \varphi\big(x-\alpha_t x_0^{(n)}\big)$.
If the diffusion state $x_t$ lies in a neighborhood of a single mixture center $\alpha_t x_0^{(k)}$ in the sense that
\begin{align}
\varphi_k(x_t,t)
\gg
\varphi_j(x_t,t)
\qquad \forall j\neq k,
\label{eq:local-dominance}
\end{align}
the $a$-component of the scale-free guidance (\cref{def:joint-guidance}) attenuates:
\begin{align*}
g_a(x_t,t)
\approx 0.
\end{align*}
\end{corollary}
\begin{proof}
Let $\eta$ satisfying
\begin{align*}
\sum_{j\neq k}\varphi\big(x-\alpha_t x_0^{(j)}\big)
\le
\eta\,\varphi\big(x-\alpha_t x_0^{(k)}\big),
\end{align*}
then the $\eta$ must be small due to \eqref{eq:local-dominance}.
Define the local bound
\begin{align*}
G(x,t)\coloneqq \max_{1\le j\le N}\big\|\partial_a\log\varphi_j(x,t)\big\|.
\end{align*}
Combining with \eqref{eq:exact-cross-block-from-w} yields
\begin{align*}
\big\|\partial_a s_{N,u}(x,t)\big\|
&\le
\frac{\alpha_t}{\sigma^2(t)}
\Big(\max_{n\in[N]}\|u_0^{(n)}\|\Big)
\sum_{n=1}^N \|\partial_a w_n(x,t)\| \\
&\le
\frac{4\alpha_t}{\sigma^2(t)}
\Big(\max_{n\in[N]}\|u_0^{(n)}\|\Big)
G(x,t)\frac{\eta}{1+\eta}.
\end{align*}
With small $\eta$, the cross-partial of the empirical score vanishes
\begin{align*}
\partial_a s_{N,u}(x,t)\approx 0.
\end{align*}
From \cref{lemma:empirical-score}, the trained score model $s_\theta$ approximates the empirical score $s_N$.
Thus, from \eqref{eq:block-guidance-final}, the $a$-component of the guidance is  attenuated
\begin{align*}
g_a(x,t)
\propto
\sigma_t^2\,\partial_{a_t}s_{\theta,u}(x,t)^\top\,r(x,t)
\approx 0.
\end{align*}
\end{proof}
The first case (\cref{thm:local-dominance-gradw,cor:local-dominance-attenuation}) show that, under local regime, responsibility gradients sufficiently vanish.
Next, we show the complementary case: nontrivial responsibility gradients occur only when $x$ lies in an overlap regime.

\noindent\emph{Case 2 (overlap regime).}
\begin{theorem}[Nontrivial responsibility gradients require overlap]
\label{thm:dw-implies-overlap}
Fix $(x,t)$ and define responsibilities $w_n(x,t)$ by \eqref{eq:w-def}.
Let the local bound $G(x,t)\coloneqq \max_{1\le j\le N}\big\|\partial_a\log\varphi_j(x,t)\big\|$
and let $p\in\arg\max_{j\in[N]} w_j(x,t)$ denote an index attaining the maximal responsibility.
If there exists an index $n\in[N]$ such that
\begin{align}
\big\|\partial_a w_n(x,t)\big\|\ge \delta
\qquad \text{for some }\delta>0,
\label{eq:dw-nontrivial}
\end{align}
then there exist $q\neq p$ such that
\begin{align}
\Big|
\big\|x-\alpha_t x_0^{(p)}\big\|_2^2
-
\big\|x-\alpha_t x_0^{(q)}\big\|_2^2
\Big|
\le
2\sigma^2(t)\log\Big(\frac{1-\tau(\delta)}{\tau(\delta)}\Big),
\qquad
\tau(\delta)\coloneqq \frac{\delta}{2G(x,t)(N-1)},
\label{eq:overlap-sqdist}
\end{align}
so at least two components have comparable mass at $(x,t)$.
\end{theorem}
\begin{proof}
By the identity \eqref{eq:dw-identity-varphi},
\begin{align*}
\partial_a w_n
=
w_n\Big(\partial_a \log \varphi_n-\sum_{j=1}^N w_j\,\partial_a \log \varphi_j\Big).
\end{align*}
Using the definition of local bound and $\sum_j w_j=1$,
\begin{align*}
\big\|\partial_a w_n\big\|
&\le
w_n\Big(\big\|\partial_a\log\varphi_n\big\|
+\sum_{j=1}^N w_j\big\|\partial_a\log\varphi_j\big\|\Big) \le
2G\,w_n.
\end{align*}
Thus \eqref{eq:dw-nontrivial} implies
\begin{align*}
w_n\ge \frac{\delta}{2G}.
\end{align*}
Since $p\in\arg\max_{j\in[N]} w_j(x,t)$, $w_p(x,t)\ge w_n(x,t)\ge \delta/(2G(x,t))$.
Moreover,
\begin{align*}
\sum_{j\neq p} w_j(x,t)=1-w_p(x,t)\ge 1-\frac{\delta}{2G(x,t)}.
\end{align*}
By the pigeonhole principle, there exists $q\neq p$ such that
\begin{align*}
w_q(x,t)\ge \frac{1-w_p(x,t)}{N-1}\ge \frac{1}{N-1}\Big(1-\frac{\delta}{2G(x,t)}\Big).
\end{align*}
This already yields two non-negligible responsibilities whenever $w_p$ is bounded away from $1$.
To obtain the explicit $\tau(\delta)$ bound, use instead the standard inequality
\begin{align}
\big\|\partial_a w_p(x,t)\big\|
\le 2G(x,t)\sum_{j\neq p} w_j(x,t)
=2G(x,t)\big(1-w_p(x,t)\big),
\label{eq:wp-upper}
\end{align}
which follows from the Case A derivation (rewrite the bracket as a sum over $j\neq p$ and use $G$).
If $\|\partial_a w_p(x,t)\|\ge \delta$, then \eqref{eq:wp-upper} implies $1-w_p(x,t)\ge \delta/(2G(x,t))$,
hence there exists $q\neq p$ with
\begin{align*}
w_q(x,t)\ge \frac{1-w_p(x,t)}{N-1}\ge \frac{\delta}{2G(x,t)(N-1)}=\tau(\delta).
\end{align*}
Together with $w_p(x,t)\ge 1-w_p(x,t)\ge \tau(\delta)$, we obtain:
\begin{align*}
w_p(x,t)\ge \tau(\delta),
\qquad
w_q(x,t)\ge \tau(\delta).
\end{align*}
For isotropic Gaussian kernels defined in \cref{lemma:closed-form-empirical-score},
\begin{align*}
\varphi_n(x,t) = 
(2\pi\sigma^2(t))^{-d/2}
\exp\Big(
-\frac{|x-\alpha_t x_0^{(n)}|_2^2}{2\sigma^2(t)}
\Big).
\end{align*}
Hence, for any two indices $p\neq q$
\begin{align}
\log\frac{\varphi_p(x,t)}{\varphi_q(x,t)}
&=
-\frac{1}{2\sigma^2(t)}
\Big(
|x-\alpha_t x_0^{(p)}|_2^2 -
|x-\alpha_t x_0^{(q)}|_2^2
\Big).
\label{eq:log-ratio-gauss}
\end{align}
If $w_p,w_q\ge \tau$ and $\sum_j w_j=1$, then
\begin{align}
\frac{\tau}{1-\tau}
\le
\frac{w_p(x,t)}{w_q(x,t)}
\le
\frac{1-\tau}{\tau}.
\label{eq:w-ratio-bound}
\end{align}
Combining \eqref{eq:log-ratio-gauss} and \eqref{eq:w-ratio-bound} yields
\begin{align}
\Big|
|x-\alpha_t x_0^{(p)}|_2^2-
|x-\alpha_t x_0^{(q)}|_2^2
\Big|
&= 2\sigma^2(t)\log\frac{\varphi_p(x,t)}{\varphi_q(x,t)} \annot{By \eqref{eq:log-ratio-gauss}}\nonumber \\
&= 2\sigma^2(t)\log\frac{w_p(x,t)}{w_q(x,t)} \annot{Since $w_n= \varphi_n/\sum_j\varphi_j$ } \nonumber \\
&\le
2\sigma^2(t)\log\Big(\frac{1-\tau}{\tau}\Big), \annot{By \eqref{eq:w-ratio-bound}} \nonumber
\end{align}
which is exactly \eqref{eq:overlap-sqdist}.
\end{proof}

\cref{thm:dw-implies-overlap} leads to a necessary overlapping condition for non-vanishing guidance during sampling process.
\begin{corollary}[Non-vanishing guidance requires overlap]
\label{cor:nonvanishing-requires-overlap}
For the $a$-component of the scale-free guidance (\cref{def:joint-guidance}) to be non-zero by
\begin{align*}
\|g_a(x_t,t)\| > 0,
\end{align*}
the state $x_t$ must lie in an overlap region of at least two mixture components, in the sense that there exist $p\neq q$ satisfying
\begin{align*}
\Big|
\|x_t-\alpha_t x_0^{(p)}\|_2^2
-
\|x_t-\alpha_t x_0^{(q)}\|_2^2
\Big|
\lesssim
\sigma^2(t).
\end{align*}
\end{corollary}
\begin{proof}
From \cref{lemma:block-guidance} and \cref{eq:exact-cross-block-from-w},
non-vanishing guidance implies $\|\partial_a w_n(x,t)\|>0$ for some $n$.
The claim then follows directly from \cref{thm:local-dominance-gradw}.
\end{proof}
A contrapositive statement of \cref{cor:nonvanishing-requires-overlap} is:
\begin{remark}[Non-overlap implies guidance attenuation]
\label{remark:nonoverlap-attenuation}
If the diffusion state $x_t$ does not lie in an overlap region of the mixture,
then the $a$-component of the scale-free guidance vanishes by
$g_a(x_t,t)\approx 0$.
\end{remark}

\cref{cor:local-dominance-attenuation},
\cref{cor:nonvanishing-requires-overlap}, and \cref{remark:nonoverlap-attenuation}
together yield a geometric interpretation of guidance behavior in
joint-embedding diffusion models:
(i) far from all mixture centers, guidance attenuates;
(ii) even when close to a single mixture center, guidance still attenuates; and
(iii) nonzero guidance is possible only when $x$ lies near at least two mixture
centers simultaneously.
Under data scarcity, the overlap between mixture components may be absent altogether, causing coefficient-space guidance to collapse throughout the sampling process.

%% file: appendix_daps_failure.tex
\label{app:daps-failure}
In this appendix, we provide formal proofs for the covariance-collapse results used in \cref{sec:daps_failure}.
We work in a separable Hilbert space $\mathcal{H}$ with inner product $\langle \cdot, \cdot \rangle$.
Let $\delta_{x_i}: \mathcal{H} \to \mathbb{R}$ denote the point evaluation functional at location $x_i \in \Omega$, satisfying $\langle \delta_{x_i}, f \rangle = f(x_i)$ for $f \in \mathcal{H}$.

Starting from the potential $U$ in \cref{eq:potential}, the functional gradient is
\begin{align}
\nabla U(x) = C^{-1}(x - x_0) + \frac{1}{\sigma_s^2} \delta_{x_i} (\langle \delta_{x_i}, x \rangle - c).
\label{eq:gradient}
\end{align}
Substituting \cref{eq:gradient} into \eqref{eq:langevin} and collecting affine terms yields the (preconditioned) linear SDE
\begin{align}
dx_t = -\Sigma\left(C^{-1} + \frac{1}{\sigma_s^2} \delta_{x_i} \otimes \delta_{x_i}\right) x_t\, dt + b\, dt + \sqrt{2} \Sigma^{1/2} dW_t,
\label{eq:linear_sde}
\end{align}
for some drift element $b \in \mathcal{H}$ that does not affect the stationary covariance.

\begin{lemma}[Lyapunov Equation for Stationary Covariance]
\label{lemma:lyapunov}
Let $B \coloneqq C^{-1} + \frac{1}{\sigma_s^2} \delta_{x_i} \otimes \delta_{x_i}$.
The stationary covariance operator $\Sigma_\infty$ of the linear SDE \eqref{eq:linear_sde} satisfies the Lyapunov equation:
\begin{align}
\Sigma B\, \Sigma_\infty + \Sigma_\infty\, (\Sigma B)^* = 2\Sigma.
\label{eq:lyapunov}
\end{align}
If $C$ and $\Sigma$ are self-adjoint (hence $B$ is self-adjoint), then the stationary covariance is
\begin{align}
\Sigma_\infty = B^{-1}.
\label{eq:stationary_cov}
\end{align}
\end{lemma}

\begin{proof}
This is a standard fact for preconditioned Langevin dynamics targeting $\pi(x)\propto \exp(-U(x))$. (See \citet{pavliotis2014stochastic}, Proposition 3.10.) For completeness, we verify that \eqref{eq:stationary_cov} solves \eqref{eq:lyapunov}.
Since $B$ and $\Sigma$ are self-adjoint, $(\Sigma B)^* = B\Sigma$.
Substituting $\Sigma_\infty=B^{-1}$ into \eqref{eq:lyapunov} gives
$
\Sigma B B^{-1} + B^{-1} B \Sigma = \Sigma + \Sigma = 2\Sigma
$, as required.
\end{proof}

\begin{lemma}[Sherman-Morrison Inversion]
\label{lemma:sherman_morrison}
Let $B = C^{-1} + \frac{1}{\sigma_s^2} \delta_{x_i} \otimes \delta_{x_i}$ where $C$ is a positive definite covariance operator. Then:
\begin{align}
B^{-1} = C - \frac{(C \delta_{x_i}) \otimes (C \delta_{x_i})}{\sigma_s^2 + C(x_i, x_i)}
\label{eq:sherman_morrison}
\end{align}
where $C(x_i, x_i) = \langle \delta_{x_i}, C \delta_{x_i} \rangle$ is the pointwise variance at $x_i$.
\end{lemma}
\begin{proof}
We apply the Sherman-Morrison formula~\cite{smformula} with base operator $C^{-1}$ and rank-1 perturbation:
\begin{align}
    (A + u \otimes v)^{-1} = A^{-1} - \frac{(A^{-1} u) \otimes (A^{-1} v)}{1 + \langle v, A^{-1} u \rangle}
\end{align}

After substituting $A=C^{-1}$ and $u \otimes v = \frac{1}{\sigma_s^2} \delta_{x_i} \otimes \delta_{x_i}$, we obtain
\begin{align}
B^{-1} = C - \frac{(C \delta_{x_i}) \otimes (C \delta_{x_i})}{\sigma_s^2 + \langle \delta_{x_i}, C \delta_{x_i} \rangle}
= C - \frac{(C \delta_{x_i}) \otimes (C \delta_{x_i})}{\sigma_s^2 + C(x_i, x_i)}.
\end{align}
\end{proof}

\begin{theorem}[Sparse constraint induces correlation shrinkage]
\label{thm:constrained_terms}
Consider the preconditioned Langevin SDE \eqref{eq:linear_sde}:
\begin{align*}
dx_t = -\Sigma\left(C^{-1} + \frac{1}{\sigma_s^2} \delta_{x_i} \otimes \delta_{x_i}\right) x_t\, dt + b\, dt + \sqrt{2} \Sigma^{1/2} dW_t,
\end{align*}
where $C \succ 0$ is the prior covariance operator, $\Sigma \succ 0$ is the noise covariance operator, $\sigma_s^2 > 0$ is the constraint strength, and $b \in \mathcal{H}$ is a constant drift element. 
Let $\Sigma_\infty$ denote the stationary covariance operator of this process.
Then for any location $x_k \in \Omega$, the covariance function values involving the constrained location $x_i$ satisfy
\begin{align}
\Sigma_\infty(x_i, x_k) = \Sigma_\infty(x_k, x_i) = \frac{\sigma_s^2}{\sigma_s^2 + C(x_i, x_i)}\, C(x_i, x_k).
\end{align}
\end{theorem}
\begin{proof}
From Lemma \ref{lemma:lyapunov}, we have $\Sigma_\infty = B^{-1}$.
Using Lemma \ref{lemma:sherman_morrison},
\begin{align}
\Sigma_\infty
= C - \frac{(C \delta_{x_i}) \otimes (C \delta_{x_i})}{\sigma_s^2 + C(x_i, x_i)}.
\end{align}
Evaluating the covariance function at $(x_i, x_k)$ gives
\begin{align}
\Sigma_\infty(x_i, x_k)
&= C(x_i, x_k) - \frac{C(x_i, x_i)\, C(x_i, x_k)}{\sigma_s^2 + C(x_i, x_i)}\\
&= C(x_i, x_k)\, \frac{\sigma_s^2}{\sigma_s^2 + C(x_i, x_i)}.
\end{align}

By symmetry of $\Sigma_\infty$, we have $\Sigma_\infty(x_k, x_i) = \Sigma_\infty(x_i, x_k)$ for all $x_k$.
\end{proof}

%% file: appendix_exp_setup.tex
\subsection{Dataset}

\paragraph{Inverse Poisson.}
We also consider the Poisson equation,
\begin{align*}
\nabla^2 u(x) = a(x), \quad x \in (0,1)^2,
\end{align*}
under homogeneous Dirichlet conditions $u|_{\partial\Omega} = 0$. 
Here $a(x)$ is generated by Gaussian random fields $a \sim \mathcal{N}(0,(-\Delta+9I)^{-2})$.
Given sparse samples of $u(x)$, the task is to recover the coefficient field $a(x)$.

\paragraph{Inverse Helmholtz.}
We study the two-dimensional Helmholtz equation,
\begin{align*}
\nabla^2 u(x) + k^2 u(x) = a(x), \quad x \in (0,1)^2,
\end{align*}
with $k=1$ and homogeneous Dirichlet boundary conditions $u|_{\partial\Omega} = 0$. 
Coefficient fields $a(x)$ are sampled from Gaussian random fields $a \sim \mathcal{N}(0,(-\Delta+9I)^{-2})$.
This setting is particularly challenging due to the oscillatory nature of the Helmholtz solution and the multiscale structure of the GRF prior.

\paragraph{Inverse Navier-Stokes.}
\label{para:dataset-ns}
We further study the two-dimensional incompressible Navier-Stokes equations in vorticity form. 
Let $u(x,t)$ denote the velocity field and $w(x,t)=\nabla\times u(x,t)$ the corresponding vorticity. 
The system evolves according to
\begin{align*}
\partial_t w(x,t) + u(x,t)\cdot\nabla w(x,t) &= \nu \Delta w(x,t) + f(x), \quad x\in(0,1)^2,\ t\in(0,T], \\
\nabla\cdot u(x,t) &= 0, \quad x\in(0,1)^2,\ t\in[0,T]
\end{align*}
with viscosity $\nu=10^{-3}$ and periodic boundary conditions. 
The system is initialized by the vorticity field $w(x,0)=a(x)$, where $a(x)$ is sampled from a Gaussian random field $a \sim \mathcal{N}(0,\,7^{3/2}(-\Delta+49I)^{-5/2})$ and treated as the unknown coefficient. 
Given sparse observations of the solution $w(x,T)$ at the terminal time, the task is to recover the initial vorticity $w(x,0)$. 

It is worth noting that, due to the loss of information, it is impossible to calculate PDE residual for physics-informed guidance; moreover, the property $\nabla\cdot(\nabla\times u)=0$ mentioned in \citet{huang2024diffusionpde} is invalid.

\subsection{Details on Decoupled Diffusion Inverse Solver}
\label{sec:ddis_details}

We summarize the key training hyperparameters in \cref{tab:model_training_hyper} and the sampling hyperparameters in
\cref{table:ddis_shared_setup,table:ddis_budget}.
All experiments are conducted on a single NVIDIA RTX~4090 GPU.
Additional details are provided below.

\paragraph{FNO padding and cropping.}
FNO layers rely on FFT-based spectral convolution and therefore assume periodic boundary conditions.
To mitigate boundary artifacts, we pad the input field by $p$ grid cells along each spatial dimension,
apply the operator $L_\phi$ on the padded grid, and crop the output back to the original domain.

\paragraph{Physics-informed training.}
In physics-informed training, we add a physics-regularization term to the loss function to help the model generalize better by enforcing the PDE residual to be small. In all cases, the neural operator is first trained on (limited) paired data supervision, and then trained with both (limited) data and physics-regularization terms. The physics-regularization term is weighted by $\lambda_{PDE}=0.1,\lambda_{BC}=10$. StepLR scheduler is used with a decay of 0.1.

\begin{table}[pt]
\centering
\caption{Model training hyperparameters.}
\label{tab:model_training_hyper}
\vspace{-0.5em}
\begin{tabular}{lll}
\toprule
\textbf{Component} & \textbf{Hyperparameter} & \textbf{Value} \\
\midrule
\multirow{10}{*}{Diffusion prior}
& Training duration & $10$M images \\
& Batch size & $32$ \\
& Channel base $c_{\text{base}}$ & $64$ \\
& Channels per resolution $c_{\text{res}}$ & $[1,2,4,4]$ \\
& Learning rate & $1\times 10^{-4}$ \\
& LR ramp-up & $5$M images \\
& EMA half-life & $0.5$M images \\
& Dropout prob. & $0.13$ \\
& \# params & $183$M \\
\midrule
\multirow{6}{*}{FNO surrogate}
& Architecture & FNO with padding \\
& Fourier modes & $(64,64)$ \\
& \# layers & $4$ \\
& Hidden channels & $64$ \\
& In / out channels & $1 \rightarrow 1$ \\
& \# params & $8$M \\
\bottomrule
\end{tabular}
\vspace{-1em}
\end{table}

\begin{table}[p]
\centering
\caption{Shared experimental setup (all tasks unless specified).}
\label{table:ddis_shared_setup}
\vspace{-0.5em}
\begin{tabular}{ll}
\toprule
\textbf{Item} & \textbf{Value} \\
\midrule
Tasks & Helmholtz / Poisson / Navier--Stokes \\
Default grid resolution (training) & $128^2$ \\
Default grid resolution (inference) & Half $64^2$ and half $128^2$ \\
Observation type & Sparse point observations \\
\# observed points $N_{\mathrm{obs}}$ & $500 \approx 3\%$ on $128^2$ \\
Observation loss & $\ell_1$ \\
FNO padding $p$ & 2 \\
RBF initialization scale & 0.05 \\
\midrule
$\sigma_{\min}$ & $0.002$ \\
$\sigma_{\max}$ & $80$ \\
Noise exponent $\rho$ & $7$ \\
\midrule
Diffusion step & $N=5$ \\
Diffusion & $\sigma_{\min}=0.001,\sigma_{\text{final}}=0$ \\
Annealing noise range & $\sigma_{\max}=10,\ \sigma_{\min}=0.01$ \\
Langevin step size $\eta$ & $0.1$ \\
Langevin noise scale $\tau$ & $10^{-3}$ \\

\bottomrule
\end{tabular}
\vspace{-1em}
\end{table}

\begin{table*}[p]
\centering
\caption{DDIS sampling hyperparameters for the three budget groups (\cref{table:performance}).}
\label{table:ddis_budget}
\vspace{-0.5em}
\begin{tabular}{llcc}
\toprule
\textbf{Task} & \textbf{Time Budget}
& \textbf{Anneal} $(N)$
& \textbf{Langevin} $(N,\ \text{weights}_1,\ \text{weights}_2)$ \\
\midrule

\multirow{3}{*}{Poisson}
& $\sim 16$s
& $100$
& $20,\ 10,\ 15$ \\
& $\sim 32$s  
& $200$
& $20,\ 10,\ 15$ \\
& $\sim 128$s 
& $200$
& $100,\ 5,\ 5$ \\
\midrule

\multirow{3}{*}{Helmholtz}
& $\sim 16$s
& $100$
& $20,\ 5,\ 10$ \\
& $\sim 32$s  
& $200$
& $20,\ 10,\ 10$ \\
& $\sim 128$s 
& $200$
& $100,\ 5,\ 5$ \\
\midrule

\multirow{3}{*}{Navier-Stokes}
& $\sim 16$s
& $100$
& $20,\ 10,\ 20$ \\
& $\sim 32$s  
& $200$
& $20,\ 10,\ 10$ \\
& $\sim 128$s 
& $200$
& $100,\ 10,\ 10$ \\
\bottomrule
\end{tabular}
\vspace{-2em}
\end{table*}

\subsection{Details on Flow-based Models}
\label{appendix:exp-flow}

\subsubsection{Prior Learning and Hyperparameters}
\label{exp:flow-arch}

For flow-based baselines, existing work primarily focuses on inverse problems involving only the solution channel $u$, leveraging a flow-based prior $p(u)$ to address the inverse problem where given $u_{obs}$. To this end, we extend representative flow-based methods to a joint-embedding setting $(a,u)$ for comparison, and follow the official OFM \citep{shi2025stochastic} implementation to train a flow-matching prior with an FNO backbone (71.42M). 
Key model and training hyperparameters for prior learning are summarized in \cref{table:flow-pretrain}. All models are trained on a single NVIDIA RTX 4090 GPU.

\begin{table}[p]
\centering
\small
\caption{Hyperparameter configurations for the flow-based prior models.}
\label{table:flow-pretrain}
\begin{tabular}{lc}
\toprule
\textbf{Hyperparameter} & \textbf{Value} \\
\midrule
\multicolumn{2}{l}{\textbf{Model Architecture}} \\
Fourier Modes & 32 \\
Hidden Channels & 128 \\
MLP Projection Width & 128 \\
Visual Channels ($a, u$) & 2 \\
Spatial Dimensions & 2D \\
\midrule
\multicolumn{2}{l}{\textbf{GP Prior}} \\
Kernel Lengthscale & 0.01 \\
Kernel Variance & 1.0 \\
Matérn Parameter & 0.5 \\
Minimum Noise & $1 \times 10^{-4}$ \\
\midrule
\multicolumn{2}{l}{\textbf{Training Setup}} \\
Total Epochs & 300 \\
Batch Size & 100 \\
Learning Rate & $1 \times 10^{-3}$ \\
Optimizer & Adam \\
LRScheduler & StepLR \\
Scheduler Step Size & 50 \\
Scheduler Decay & 0.8 \\
\bottomrule
\end{tabular}
\end{table}

\subsubsection{Inference Setup}

During inference, we evaluate the following flow-based posterior sampling methods, which operate on the same pretrained model described above, and compare them with our decoupled approach. To ensure a fair comparison and verify the quality of the learned prior, we also report performance on the "forward problem," defined as the reconstruction of the full solution field $u$ from partial/masked observations $ u _ {obs} $, to align with these methods' original tasks.

\paragraph{ECI-sampling.} ECI-sampling \citep{cheng2025gradientfree} applies an extrapolation--correction--interpolation scheme during the flow ODE integration.
It enforces hard constraints (e.g., Dirichlet conditions) by replacing observations in intermediate predictions and applying these ECI steps at each ODE step. Due to inherent method-level limitations in robustness, as discussed in \cref{sec: A.5} and visualized in \cref{fig:fundaps_failure}, we conducted an extensive hyperparameter search (summarized in \cref{table:eci-poisson}), but observed only limited performance gains even on the forward problem, and thus report representative configurations in the main results.

\paragraph{OFM Regression.}
Operator Flow Matching (OFM) \citep{shi2025stochastic} interprets the function-space flow model as a bijective mapping from a Gaussian process (GP) prior space to the target function space, and formulates exact posterior inference in probability form. In practice, it requires solving a flow ODE alongside Hutchinson trace estimation, then relying heavily on Langevin Dynamics to optimize the latent GP variables. We follow the official implementation, using a learning rate of $ 1 \times 10^{-3}$ and an observation noise of $1 \times 10^{-3}$; other inference hyperparameters are shown in \cref{table:flow-setup}.

\subsubsection{Quantitative Results and Analysis} 
\label{exp:ofm-issue}
The quantitative performance is shown in \cref{table:flow-setup}. As these methods are originally designed for single-channel inverse problems, we also report their performance on forward problems for completeness. While OFM yields more competitive results than ECI-sampling, it encounters method-level bottlenecks that limit its practical utility in complex PDE settings:

\paragraph{Computational Efficiency.} In practice, while the official OFM implementation suggests 20,000 Langevin steps for convergence, this is computationally prohibitive in our PDE setting, as it would exceed 20 hours per sample. This requirement, when coupled with adaptive ODE solvers and joint $(a,u)$ backpropagation, significantly escalates VRAM consumption and makes standard inference budgets impractical. Qualitative results for OFM across different step counts are provided in \cref{fig:ofm}. 

\paragraph{Numerical Instability.} Furthermore, the steep gradients from the data-fidelity term $\| M \odot u - u_{\mathrm{obs}}\|^2 / \sigma^2$ can drive samples off the learned prior manifold, leading to frequent numerical instabilities (e.g., NaN/Inf or OOM) under standard GPU budgets. To ensure a statistically meaningful and stable evaluation despite these challenges, we report the OFM regression average performance across 10 independent runs under practical constraints.

\begin{table}[p]
\centering
\small
\caption{\textbf{Performance comparison of ECI-sampling on the Poisson inverse problem.} $\ell_{2,\mathrm{fwd}}$ and $\ell_{2,\mathrm{inv}}$ denote the relative $\ell_2$ error (\%) for the forward and inverse problems, respectively. For cases where the model fails to produce physically meaningful results, relative errors exceeding 100\% are capped at 100\% for clarity.}

\label{table:eci-poisson}
\begin{tabular}{cccccc}
\toprule
Method 
& \makecell{Euler \\ Steps} 
& \makecell{ECI \\ Steps} 
& $\ell_{2,\mathrm{inv}}$ & $\ell_{2,\mathrm{fwd}}$\\
\midrule
\multirow{12}{*}{ECI-sampling} 
& 100  & 1  & 94.56 & 67.47 \\
& 200  & 1  & 94.83 & 58.95 \\
& 200  & 5  & 95.29 & 42.64 \\
& 200  & 10 & 95.25 & 36.17 \\
& 400  & 5  & 95.18 & 35.95 \\
& 400  & 10 & 94.70 & 30.12 \\
& 800  & 5  & 94.63 & 30.04 \\
& 800  & 10 & 93.81 & 25.78 \\
& 2000 & 5  & 93.47 & 24.72 \\
& 2000 & 10 & 92.69 & 22.55 \\
& 4000 & 5  & 92.70 & 22.55 \\
& 4000 & 10 & 91.69 & 21.27 \\
\bottomrule
\end{tabular}
\end{table}

\begin{table}[t]
\centering
\small
\setlength{\tabcolsep}{2.35pt}
\caption{
\textbf{Performance of flow-based inverse solvers across PDE tasks.} Times (s/sample) are averaged across tasks and normalized by batch size.
For entries without explicit Euler step counts, adaptive ODE solvers (e.g., \texttt{dopri5} \citep{DORMAND198019}) are used.
$\ell_{2,\mathrm{fwd}}$ and $\ell_{2,\mathrm{inv}}$ denote the relative $\ell_2$ error (\%) for the forward and inverse problems, respectively.}
\label{table:flow-setup}
\vspace{-0.5em}
\begin{tabular}{lcccccccccccccc}
\toprule
\multirow{2}{*}{Method}
& \multirow{2}{*}{\makecell{Euler \\ Steps}}
& \multirow{2}{*}{\makecell{ECI \\ Steps}}
& \multirow{2}{*}{\makecell{Langevin \\ Steps}}
& \multirow{2}{*}{\makecell{Hutchinson \\ Samples}}
& \multicolumn{2}{c}{Poisson}
& \multicolumn{2}{c}{Helmholtz}
& \multicolumn{2}{c}{N-S}
& \multirow{2}{*}{GPU}
& \multirow{2}{*}{\makecell{Time \\ (s)}} \\
\cmidrule(lr){6-7} \cmidrule(lr){8-9} \cmidrule(lr){10-11}
& & & & 
& $\ell_{2,\mathrm{inv}}$ & $\ell_{2,\mathrm{fwd}}$
& $\ell_{2,\mathrm{inv}}$ & $\ell_{2,\mathrm{fwd}}$
& $\ell_{2,\mathrm{inv}}$ & $\ell_{2,\mathrm{fwd}}$
& & \\
\midrule

\multirow{2}{*}{ECI}
& 800 & 5 & / & / & 94.63 & 30.04 & 92.83 & 25.57 & 42.36 & 20.61 & RTX 4090 & 0.20 \\
& 2000 & 5 & / & / & 93.47 & 24.72 & 93.23 & 23.31 & 41.68 & 17.08 & RTX 4090 & 0.38 \\

\midrule

\multirow{4}{*}{OFM$^{\dagger}$}
& 150 & / & 100  & 1 & 51.48 & 13.10 & 37.66 & 9.18 & 35.74 & 28.10 & RTX 4090 & 246.05 \\
& 150 & / & 1000 & 4 & 31.31 & 15.26 & 39.81 & 29.04 & 29.78 & 23.02 & RTX 4090 & 2445.13 \\
& /  & / & 100  & 1 & 71.87 & 20.15 & 49.60 & 12.85 & 37.57 & 31.15 & 2$\times$ H100$^{\ddag}$ & 470.44 \\
& /  & / & 1000 & 4 & 47.04 & 7.41 & 42.07 & 8.84 & 20.98 & 13.55 & 2$\times$ H100$^{\ddag}$ & 4366.34 \\

\bottomrule
\end{tabular}
\parbox{\linewidth}{\footnotesize
$^{\dagger}$ Unstable inference runs producing NaN or Inf values are assigned a relative $\ell_2$ error of 100\%.
}
\parbox{\linewidth}{\footnotesize
$^{\ddag}$ We select these configurations in \cref{table:performance} based on forward performance, as Euler solvers may introduce overly smooth shortcuts.
}
\vspace{-1em}
\end{table}

\subsection{Details on Other Baseline Methods}

For Baseline Methods, we compare with the following plug-and-play algorithms that solve PDE inverse problems via a joint-embedding diffusion models, as shown below.

\paragraph{DiffusionPDE.} DiffusionPDE \citep{huang2024diffusionpde} uses a finite-dimensional diffusion model with a 
physics-informed DPS formulation, combining observation loss and PDE residuals in its posterior sampling.
We use their official checkpoints (54.41M parameters) and maintain the original hyperparameter settings while varying the diffusion sampling steps (500, 1000, and 3000) to report results under different computational budgets in \cref{table:performance}. The results for FNO~\citep{li2020fourier} and DeepONet~\citep{lu2019deeponet} are taken from the original DiffusionPDE paper.

\paragraph{FunDPS.} FunDPS \citep{yao2025guided} uses a joint-embedding diffusion model in function space and represents the prior state-of-the-art for sparse-observation inverse PDE problems. Their model has 54M parameters. We maintain the original hyperparameter settings while varying the diffusion sampling steps to evaluate performance under different budgets.